\documentclass[twoside]{article}

\usepackage[accepted]{aistats2021arxiv}
\usepackage[round]{natbib}

\usepackage{hypbmsec}
\usepackage{hyperref}
\usepackage{xcolor}
\usepackage{amssymb}
\usepackage{amsfonts}
\usepackage{xifthen}
\usepackage{xparse}
\usepackage{dsfont}
\usepackage{xspace}
\usepackage{amsmath}
\usepackage{amsthm}
\usepackage{float}

\usepackage{eqnarray}
\usepackage[shortlabels]{enumitem} 

\usepackage{soul}
\usepackage{hl_short}

\newtheorem{theorem}{Theorem}
\newtheorem{lemma}[theorem]{Lemma}
\newtheorem{corollary}[theorem]{Corollary}

\newcommand{\lr}[1]{\left (#1\right)}

\newcommand\kl[2]{KL\left(#1 \mid #2\right)}

\newcommand{\V}{\mathbb{V}}
\newcommand{\hatV}{\hat{\mathbb{V}}}
\NewDocumentCommand{\1}{o}{\mathds 1{\IfValueT{#1}{\lr{#1}}}}
\let\P\undefined
\NewDocumentCommand{\P}{o}{\mathbb P{\IfValueT{#1}{\lr{#1}}}}

\DeclareMathOperator{\MSE}{sq}
\DeclareMathOperator{\zeroone}{0/1}
\DeclareMathOperator{\CE}{ce}
\let\L\undefined
\newcommand{\L}{L}

\begin{document}

%

%

\twocolumn[

\aistatstitle{Diversity and Generalization in Neural Network Ensembles}

\aistatsauthor{ Luis A. Ortega \And Rafael Caba\~{n}as \And  Andrés R. Masegosa }

\aistatsaddress{ Autonomous University of Madrid 
\And    University of Almer\'ia,\\ IDSIA
\And Aalborg University } ]

\begin{abstract}
Ensembles are widely used in machine learning and, usually, provide state-of-the-art performance in many prediction tasks. From the very beginning, the diversity of an ensemble has been identified as a key factor for the superior performance of these models. But the exact role that diversity plays in ensemble models is poorly understood, specially in the context of neural networks. In this work, we combine and expand previously published results in a theoretically sound framework that describes the relationship between diversity and ensemble performance for a wide range of ensemble methods. More precisely, we provide sound answers to the following questions: how to measure diversity, how diversity relates to the generalization error of an ensemble, and how diversity is promoted by neural network ensemble algorithms. This analysis covers three widely used loss functions, namely, the squared loss, the cross-entropy loss, and the 0-1 loss; and two widely used model combination strategies, namely, model averaging and weighted majority vote. We empirically validate this theoretical analysis with neural network ensembles.

\end{abstract}

\section{Introduction}


Ensemble methods are one of the most widely used and studied techniques in machine learning \citep{hansen:90,Bre96,Bre01}. It has been successfully applied in many real-world problems \citep{girshick2014rich,wang2012mining,zhou2014ensemble,ykhlef2017efficient} and is usually part of the winning strategies in many machine learning competitions \cite[e.g.,][]{chen2016xgboost,hoch2015ensemble,puurula2014kaggle,stallkamp:12}. Recently, ensembles are also becoming very popular to improve uncertainty modeling in deep neural networks \citep{lakshminarayanan2017simple,wen2019batchensemble,maddox2019simple,wenzel2020hyperparameter}. 


Ensembles are created by combining several individual predictors. It is widely accepted \citep{dietterich2000ensemble,lu2010ensemble}  that the prediction performance of an ensemble jointly depends of the individual performance and the diversity of its individual members. Intuitively speaking, a set of predictors is diverse when their predictions do not coincide on all the samples. We know that when classifiers are diverse, they tend to make independent errors, therefore when they are aggregated, their errors tend to cancel out \citep{BK16}, which improves the ensemble prediction. For this reason, diversity has long been recognized as a key factor in ensemble performance \citep{kuncheva2003measures,cunningham2000diversity,brown2005managing}. The same cancelation of errors effect happens in the case of neural network ensembles \citep{hansen:90,lee2016stochastic,lakshminarayanan2017simple} where heuristic measures of diversity are usually analyzed to get insights of the ensemble learning algorithms \citep{fort2019deep,wen2019batchensemble,wenzel2020hyperparameter}. %

Unfortunately, there is a lack of consensus surrounding the underlying theory that can explain the role of diversity in the generalization performance of ensembles. The error rate of an ensemble and an individual predictor, for example, is well defined by the use of a loss function, but there is no well-established definition of diversity \citep{kuncheva2003measures}. And it is not well known how exactly the diversity among ensemble members affects the generalization error of the ensemble.

In this work, we introduce a novel theoretical framework that explains the relationship between diversity and the generalization performance of an ensemble. This theoretical framework is derived from previously published results with no direct connection among them \citep{krogh1995validation,masegosa2020learning,masegosa2020second}. The main contribution of our work is to find a theoretically sound way to combine these previous results in a single theoretical framework that explains the role of diversity in the generalization performance of a wide range of different ensembles. In our opinion, this general framework could potentially help the machine learning community to have a better understanding of the underlying trade-offs that have to be considered when designing novel ensemble learning algorithms, specially in the context of neural networks. The detailed contributions of this work are the followings: a general measure of ensemble diversity; a theoretical analysis that shows how the correlation among ensemble members affects diversity; the exact trade-off that exists between this diversity measure, the performance of the individual predictors and the generalization error of the ensemble; an analysis of the strategies used by most of current neural network ensemble learning algorithms to promote diversity; and, finally, an empirical evaluation of this theoretical framework. This analysis covers model averaging and weighted majority vote ensembles under the cross-entropy loss, square error and 0-1 loss.

\section{Related Work}
\label{sec:rw}


\subsection*{Diversity Measures}

Many different names have been given to the concept of diversity, including ambiguity \citep{krogh1995validation}, dependency \citep{zhou2012ensemble}, orthogonality \citep{kuncheva2003measures}, disagreement \citep{masegosa2020second}, and so on. And there is a huge literature proposing many different measures of diversity, mostly employing the predictions made by the individual models \citep{tang2006analysis, kuncheva2003measures,zhou2010multi,chandra2004divace,roli2001methods,buschjager2020generalized}. However, none of these publications presents a generic diversity measure that can be applied to different kind of ensembles and has formal ties to the ensemble generalization error.

\subsection*{Diversity and Generalization}

The first theoretical attempts to explain why diversity reduces the ensemble error appeared in \citep{krogh1995validation,geman1992neural}, in the context of regression ensembles. These proposals show why ensembles with high diversity reduce the prediction error. However, these works do not describe the relationship between the empirical diversity of an ensemble and the generalization error of the ensemble. In this work, we extend the results given by \citet{krogh1995validation} and establish such a link for regression ensembles using PAC-Bayesian bounds  \citep{McA98}. Moreover, our analysis also applies to ensembles of classifiers. 

Many different works have tried to adapt the squared error decomposition of regression ensembles \citep{krogh1995validation} to classification settings \citep{brown2009information, zhou2012ensemble,yu2011diversity,jiang2017generalized}. In this work, we employ upper bounds over the ensemble error to derive a novel decomposition which applies to different loss functions, and which incorporates a term due to the error of individual members as well as a diversity component. 

In the context of majority vote ensembles and PAC-Bayesian theory, \citep{LMR11,GLL+15} proposed a PAC-Bayesian bound which depends on the error rate and the variance of the individual classifiers. This variance term can be interpreted as a diversity measure. However the analysis only applies to binary classifiers.
Recently, \citet{masegosa2020second} presented another theoretical analysis  of majority vote ensembles using a novel PAC-Bayesian bound which applies to multiclass classification problems.
However, this bound only contains a diversity term when applied to binary classification, which coincides with diversity term found by \citet{LMR11,GLL+15}.
In this work, we extend the results of \citet{masegosa2020second} and introduce a PAC-Bayesian bound for multiclass classification which explicitly contains a novel term measuring the diversity of the ensemble. We also establish novel connections with regression ensembles and model averaging of probabilistic classifiers. 

In the context of model averaging of probabilistic classifiers, \citet{masegosa2020learning} introduced a novel theoretical analysis of this problem using a PAC-Bayesian bound which explicitly contains a term measuring the diversity of this kind of ensembles. 
However, \citet{masegosa2020learning} was mainly focused on the analysis of Bayesian model averaging under model misspecification. The role of diversity in ensemble learning was briefly mentioned in this work and was not theoretically or empirically analyzed. We extend this work by doing a considerably more in-depth theoretical and empirical investigation of the role that diversity plays in the generalization performance of various kind of ensembles. 

\subsection*{Diversity and Ensemble Learning}
Virtually, all ensembles methods encourage diversity among their individual models either implicitly or explicitly. Classic techniques like Bagging \citep{Bre96,Bre01} or Boosting \citep{FS96} implicitly encourage diversity by creating different training data sets. Deep ensembles also resort to implicit techniques such as random initialization \citep{lakshminarayanan2017simple,wen2019batchensemble}, tweaking the optimizer \citep{maddox2019simple,zhang2019cyclical,wenzel2020hyperparameter} or employing different hyperparameter settings \citep{wenzel2020hyperparameter}. In this work, we propose a theoretical explanation of why these diversity-promoting strategies result in superior ensemble models. 

A growing number of ensemble learning methods  \citep{liu1999ensemble,jiang2017generalized,buschjager2020generalized, pang2019improving,jain2020maximizing} 
are based on optimizing loss functions which include a term to explicitly induce diversity. However, these loss functions do not have a theoretical connection to the generalization error of the ensemble in contrast to the ones considered here. \citet{masegosa2020learning} introduced an ensemble learning algorithm for the cross-entropy loss based on PAC-Bayesian bounds which explicitly encourage diversity, but it was only applied to ensembles of simple neural networks. In this work, we significantly extend the experimental evaluation presented in \citep{masegosa2020learning} and apply this analysis to regression and weighted majority vote ensembles. 



\section{Preliminaries}
\label{sec:preliminaries}

Let $D=\{(\bmx_1,y_1),\ldots,(\bmx_n,y_n)\}$ be a set of independent and identically distributed data samples according to an unknown distribution $\nu$ over ${\cal X} \times {\cal Y}$, where $\cal{X}$ and $\cal{Y}$ are arbitrary. Let $h_\bmtheta = h(\cdot \ ;\bmtheta)$ denote a single predictor, and $h_\bmtheta(\bmx) = h(\bmx;\bmtheta)$ its prediction for an input \(\bmx \in \cal X\). We assume that predictors are parameterized by $\bmtheta\in \bmTheta$ where \(\bmTheta\) is composed by a finite set of such parameter vectors, \(\bmTheta = \{\bmtheta_1, \dots, \bmtheta_K\}\) \footnote{Appendix \ref{app:4.4} shows that is not a restrictive setting indeed and, also, how to handle  $\bmTheta\subseteq \Re^M$.}.

\emph{Regression ensembles} are defined using the $\rho$-weighted model average predictor and the \emph{squared error} loss, denoted as the $\MSE$-loss. Thus, for an specific data-point \((\bmx, y)\), the loss of an individual regressor is defined as $\ell_{\MSE}(\bmtheta,\bmx,y)=(y - h(\bmx;\bmtheta))^2$ and the loss of this ensemble is defined as $\ell_{\MSE}(\rho, \bmx, y) = (y- \E_\rho [h(\bmx, \bmtheta)])^2$, where the subscript ``$sq$'' is used to distinguish these losses from the rest.


\emph{Weighted majority voting ensembles} are defined using the $\rho$-weighted majority vote predictor and the zero-one loss, denoted as \emph{$\zeroone$-loss}. Thus, for an specific data-point \((\bmx, y)\), the loss of an individual classifier is defined as $\ell_{0/1}(\bmtheta, \bmx, y) = \1[ h(\bmx, \bmtheta) \neq y]$ and the loss of this ensemble is defined as $\ell_{0/1}(\rho, \bmx, y) = \1[\arg \max_{y'} \E_\rho [\1(h(\bmx, \bmtheta)=y')] \neq y]$.


\emph{Model averaging ensembles} are defined using the $\rho$-weighted model average predictor and the \emph{cross entropy loss} (a.k.a. log-loss), denoted as the $\CE$-loss. In this ensemble, the individual models are probabilistic classifiers whose output is a conditional distribution over the class labels ${\cal Y}$ given the sample $\bmx$, i.e. $h(\bmx;\bmtheta)=p(\cdot|\bmx)$. Thus, for an specific data-point \((\bmx, y)\), the loss of an individual predictor is defined as $\ell_{\CE}(\bmtheta, \bmx, y)=-\log p(y|\bmx, \bmtheta)$ and the loss of this ensemble is defined as $\ell_{\CE}(\rho, \bmx, y) = -\log \E_\rho [p(y|\bmx, \bmtheta)]$.


 For any of these loss functions, denoted generically by $\ell(\bmtheta,\bmx,y)$, the empirical loss of an individual model $h_\bmtheta$ over the dataset $D$ is defined by $\hat{L}(\bmtheta, D) = \frac{1}{n}\sum_{i=1}^n \ell(\bmtheta,\bmx,y)$  and the population-level or expected counterpart is the expected loss of $h_\bmtheta$ defined by $L(\bmtheta) = \E_\nu[\ell(\bmtheta,\bmx,y)]$ where the expectation $\E_\nu$ is over the random choice of $(\bmx, y) \sim \nu$. Similarly, we can define the expected loss of an ensemble, denoted $L(\rho)$, as $L(\rho) = \E_\nu[\ell(\rho,\bmx,y)]$, for any of the previous ensemble loss functions $\ell(\rho,\bmx,y)$. 
 
 Throughout the rest of the paper, the terms $\L(\rho)$, $L(\bmtheta)$ and $\hat{\L}(\bmtheta,D)$ usually appear without any subscript. In this case, we highlight that the performed analyses apply to any of these three ensemble settings. When the corresponding subscript is attached (i.e. $\MSE$, $\zeroone$ and $\CE$), we refer to the specific ensemble model.

\section{Diversity and Generalization}

\subsection{Decomposing the Loss of an Ensemble Using an Upper Bound}

In this section, we introduce the following upper bounds of an ensemble's expected loss $L(\rho)$ as a way to decompose this error function. These upper bounds are expressed in terms of the $\rho$-average loss of the individual models $\E_\rho[L(\bmtheta)]$ and our newly proposed diversity measure, denoted as $\mathbb{D}(\rho)$.

\begin{theorem}
\label{thm:decomposition}
Under the settings given in Section \ref{sec:preliminaries}, we have that
\[L(\rho) \leq \alpha(\E_\rho[\L(\bmtheta)] - \mathbb{D}(\rho)), \]
\noindent where $\alpha$ equals $1$ if we consider the $\MSE$-loss or the $\CE$-loss, and $4$ if we consider the $\zeroone$-loss. Furthermore, for the $\MSE$-loss, this inequality becomes an equality. The expression of the diversity measure for each of these loss functions is:
    \begin{eqnarray*}
    \mathbb{D}_{\MSE}(\rho) &=& \E_\nu\Big[\V_\rho(h_R(\bmx;\bmtheta))\Big],\\
    \mathbb{D}_{\CE}(\rho) &=& \E_\nu\left[\V_\rho\left(\frac{p(y\mid\bmx,\bmtheta)}{\sqrt{2} \max_\bmtheta p(y\mid\bmx,\bmtheta)}\right)\right],\\
    \mathbb{D}_{\zeroone}(\rho) &=& \E_\nu\Big[\V_\rho\Big(\1(h_W(\bmx;\bmtheta)\neq y)\Big)\Big],
    \end{eqnarray*}
\noindent where $\V_\rho(\cdot)$ denotes the variance of a function w.r.t. the data generating distribution $\nu$, and for $\mathbb{D}_{\CE}(\rho)$ to be well-defined we need that $0<\max_{\bmtheta\in \bmTheta} p(y \mid \bmx,\bmtheta)\leq 1$ for every \((\bmx, y) \in supp(\nu)\). Finally, note that all diversity terms described above can be written as 
    \[
        \mathbb{D}(\rho) = \E_\nu \Big[ \mathbb{V}_{\rho} \left( f(y,\bmx;\bmtheta) \right)\Big],
    \]
with a specific function \(f\) for each of the loss functions. 
\end{theorem}

The $\MSE$-version of Theorem \ref{thm:decomposition} is equivalent to the well-known decomposition of the squared error of a regression ensemble \citep{krogh1995validation}. On the other hand, the $\CE$-version is equivalent to the one previously proposed by \citep{masegosa2020learning} based on second-order Jensen inequalities \citep{becker2012variance,liao2019sharpening}. In Appendix \ref{app:4.1}, we also detail a tighter variant of this inequality proposed by \citep{masegosa2020learning}. Lastly, the $\zeroone$-version of Theorem \ref{thm:decomposition} is novel and based on the analysis given by \citep{masegosa2020second}, which, in turn, is based on second-order Markov inequalities.  

Even though the three versions of this upper-bound are derived using completely different approaches, here we show that their expressions are surprisingly similar. Showing that the loss of an ensemble $L(\rho)$ can be upper bounded by (or is equal to) the $\rho$-average loss of the individual models $\E_\rho[L(\bmtheta)]$ minus the diversity $\mathbb{D}(\rho)$ among the individual models of the ensemble.


\subsection{How to Measure the Diversity of an Ensemble?}

In this work, we propose the use of the diversity term $\mathbb{D}(\rho)$ given in Theorem \ref{thm:decomposition} as a diversity measure of an ensemble. We start showing that this diversity measure satisfies some intuitive properties:
\begin{lemma}
\label{lemma:varianceproperties}
The diversity terms $\mathbb{D}(\rho)$ defined in Theorem \ref{thm:decomposition} satisfy the following properties:
\begin{enumerate}[i)]
    \item If all the ensemble members provide the same predictions, or if \(\rho\) outs all its probability mass on a single predictor, then $\mathbb{D}(\rho)$ is null.
    \item $0\leq \mathbb{D}(\rho) \leq \E_\rho[\L(\bmtheta)]$.
    \item $\mathbb{D}(\rho)$ is invariant to reparametrizations.
\end{enumerate}
\end{lemma}


The above properties show that $\mathbb{D}(\rho)$ can be considered as a measure of the diversity of an ensemble. The first property follows the intuition of diversity as a measure of the difference among ensemble's members errors, while the second property is something which has been empirically found in the literature: the diversity of an ensemble usually decreases when the predictive error of individual members is reduced \citep{fort2019deep}. The last property is a desirable result for any diversity measure. 

Another common knowledge about diversity is that it decreases when the predictors are highly correlated \citep{BK16,brown2009information, zhou2012ensemble,yu2011diversity}. The following result shows how this diversity measure nicely captures this relationship of diversity and correlation among predictors:

\begin{theorem}
\label{theorem:diversity_as_variance}
The diversity terms $\mathbb{D}(\rho)$ defined in Theorem \ref{thm:decomposition} can be written as
    \begin{align*}
        \mathbb{D} (\rho) &= \V_{\nu\times\rho}\Big(f(y,\bmx;\bmtheta)\Big)\\
            &\quad- \E_{\rho\times\rho}\Big[Cov_{\nu}(f(y,\bmx;\bmtheta),f(y,\bmx;\bmtheta'))\Big], 
    \end{align*}
    \noindent where $\rho\times\rho$ denotes the joint distribution over $\bmTheta \times\bmTheta$, $\rho\times\nu$ denotes the joint distribution over $\bmTheta \times({\cal\mathbf{X}},{\cal Y})$, and $Cov_{\nu}(\cdot,\cdot)$ is the co-variance between two models with respect to the data generating distribution $\nu$.
\end{theorem}

A proof of this novel result is provided in Appendix \ref{app:4.2}. This result states that our diversity measure increases as we reduce the correlation among ensembles. Eventually, we will have much higher diversity if ensembles are anti-correlated (i.e. negative covariance). But that above result also introduces a novel insight: ensemble diversity is no only about the correlation among individual models. The above decomposition of the diversity also shows, through the $\V_{\nu\times\rho}\Big(f(y,\bmx;\bmtheta)\Big)$ term, that ensembles with high diversity should provide different predictions across the different individual models and across the different data samples. Although this is out the scope of this paper, this last term could potentially be used to study why randomization approaches to build ensembles (e.g. Random Forests \cite{breiman2001random}) give rise to high diverse ensembles, because they directly try to maximize this term, which is positively related to diversity. 




    
    
    
    
    
    

\subsection{How is Diversity Related to the Performance of an Ensemble?}
\label{sec:diversitygeneralization}

The role that diversity plays in the performance of an ensemble is described by Theorem \ref{thm:decomposition}. According to this result, the generalization error of an ensemble $L(\rho)$ should be reduced if we increase $\mathbb{D}(\rho)$, which is a measure of the diversity of the ensemble as shown in the previous section. However, Theorem \ref{thm:decomposition} provides additional novel insights. 

The next result formalizes a empirically observed phenomenon \citep{dietterich2000ensemble,lu2010ensemble} that higher ensemble diversity induces a higher gap between the average loss of the individual models (i.e., $\E_\rho[\L(\bmtheta)]$)  and the expected loss of the ensemble (i.e. $\L(\rho)$). In other words, the higher the diversity, the higher is the advantage of combining these models. 
\begin{corollary}
\label{collorary:gap}
Under the settings given in Section \ref{sec:preliminaries}, we have that
\[\mathbb{D}(\rho) \leq \E_\rho[\L(\bmtheta)] - \frac{1}{\alpha}\L(\rho),\]
\noindent where $\alpha$ is equal to $1$ if we consider the $\MSE$-loss or the $\CE$-loss, and $4$ if we consider the $\zeroone$-loss. For the $\MSE$-loss, this inequality becomes an equality. 
\end{corollary}

Another open question in the ensemble's literature is under which situations an ensemble of models outperforms a single model. Next result establishes that this occurs when the ensemble's diversity is large enough.

\begin{corollary}
\label{collorary:singlemodel}
Under the settings given in Section \ref{sec:preliminaries}, we have that an ensemble of models weighted according to a distribution $\rho$ performs better than a single model $\bmtheta^\star$, i.e. $\L(\rho)< \L(\bmtheta^\star)$, if 
\[\E_\rho[\L(\bmtheta)] - \frac{1}{\alpha} L(\bmtheta^\star) < \mathbb{D}(\rho)\]
\noindent where $\alpha$ is equal to $1$ if we consider the $\MSE$-loss or the $\CE$-loss, and $4$ if we consider the $\zeroone$-loss. For the $\MSE$-loss, the inverse implication also holds.
\end{corollary}

However, $\mathbb{D}(\rho)$ is defined in terms of the unknown data-generating distribution $\nu$. As a result, $\mathbb{D}(\rho)$ can not be computed. To address this issue, we propose  the use of the empirical version of $\mathbb{D}(\rho)$, denoted by $\hat{\mathbb{D}}(\rho,D)$, which directly depends on the empirical distribution defined by the data sample $D$. $\hat{\mathbb{D}}(\rho,D)$ satisfies the same properties as $\mathbb{D}(\rho)$, as shown in Appendix \ref{app:4.3}, and measures the diversity of an ensemble in a given data sample $D$. The relevant question now is how this empirical diversity measure relates to the generalization performance of an ensemble. To answer this question we will rely on PAC-Bayesian bounds. 


PAC-Bayesian analysis \citep{McA98,See02,LST02} is based on \textit{Probably Approximate Correct} upper bounds over the generalization error of a model. These PAC-Bayes bounds depend on the empirical error of the model. More precisely, in this work, we consider a new family of PAC-Bayesian bounds which upper bounds the generalization error of an ensemble (i.e. $\L(\rho)$) based on the $\rho$-weighted empirical errors of the individual models ($\E_\rho[\hat\L(\bmtheta,D)]$) and the empirical diversity of their predictions ($\hat{\mathbb{D}}(\rho,D)$). PAC-Bayesian  bounds hold with high probability over random realizations of the training data sample. 




\begin{theorem}[\textbf{PAC-Bayes bounds}]
\label{thm:2ndPACBayes}For any prior distribution $\pi$ over $\bmTheta$ independent of $D$ and for any $\xi\in (0,1)$ and any $\lambda>0$, with probability at least $1-\xi$ over draws of training data $D\sim \nu^n$, for all distribution $\rho$ over $\bmTheta$, simultaneously,
\[L(\rho) \leq \alpha\Big(\E_\rho[\hat L(\bmtheta,D)] - \hat{\mathbb{D}}(\rho,D) + \frac{2\kl{\rho}{\pi}+   \epsilon}{\lambda\, n}\Big)\label{eq:pac:general},\]
\noindent where $\alpha$ is equal to $1$ if we consider the $\MSE$-loss or the $\CE$-loss, and $4$ is we consider the $\zeroone$-loss.
$KL$ refers to the Kullback-Leibler divergence between $\rho$ and the prior $\pi$. And $\epsilon>0$ is a function of \(\nu,\pi,\lambda,n\) and \(\xi\), which is independent of $\rho$ but also depends on the specific loss (In Appendix \ref{app:4.3}, we detail the functional forms of these $\epsilon$ terms).
\end{theorem}

The $\MSE$ and $\zeroone$ versions of the PAC-Bayesian bound of Theorem \ref{thm:2ndPACBayes} are novel. While the $\CE$-version of Theorem \ref{thm:2ndPACBayes} was previously proposed in \citep{masegosa2020learning}.

The preceding result provides a clear theoretical explanation for the widely observed phenomena of accurate and diversified models leading to ensemble methods with low generalization error \citep{dietterich2000ensemble}. High accurate models imply lower values of $\E_\rho[\hat\L(\bmtheta,D)]$, while highly diverse models implies higher values of $\hat{\mathbb{D}}(\rho,D)$. Consequently, the combined effect of both, as described by the second-order PAC-Bayesian bounds of Theorem \ref{thm:2ndPACBayes}, induces a lower upper bound over the generalization error of the ensemble.

\subsection{How to Exploit Diversity to Learn Ensembles?}
\label{sec:diversity:learning}

The PAC-Bayesian bounds of Theorem \ref{thm:2ndPACBayes} provide a sound framework for ensemble learning. As they simultaneously hold for all distributions $\rho$, we can choose the distribution $\rho$ which minimizes these high-probability bounds over the generalization error of the ensemble \citep{McA98,See02,LST02}. In Appendix~\ref{app:4.4}, we provide a complete description of how this approach precisely applies to ensembles of neural networks (i.e how each distribution $\rho$ defines an ensemble of neural networks), following the ideas introduced by \citep{masegosa2020learning}. 

As discussed in Section \ref{sec:rw}, a growing number of ensemble learning approaches  employs learning objectives which, besides promoting individual models with low error (i.e. small  $\E_{\rho}[\hat{L}(\bmtheta,D)]$), they also include a term which explicitly promotes diversity. In consequence, this work provides a theoretical justification for these approaches: promoting diversity can improve the generalization performance of the ensemble. We elaborate on this analysis in Appendix \ref{app:4.4}.

This is not the case for ensembles of deep neural networks. In this context, individual models are large and operate in the interpolation regime \citep{zhang2016understanding} and, in consequence, the $\E_\rho[\hat{L}(\bmtheta,D)]$ (present in Theorem~\ref{thm:2ndPACBayes}) will be very close to zero, if not null. Moreover, using Lemma \ref{lemma:varianceproperties}, the empirical diversity term $\hat{\mathbb{D}}(\rho,D)$ of this bound will be very close to zero too. As a result, ensemble learning algorithms based on the direct minimization of Theorem~\ref{thm:2ndPACBayes} become useless because the diversity term ($\hat{\mathbb{D}}(\rho,D)$) will hardly influence the learning process.  In Appendix~\ref{app:4.4}, we provide a mathematical formulation of this issue and, in the next section, we also provide empirical evidence.

Given this, we seem to arrive to a contradiction because, according to this analysis, having high diversity is essential for the generalization performance of an ensemble. And, at the same time, ensembles of deep neural networks have almost null empirical diversity and strong generalization performance. However, the key point is that having low diversity over the training data (i.e. $\hat{\mathbb{D}}(\rho,D)\approx 0$) does not imply we will not have diversity over the test data (i.e. $\mathbb{D}(\rho)>0$), which is what we really want to have according to Theorem~\ref{thm:decomposition}. 

\begin{figure}
\centering
\begin{tabular}{cc}
  \hspace{-5pt}\includegraphics[width=0.5\linewidth]{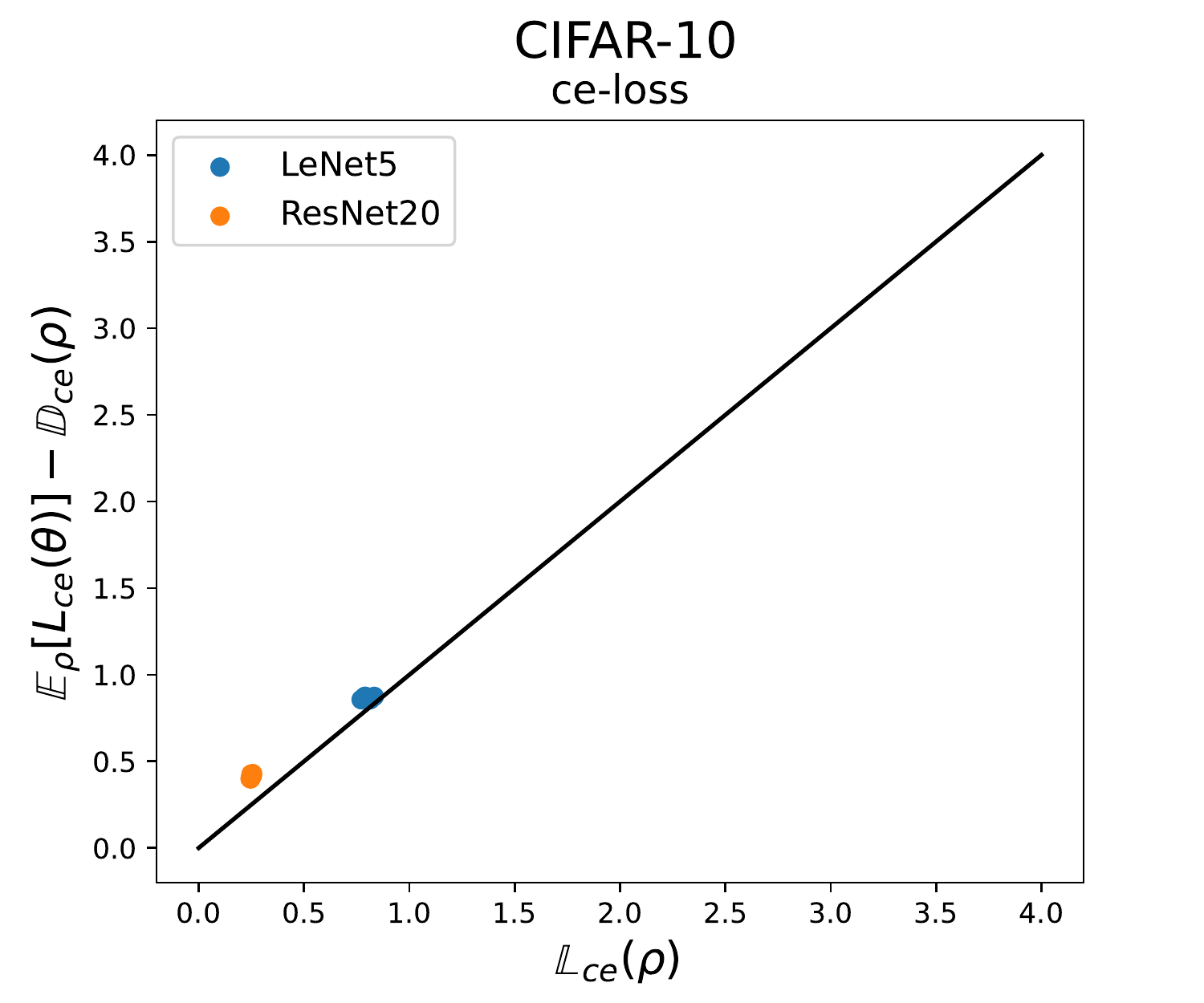}
  & 
  \hspace{-5pt}\includegraphics[width=0.5\linewidth]{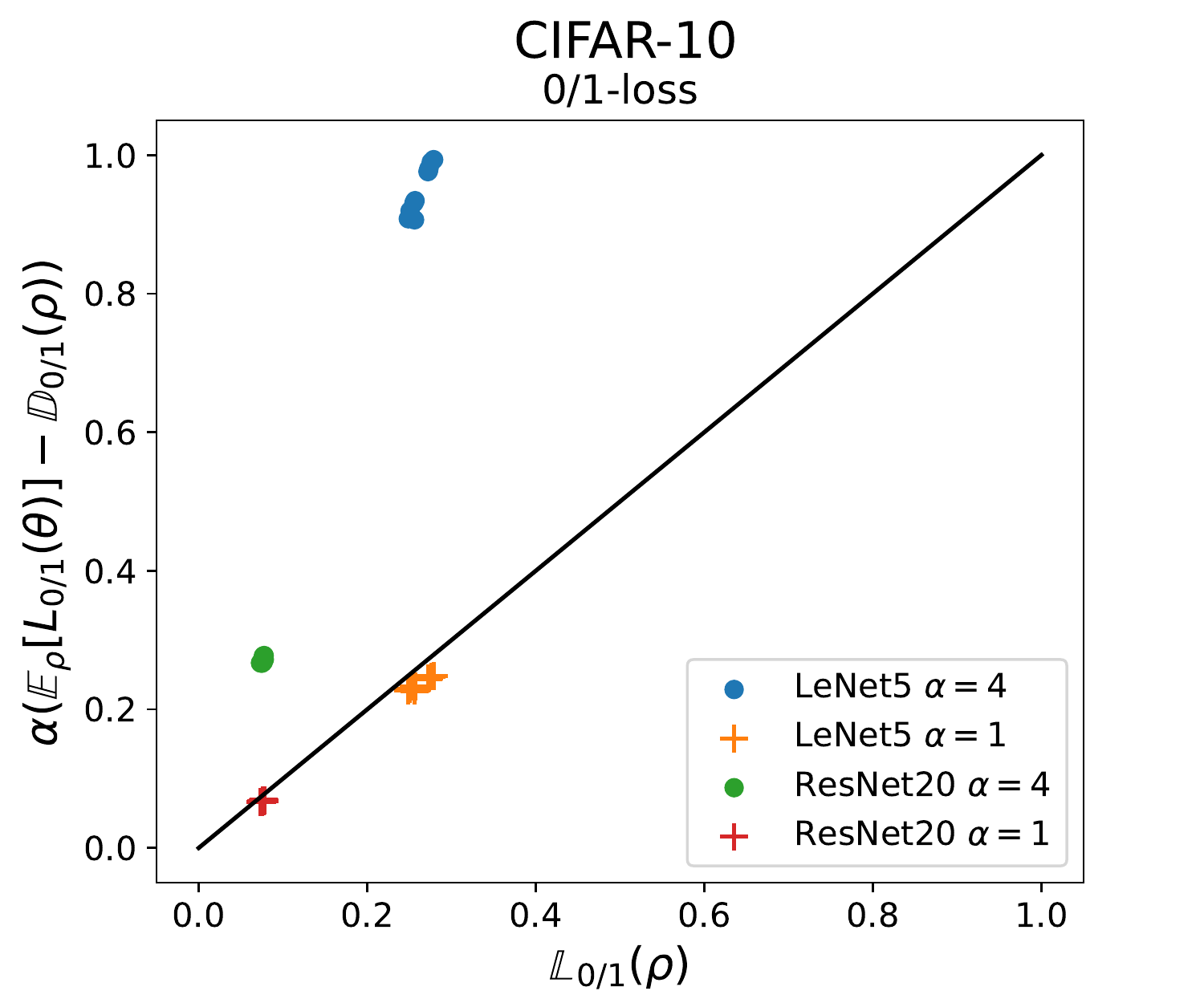}
  
  \\
  \hspace{-5pt}\includegraphics[width=0.5\linewidth]{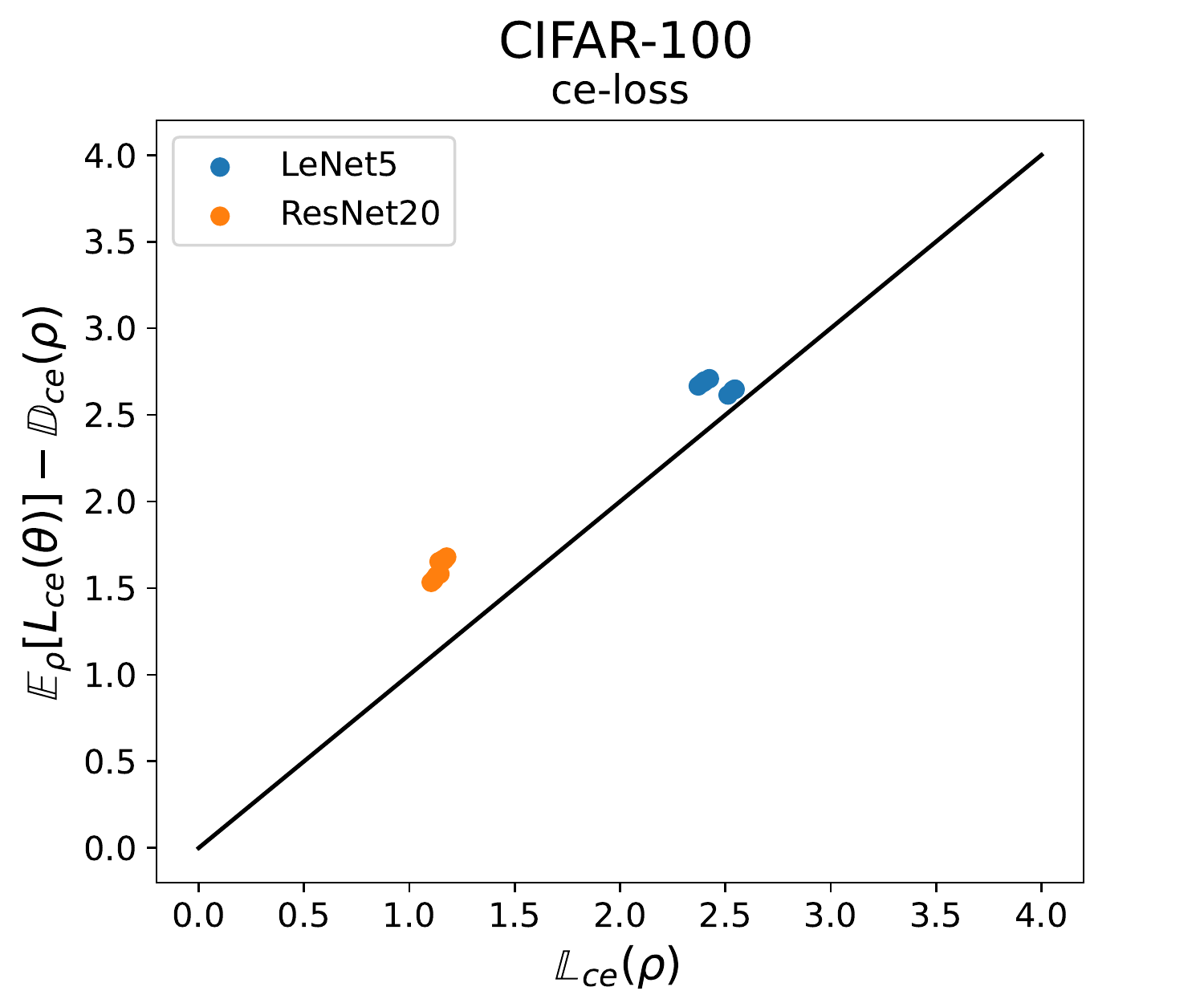}
  & 
  \hspace{-5pt}\includegraphics[width=0.5\linewidth]{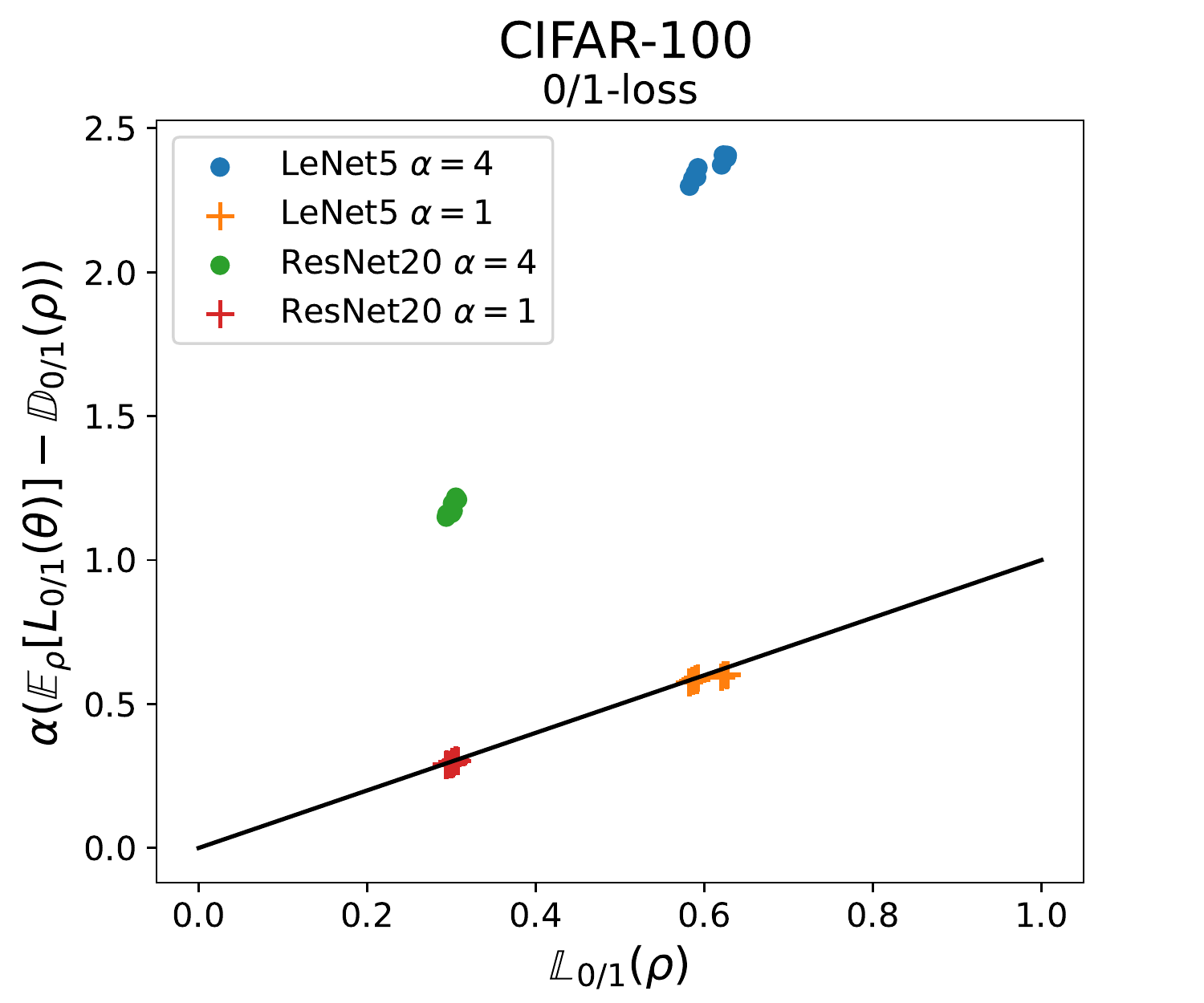}
  
\end{tabular}
\caption{\textbf{Evaluation of Theorem \ref{thm:decomposition}}. Each point is an ensemble model learn with \textit{Ensemble} or \textit{PAC2B-Ensemble}. The distance to the black line measures the tightness of the upper bounds of Theorem \ref{thm:decomposition}.} \label{fig:thm:decomposition}
\end{figure}

Most of the current state-of-the-art deep ensemble learning algorithms follow this general scheme: they independently learn each neural network of the ensemble by minimizing the provided loss function (usually the $\CE$-loss or the $\MSE$-loss) using some randomization method (e.g. random initialization of the parameters  \citep{lakshminarayanan2017simple,wen2019batchensemble} , or different hyper-parameters for the gradient descent algorithm \citep{wenzel2020hyperparameter}, etc.) in order to force the gradient descent algorithm to converge to different local minimum of the loss function. Current state-of-the-art deep ensemble learning algorithms exploit the highly multi-modal landscape of the loss function \citep{fort2019deep} to achieve that. 

In consequence, when the ensemble is composed by $K$ models $\{\bmtheta_1,\ldots,\bmtheta_K\}$ defining different predictive functions, then the expected diversity $\mathbb{D}(\rho)$ will be positive, as stated in the following result:

\begin{lemma}
If  there exists $\bmtheta_i\neq \bmtheta_j$ and an input sample $\bmx\in supp(\nu)$, where $supp(\nu)$ denotes the support of the data generating function, such that $h(\bmx;\bmtheta_i)\neq h(\bmx;\bmtheta_j)$, we then have that $\mathbb{D}(\rho)>0$. 
\end{lemma}

This analysis shows that ensembles of deep neural network promote diversity by learning neural networks which induce different predicitve functions. And this is achieved, in general, by using randomization strategies that exploits the highly multi-modal landscape of the loss function of deep neural networks . In the next section, we empirically illustrate this analysis. But the evidence already given in some previous works \citep{fort2019deep} clearly aligned with these conclusions.

\section{Experimental Evaluation}\label{sec:experiments}

\begin{figure*}[!hbt]
\centering
\begin{tabular}{ccc}
\multicolumn{1}{c}{\hspace{-15pt}\includegraphics[width=0.35\linewidth]{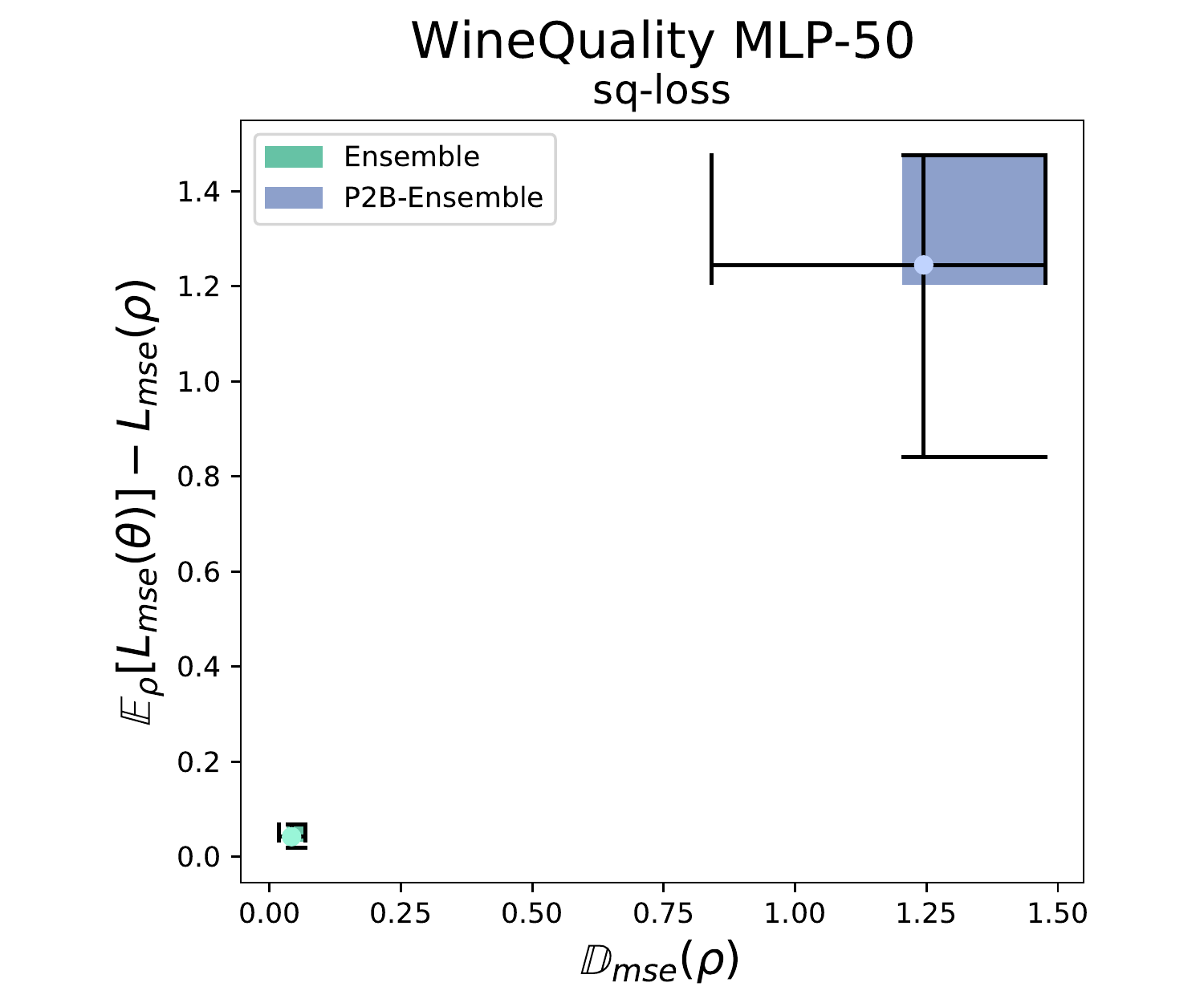}}
&
\multicolumn{1}{c}{\hspace{-15pt}\includegraphics[width=0.35\linewidth]{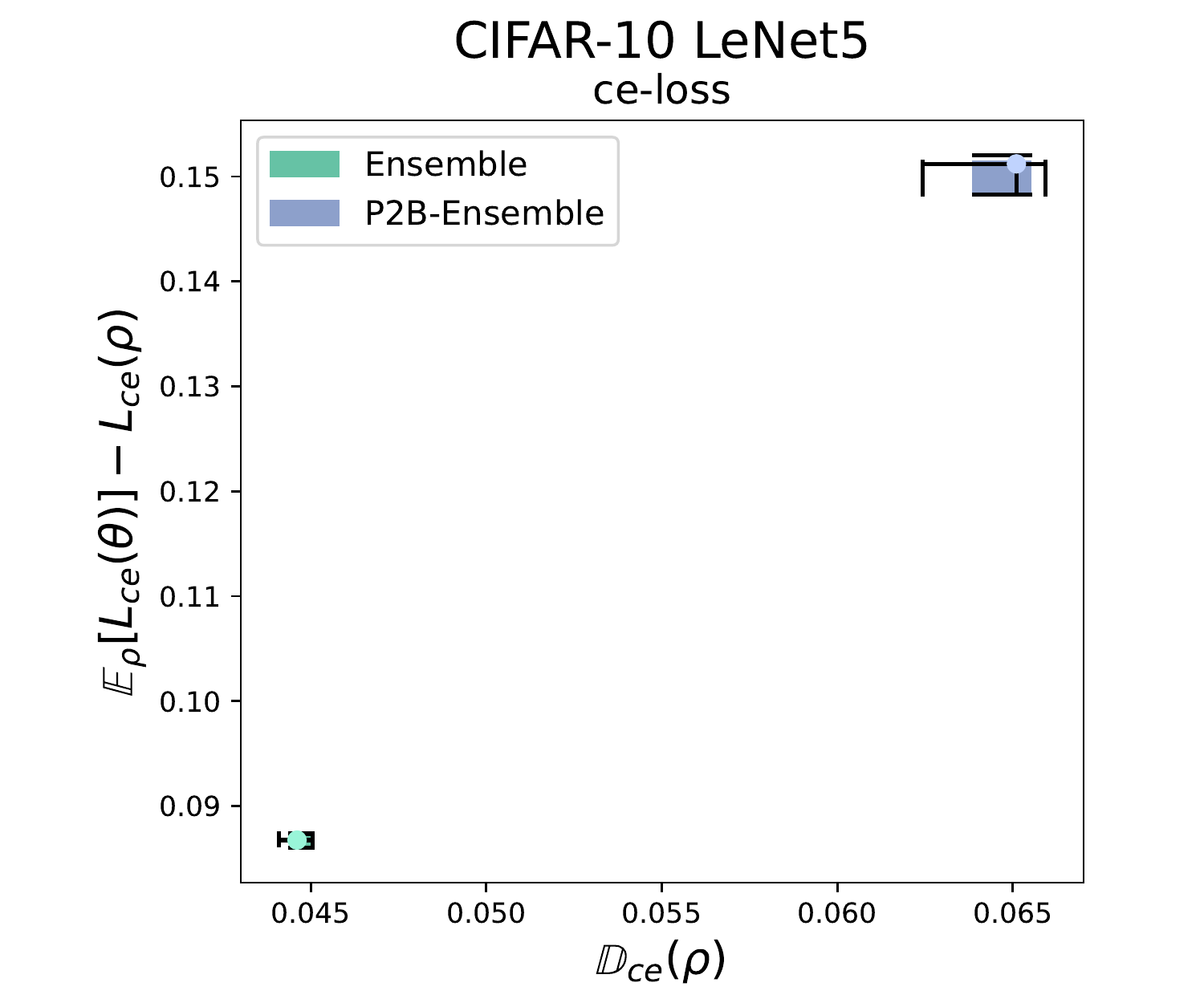}}
&
\multicolumn{1}{c}{\hspace{-15pt}\includegraphics[width=0.35\linewidth]{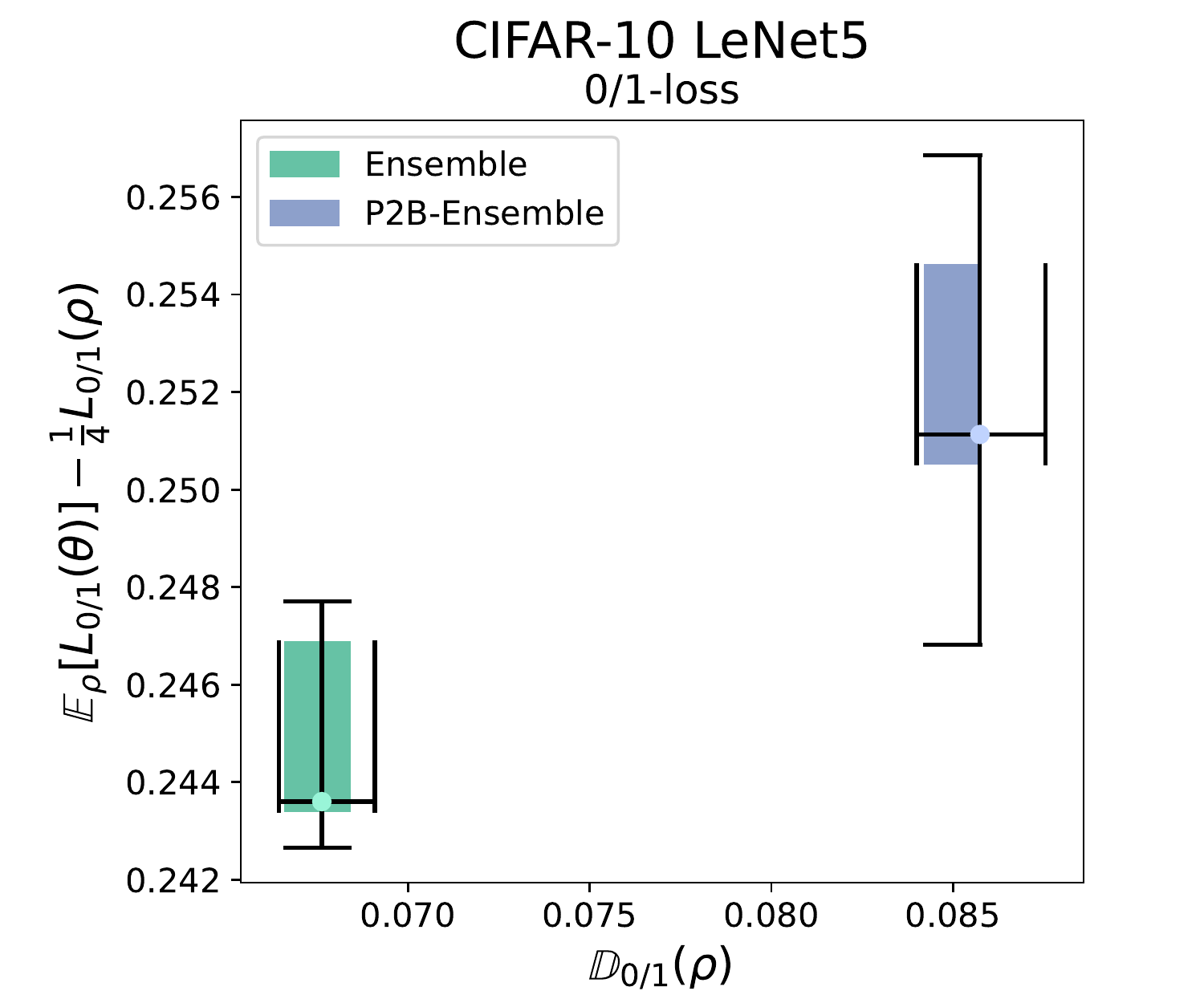}}
\\
\multicolumn{1}{c}{\hspace{-15pt}\includegraphics[width=0.35\linewidth]{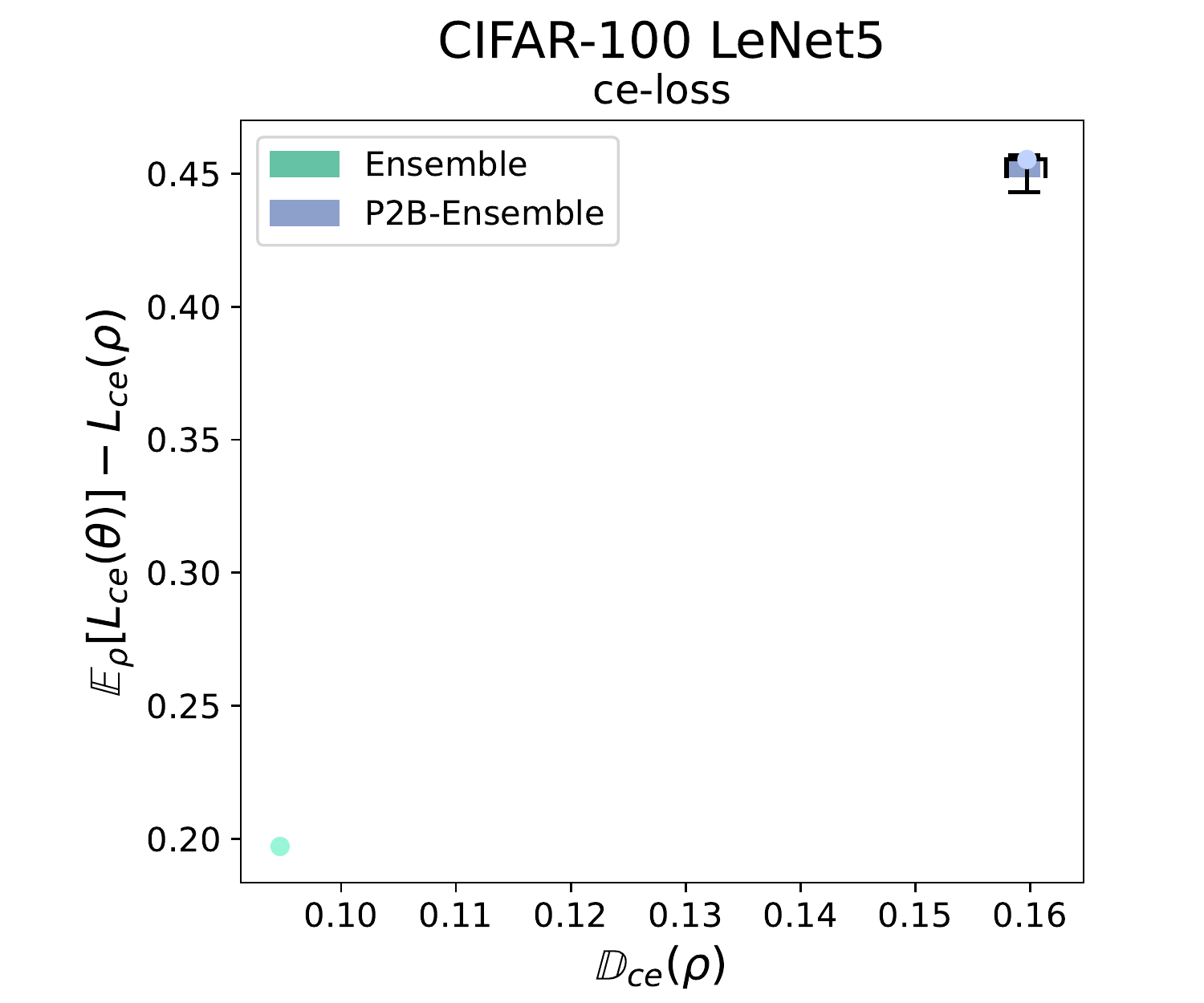}}
&
\multicolumn{1}{c}{\hspace{-15pt}\includegraphics[width=0.35\linewidth]{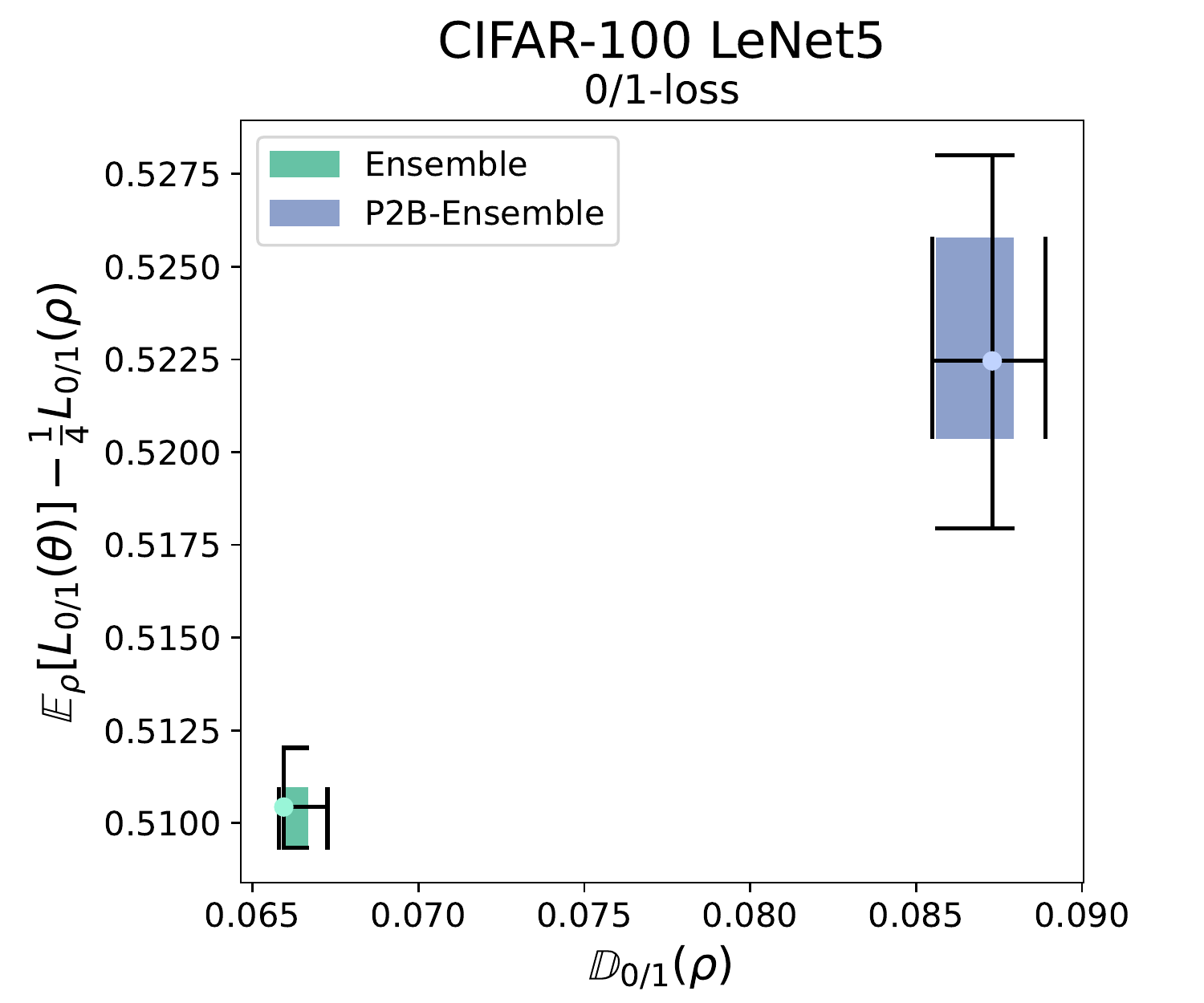}}
&
\multicolumn{1}{c}{\hspace{-15pt}\includegraphics[width=0.35\linewidth]{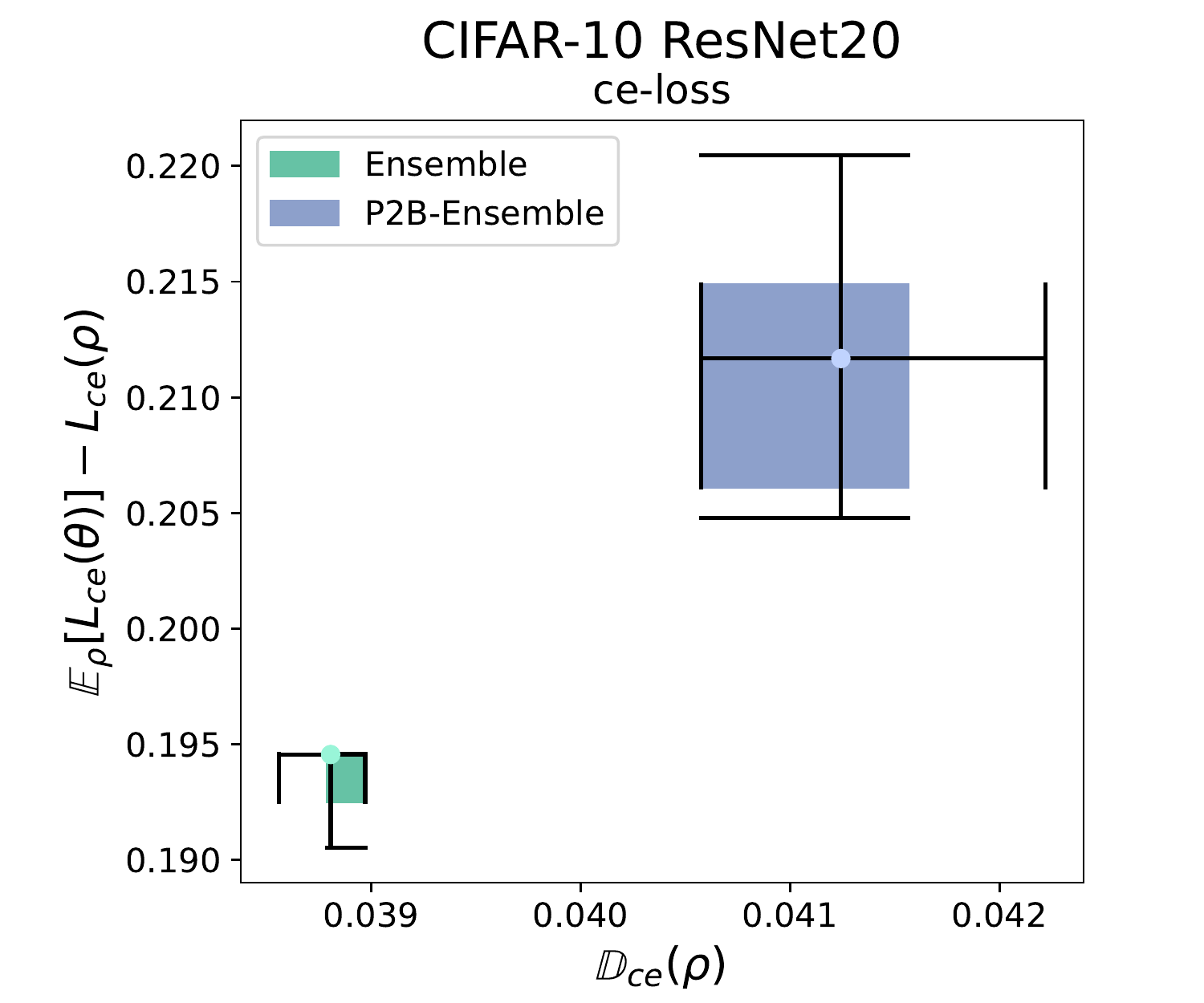}}
\\
\multicolumn{1}{c}{\hspace{-15pt}\includegraphics[width=0.35\linewidth]{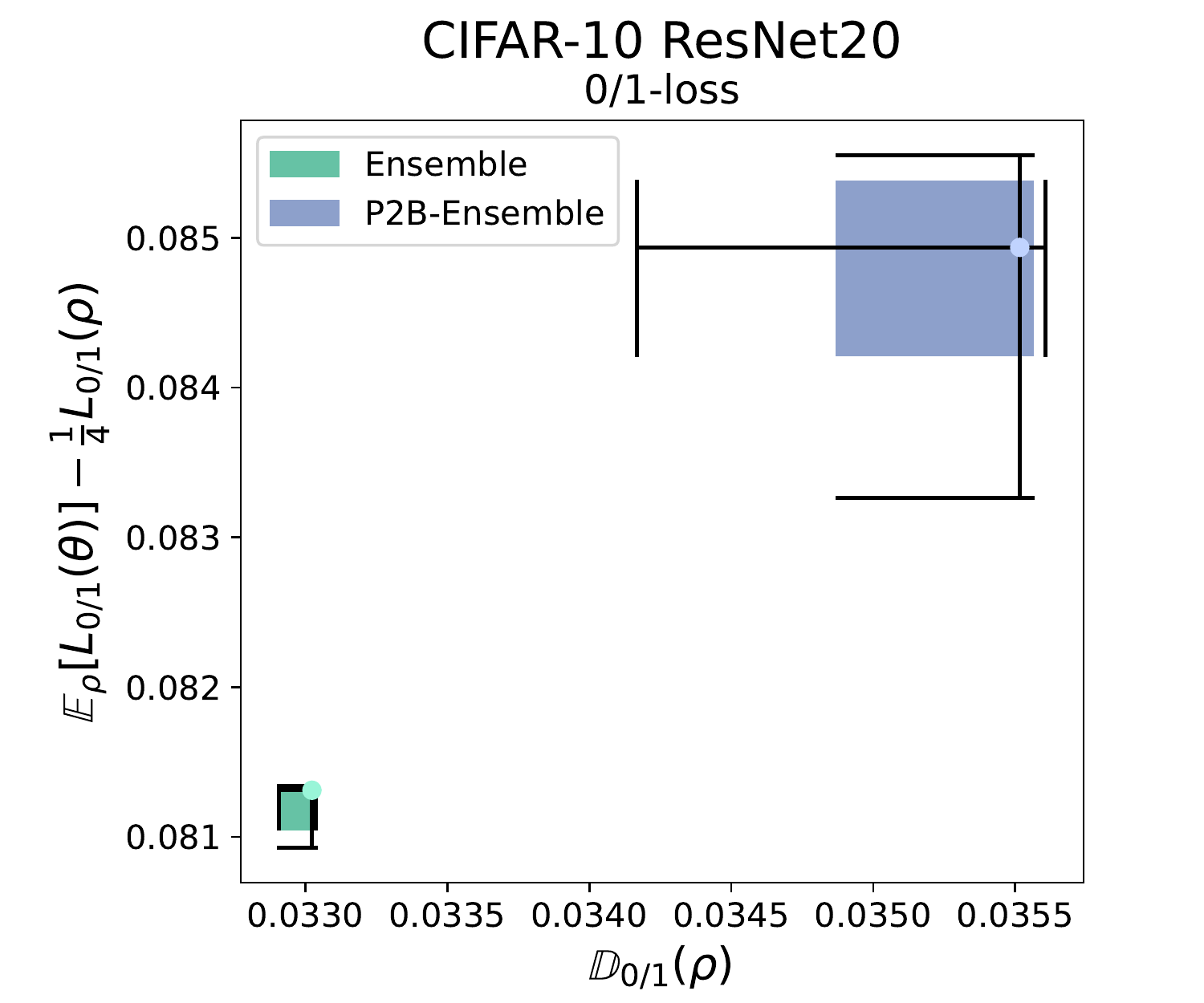}}
&
\multicolumn{1}{c}{\hspace{-15pt}\includegraphics[width=0.35\linewidth]{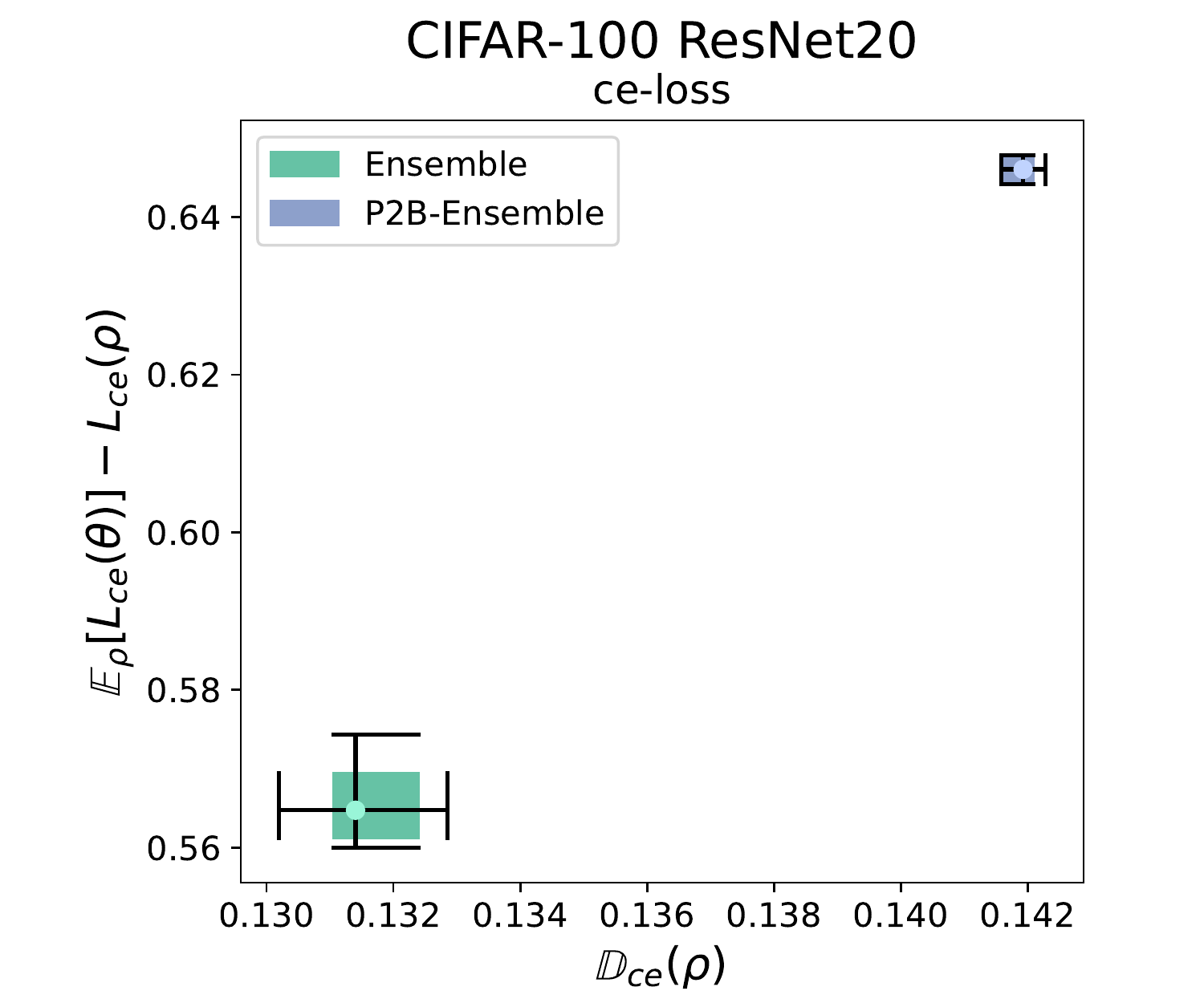}}
&
\multicolumn{1}{c}{\hspace{-15pt}\includegraphics[width=0.35\linewidth]{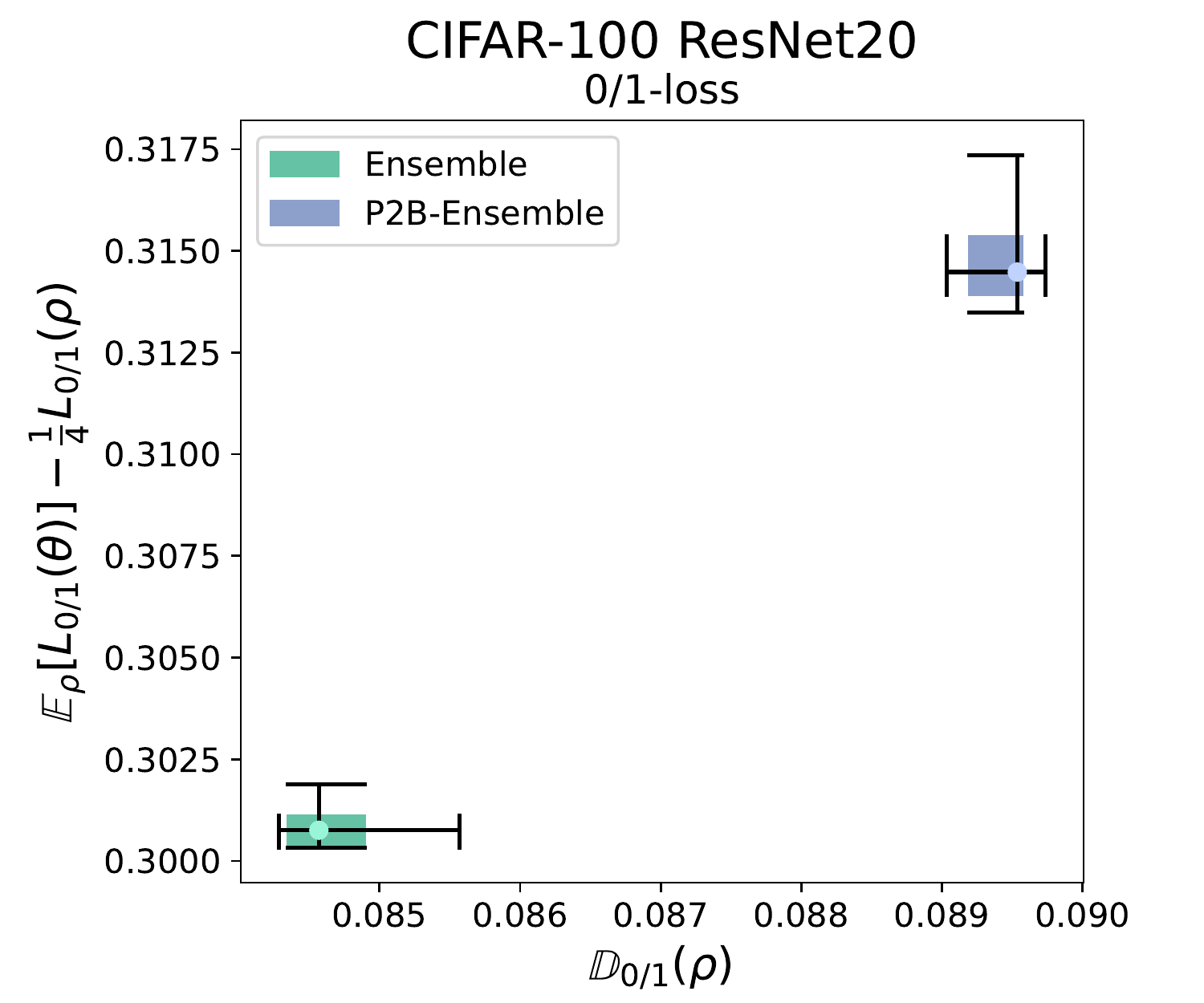}}
\\
\end{tabular}
  \caption{Each box-plot represents a set of ensemble models learned with the same algorithm. The X-axis represents the diversity of the ensemble and the Y-axis represents the gap between the loss of the individual ensemble models and the loss of the ensemble (a positive gap indicates that the ensemble performs better than the individual models).}\label{fig:collorary:gap}
\end{figure*}

\subsection{Experimental Settings}

We performed the empirical evaluation for the three ensembles detailed in Section \ref{sec:preliminaries}. Note that regression ensembles are associated to the $\MSE$-loss, weighted majority vote ensembles are associated to the $\zeroone$-loss, and model averaging ensembles to the $\CE$-loss. We will use these losses as a way to refer to the different ensembles (e.g. the generalization error of a regression ensemble will be denoted by $\L_{\MSE}(\rho)$).

The regression ensemble is evaluated on the Wine-Quality \citep{cortezpaulo} data set using a  multilayer-perceptron with one layer containing 50 hidden units and a dropout layer, this model is denoted as MLP50. 

The majority vote and the model averaging ensembles are evaluated on two standard data sets, CIFAR-10 and CIFAR-100 \citep{krizhevsky2009learning}, using two networks: LeNet5 \citep{lecun1989backpropagation} and Resnet20 \citep{he2016deep}. LeNet5 is chosen for its simplicity, meaning it does not operate on the interpolation regime (i.e. the empirical error is not close to zero). On the other hand, we use ResNet20 because it operates in (or close to) the interpolation regime for the CIFAR-10 and the CIFAR-100 data sets. More complex networks could have been employed, but for the aim of this empirical evaluation ResNet20 is powerful enough. 

All ensembles are made of four individual models. This is not an arbitrary amount as this setting is the default one in the widely used neural network ensemble library \emph{Uncertainty Baselines}\footnote{\url{https://github.com/google/uncertainty-baselines}}. However, Appendix \ref{app:experiments} shows results for other ensemble sizes.

\begin{figure*}[!htb]
\centering
\begin{tabular}{ccc}
\multicolumn{1}{c}{\hspace{-15pt}\includegraphics[width=0.35\linewidth]{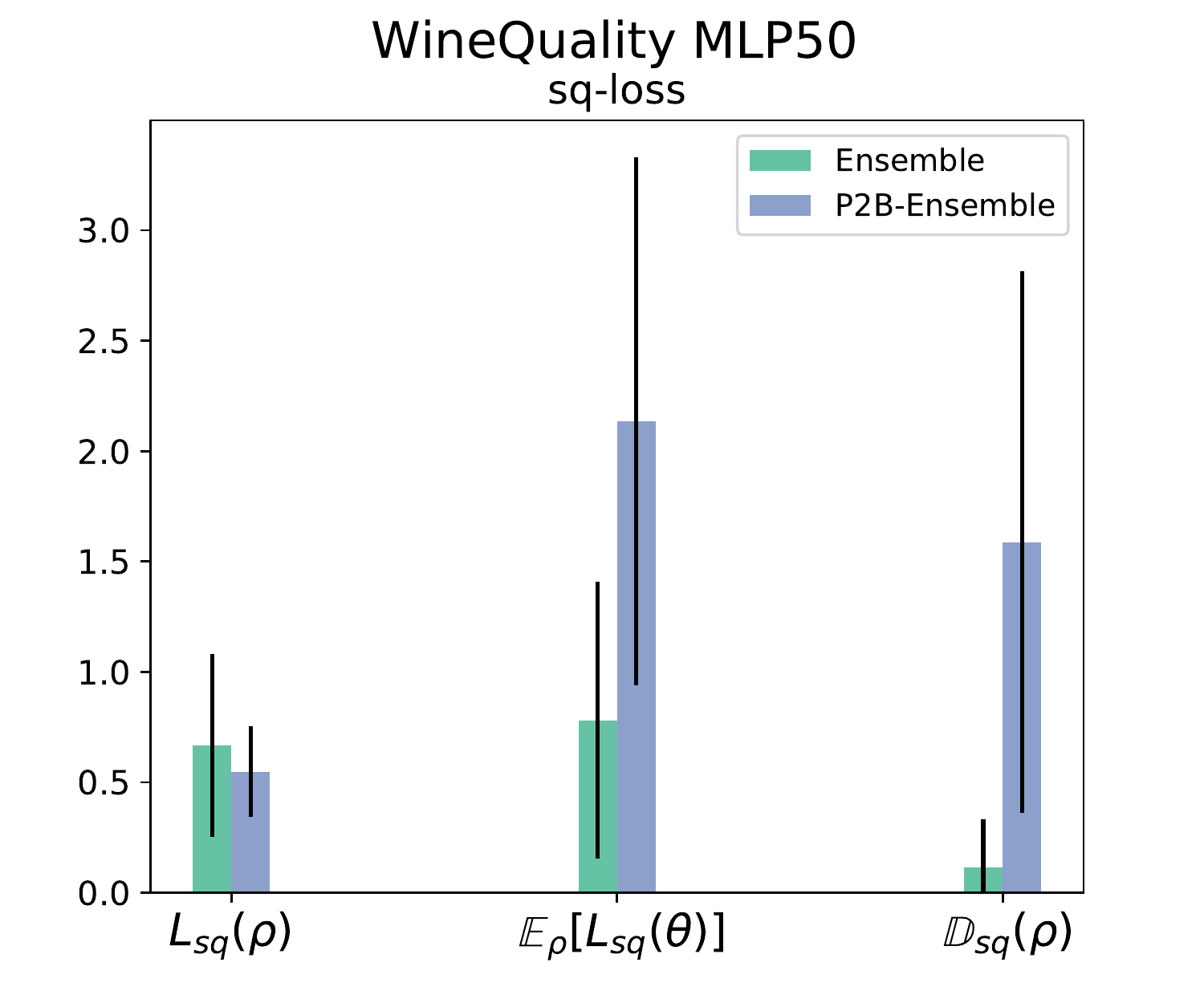}}
&
\multicolumn{1}{c}{\hspace{-15pt}\includegraphics[width=0.35\linewidth]{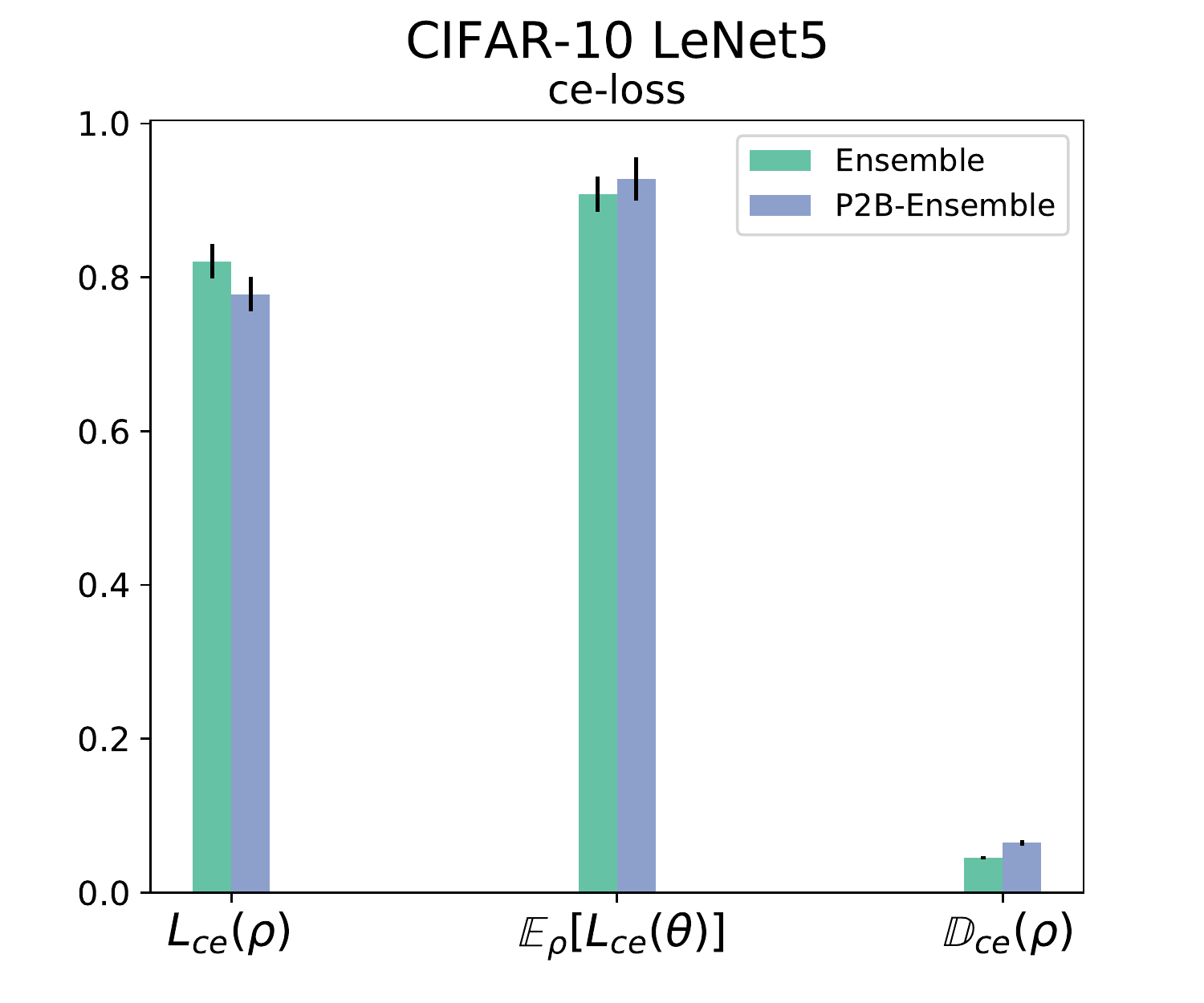}}
&
\multicolumn{1}{c}{\hspace{-15pt}\includegraphics[width=0.35\linewidth]{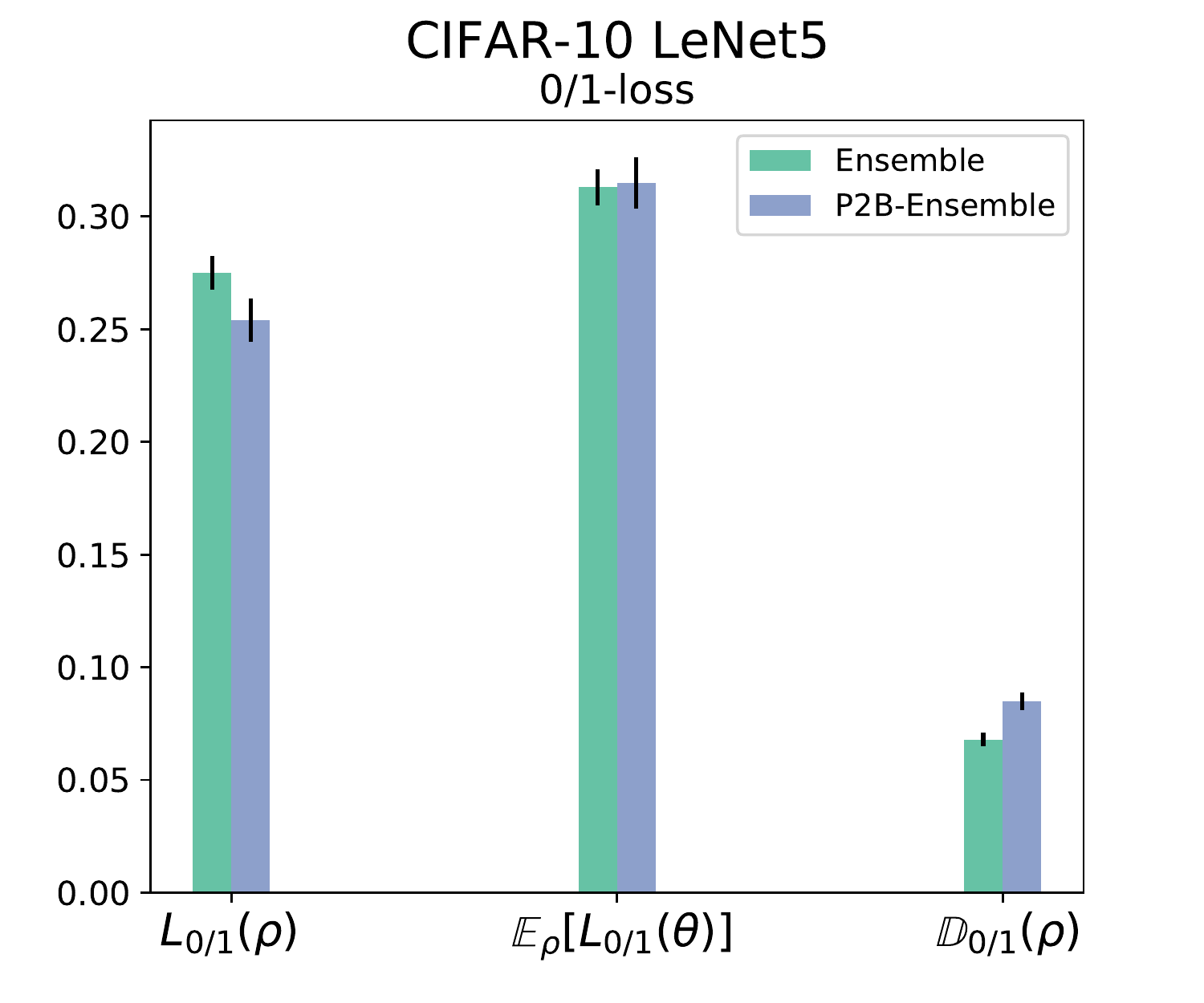}}
\\
\multicolumn{1}{c}{\hspace{-15pt}\includegraphics[width=0.35\linewidth]{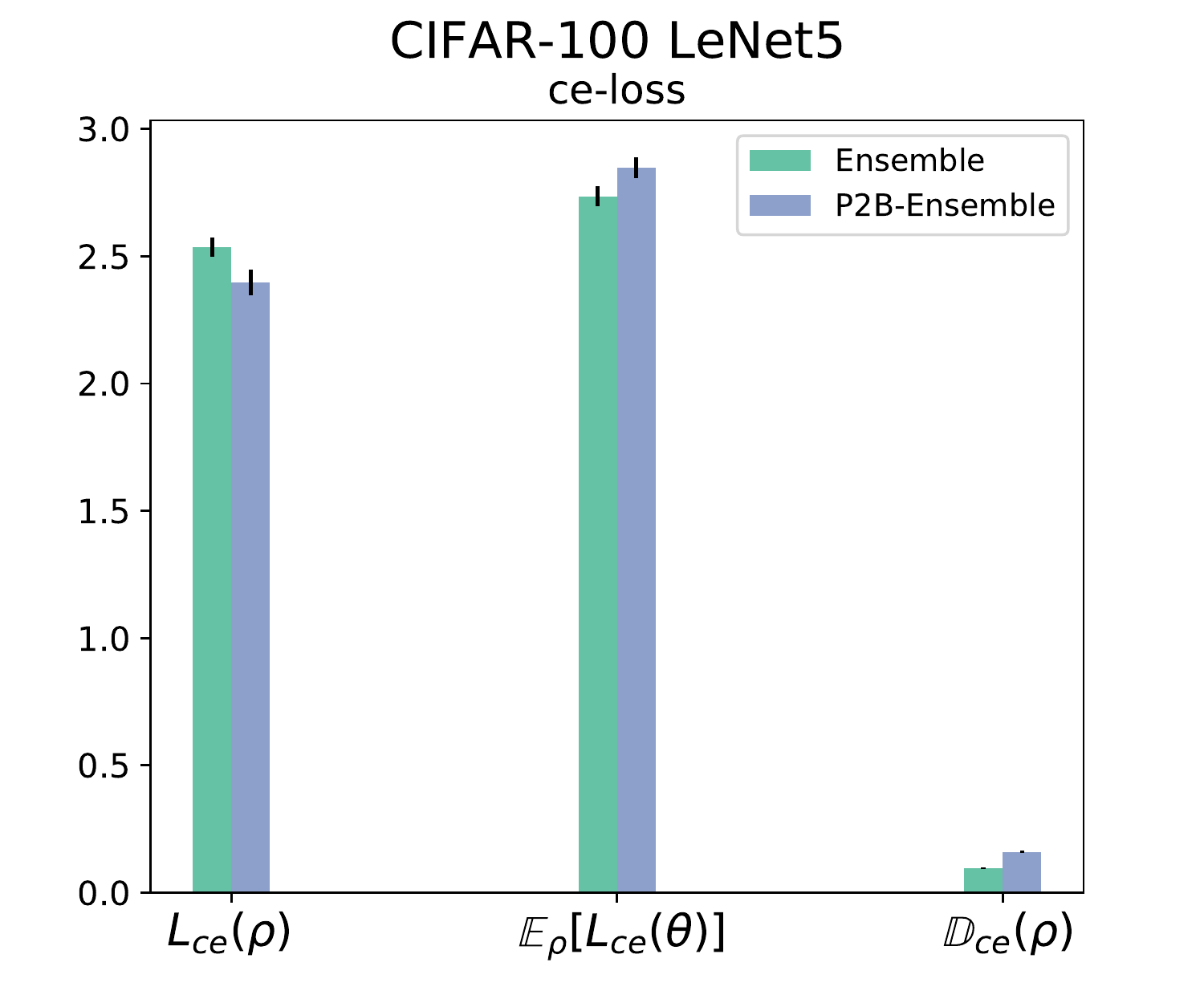}}
&
\multicolumn{1}{c}{\hspace{-15pt}\includegraphics[width=0.35\linewidth]{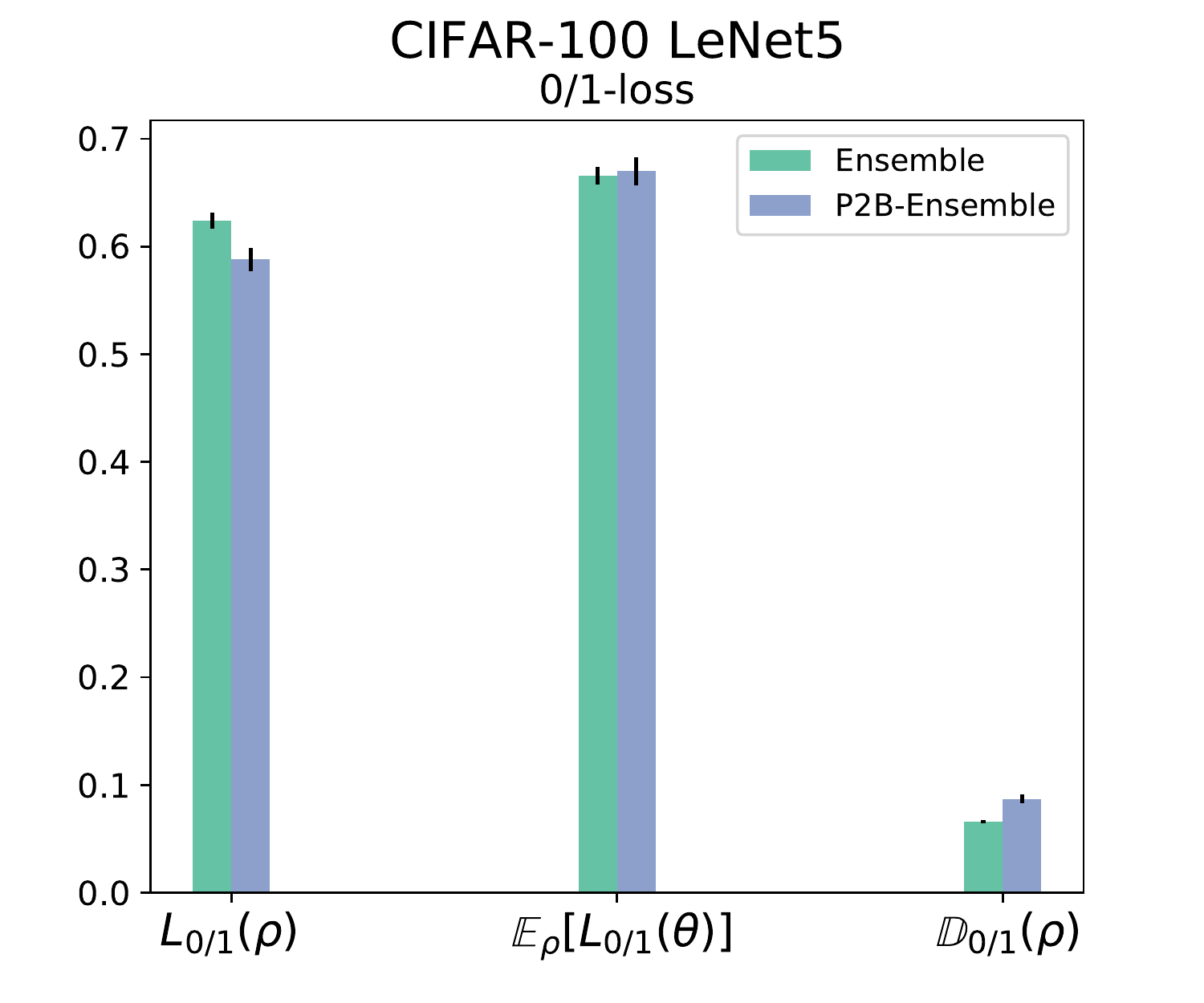}}
&
\multicolumn{1}{c}{\hspace{-15pt}\includegraphics[width=0.35\linewidth]{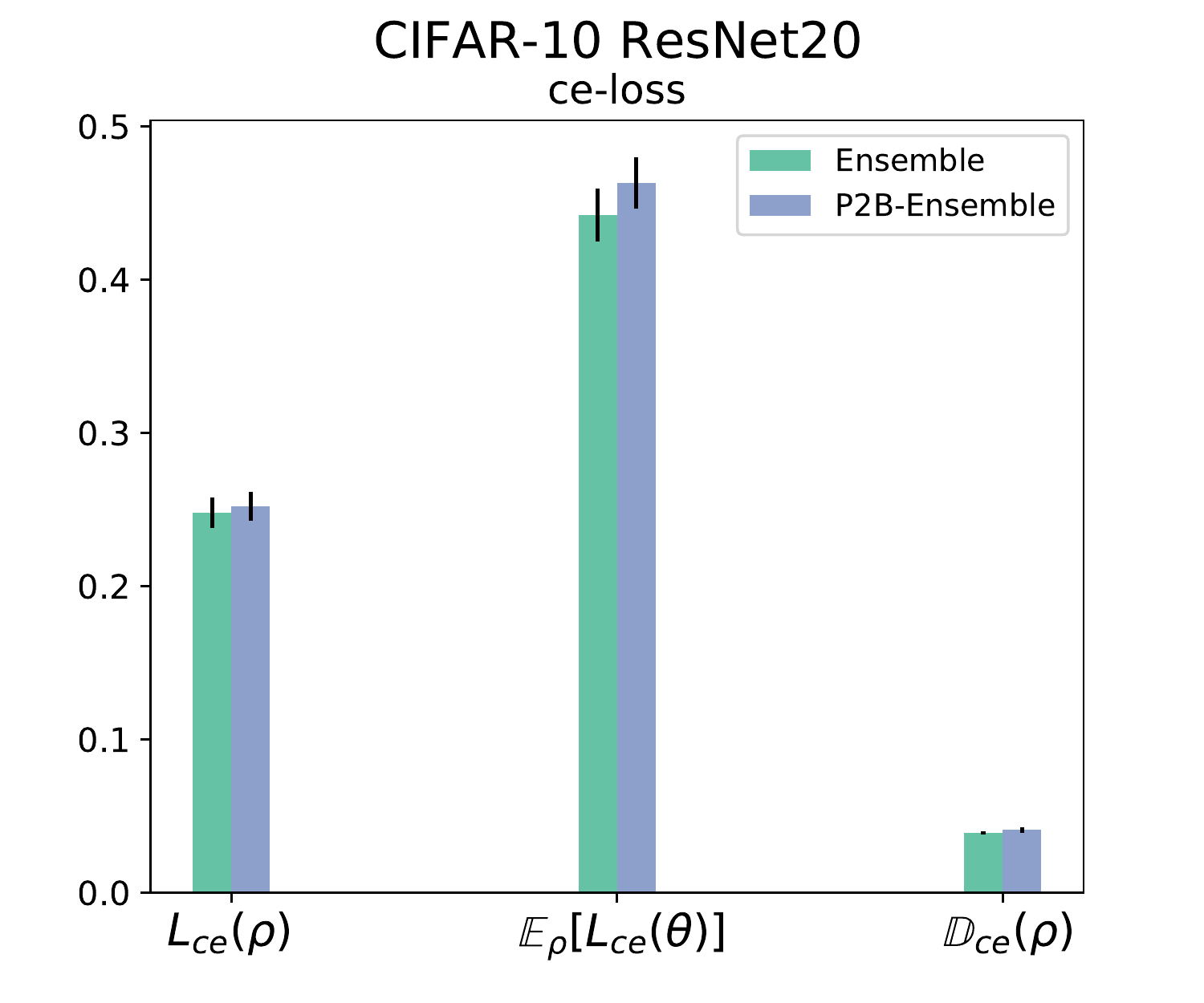}}
\\
\multicolumn{1}{c}{\hspace{-15pt}\includegraphics[width=0.35\linewidth]{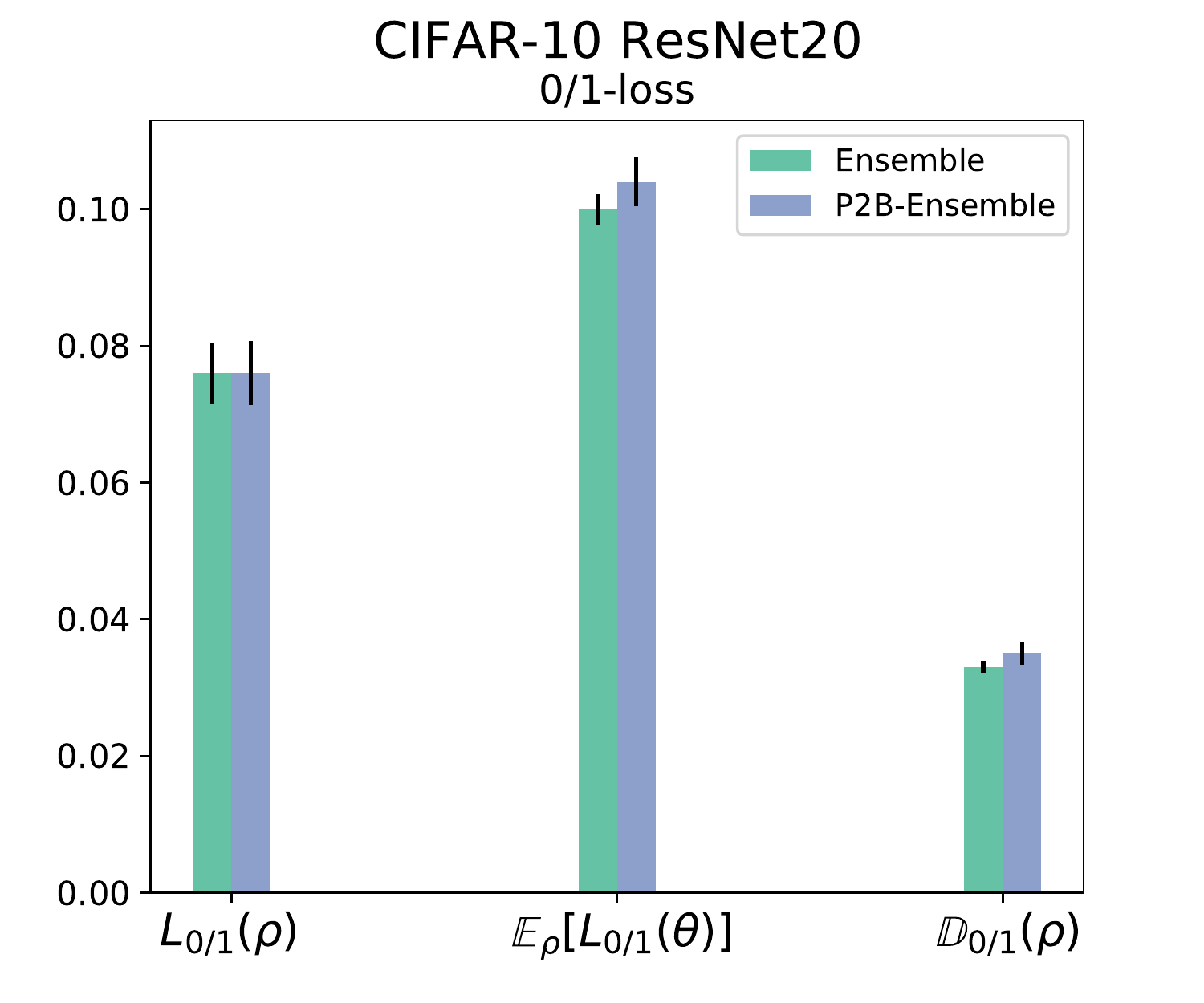}}
&
\multicolumn{1}{c}{\hspace{-15pt}\includegraphics[width=0.35\linewidth]{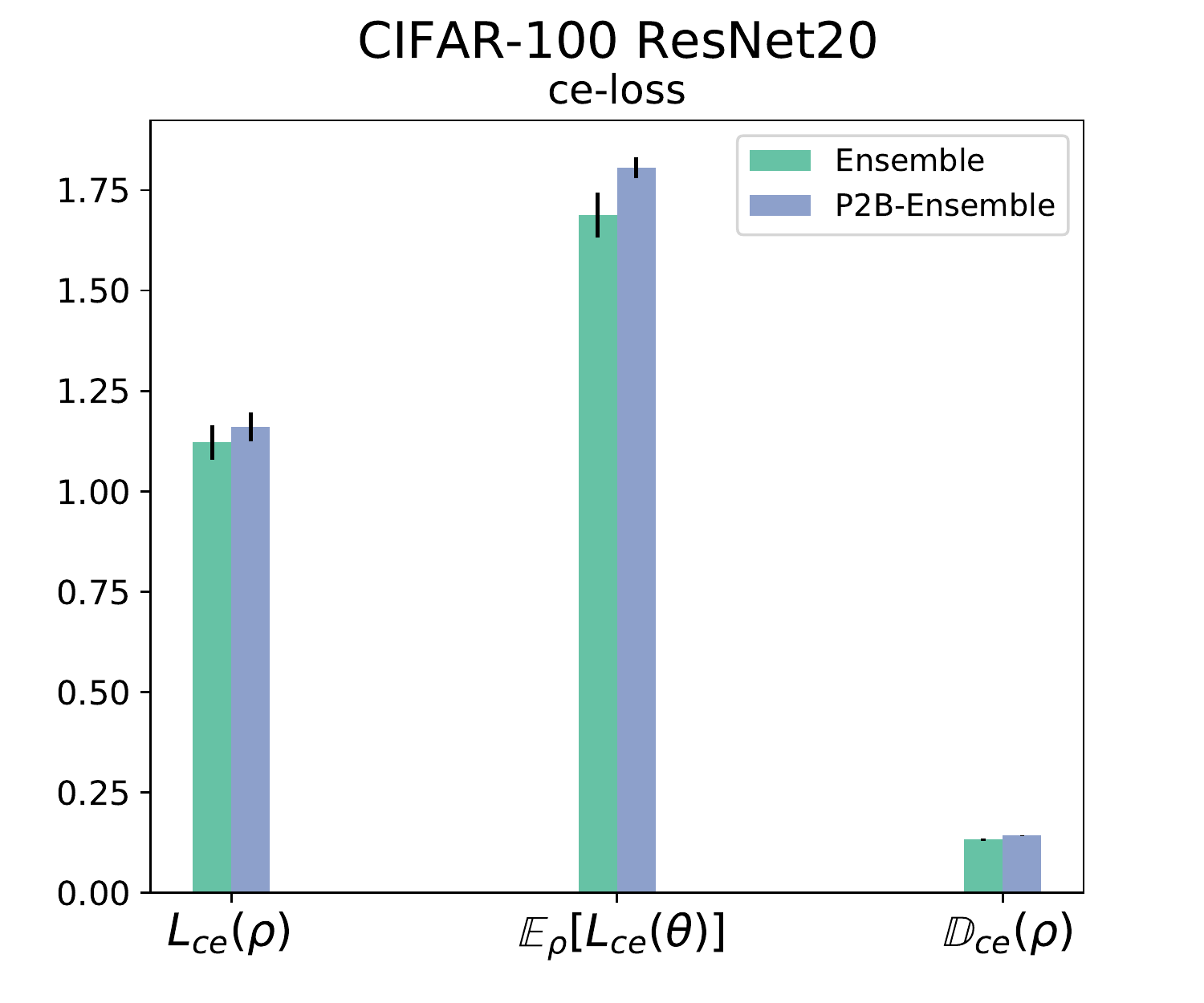}}
&
\multicolumn{1}{c}{\hspace{-15pt}\includegraphics[width=0.35\linewidth]{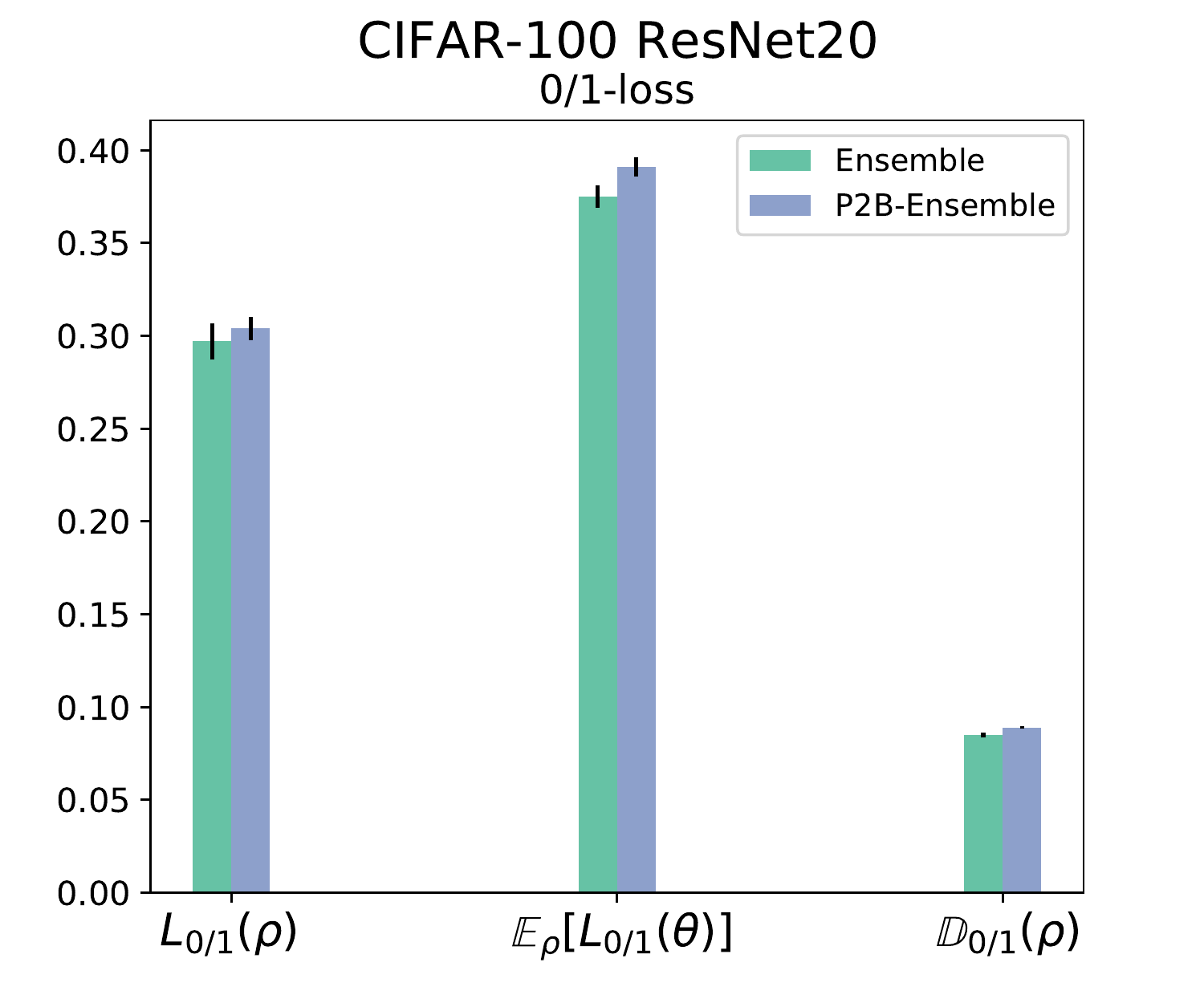}}
\\

\end{tabular}
  \caption{Mean plus/minus three standard deviations of the ensemble test error $L(\rho)$, average test error of the individual models $\E_\rho[L(\bmtheta)]$ and ensemble diversity in the test set $\mathbb{D}(\rho)$.}\label{fig:PACLearning}
\end{figure*}

We consider two ensemble learning algorithms; \emph{ensemble} \citep{lakshminarayanan2017simple}, where each randomly initialized model is independently optimized using gradient descent, and \emph{P2B-Ensemble}  \cite{masegosa2020learning}, where the randomly initialized models are jointly learned minimizing the PAC-Bayes bound of Theorem~\ref{thm:2ndPACBayes}. The former is the best representative approach of neural network ensemble learning algorithms with a randomization approach \citep{lakshminarayanan2017simple}. In both cases, we use the $\CE$-loss for the CIFAR-10 and CIFAR-100 data sets, and the $\MSE$-loss for Wine-Quality data set. For learning the majority vote ensembles we optimize the $\CE$-loss instead of the $\zeroone$-loss because the $\zeroone$-loss is non-differentiable, and the $\CE$-loss is a good proxy. In any case, we can compute the error and diversity measures for this ensemble and analyse their relationship. 


We employ the test data set to approximate the generalization errors of the ensemble $\L(\rho)$, of the individual models $\L(\bmtheta)$ and, also, the expected diversity of the ensemble $\mathbb{D}(\rho)$. For the experiments with the $\CE$-loss, we use a more complex but tighter version of the diversity also introduced in \cite{masegosa2020learning} (see Appendix \ref{subsec:tighter_ineq} for details). We do not evaluate the ensembles using out-of-distribution benchmarks \citep{snoek2019can} because the aim of this experimentation is to empirically evaluate our theoretical analysis, which assumes that both the test and train data are generated from the same distribution.

Five executions with different seeds were made for each experiment. Full details of the experimental settings are given in Appendix \ref{app:experiments}. The needed code to reproduce the experiments is provided in the supplementary material, and is based on publicly available libraries.

\subsection{Experimental Evaluation}\label{sec:experimental_ev}

We start this evaluation by assessing how tight are the upper bounds provided by Theorem \ref{thm:decomposition}. Figure \ref{fig:thm:decomposition} show the results of this evaluation. Each point corresponds to an ensemble model, which is characterized by a given $\rho_\delta$ distribution. The distance of the points to the identity line measures the tightness of the upper bound of Theorem \ref{thm:decomposition}. In this results, we can see that for the $\MSE$-version, the upper bound is completely tight, as stated in the Theorem. For the $\CE$-version, the upper bound is also fairly tight too. For the $\zeroone$-version, we evaluate two versions. The original one with $\alpha=4$, which is quite loose. And another version with $\alpha=1$, which is not an upper bound, to illustrate that it seems plausible that there may exist the possibility to derive bounds with $\alpha$ values closer to 1. Because $\alpha=4$ is a worst case factor \citep{masegosa2020second}.

In Section \ref{sec:diversitygeneralization}, we discussed how, according to our theoretical analysis, diversity is directly related to the enhanced performance we get when combining different models. More precisely, as shown in Corollary~\ref{collorary:gap}, those ensembles with higher diversity $\mathbb{D}(\rho)$ should raise a higher gap between the average performance of the individual models $\E_\rho[L(\bmtheta)]$ and the performance of the ensemble $\L(\rho)$. Figure \ref{fig:collorary:gap} shows how high diversity is consistently linked to a higher performance gap (i.e. $\E_\rho[L(\bmtheta)] - L(\rho)$) in all the cases. This figure also shows how \textit{PAC2B-Ensemble}, which explicitly promotes diversity, consistently induces higher-diversity ensembles than the standard \textit{Ensemble} algorithm.

In Section \ref{sec:diversity:learning} we discussed why the \textit{P2B-Ensemble} algorithm, which minimizes the PAC-Bayes bound of Theorem~\ref{thm:2ndPACBayes}, learns better ensembles if the individual network models do not operate in the interpolation regime. By looking at the results of Figure \ref{fig:PACLearning}, we can see how this analysis matches our empirical findings:

\begin{itemize}
    \item The \textit{P2B-Ensemble} algorithm effectively induces ensembles with higher diversity that have better generalization performance than the \textit{Ensemble} algorithm for MLP50 and LetNet5.
    
    \item However, \textit{P2B-Ensemble} on ResNet20 does not induce ensembles with much higher diversity than the \textit{Ensemble} algorithm. Moreover, it performs similar or worst than the \textit{Ensemble} algorithm. 
\end{itemize}

Figure~\ref{fig:empirical_diversity} shows that the empirical error and empirical diversity terms of the bound of Theorem~\ref{thm:2ndPACBayes} for the ResNet20 ensemble is much smaller than for the LeNet5 ensemble. In that case, ResNet20 is operating close to the in the interpolation regime for CIFAR-10 and close to the interpolation regime for CIFAR-100.

\begin{figure}[!hbt]
\centering
\begin{tabular}{cccccc}
\multicolumn{2}{c}{\hspace{-5pt}\includegraphics[width=0.5\linewidth]{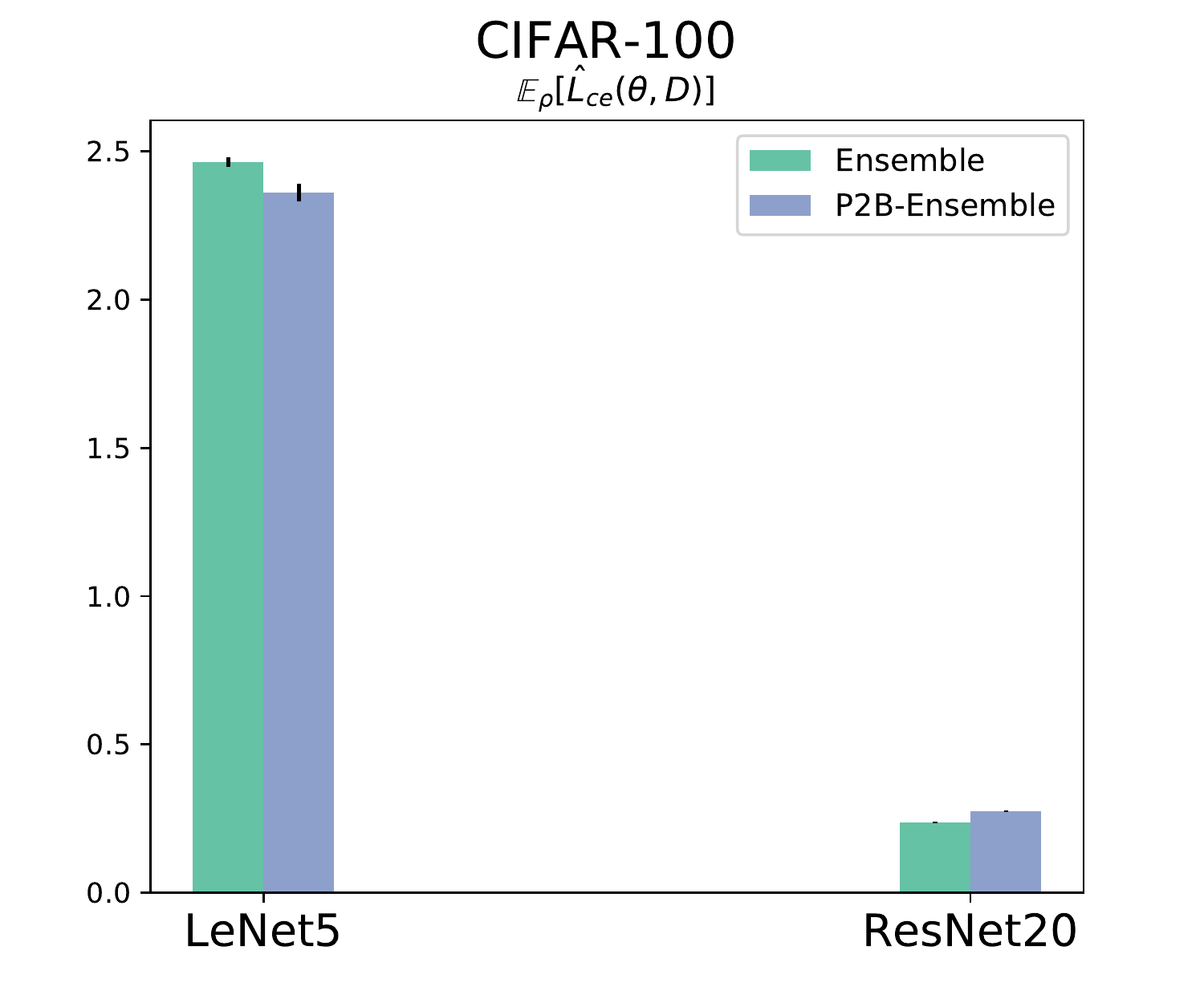}}
&
\multicolumn{2}{c}{\hspace{-5pt}\includegraphics[width=0.5\linewidth]{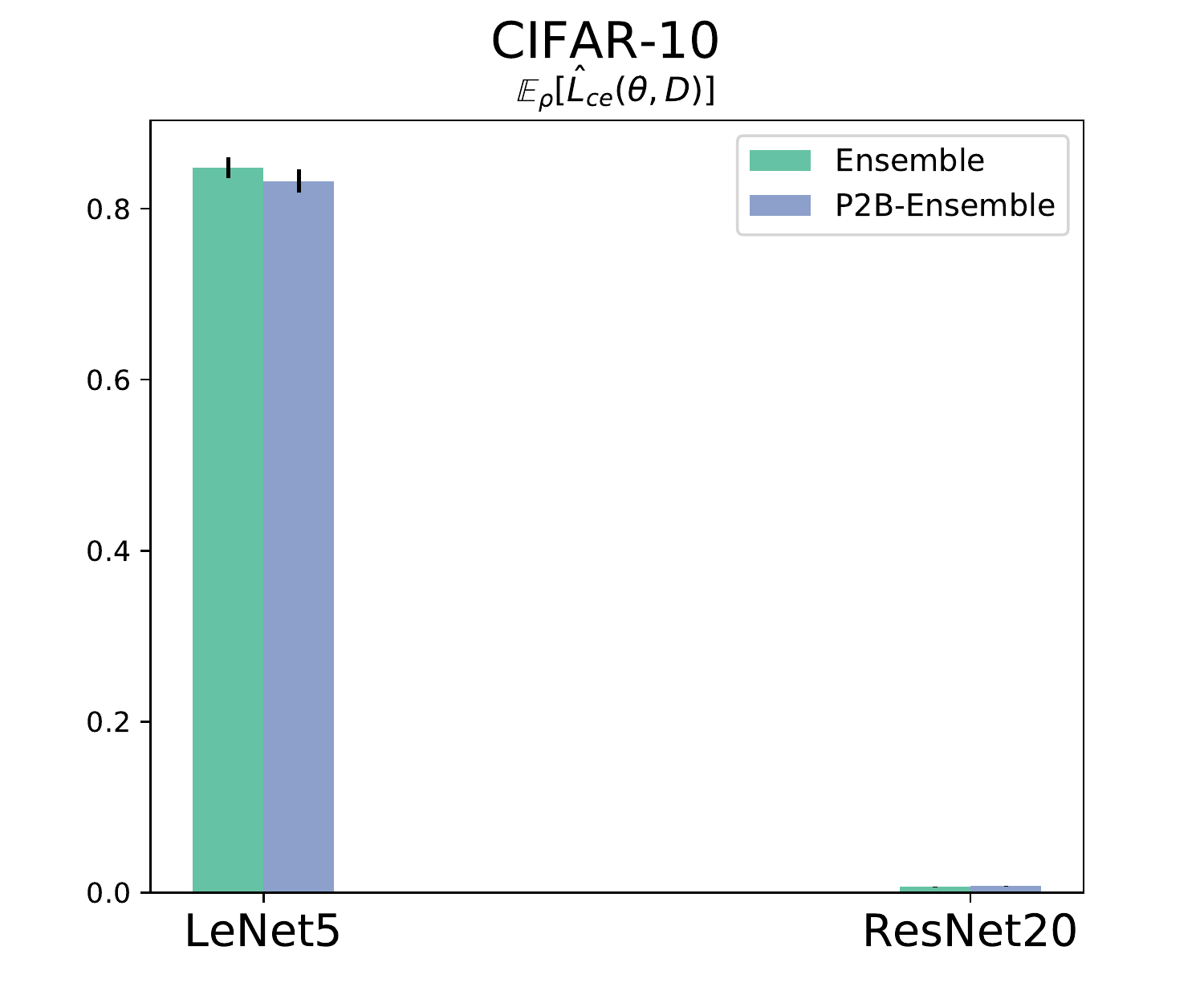}}
\\
\multicolumn{2}{c}{\hspace{-5pt}\includegraphics[width=0.5\linewidth]{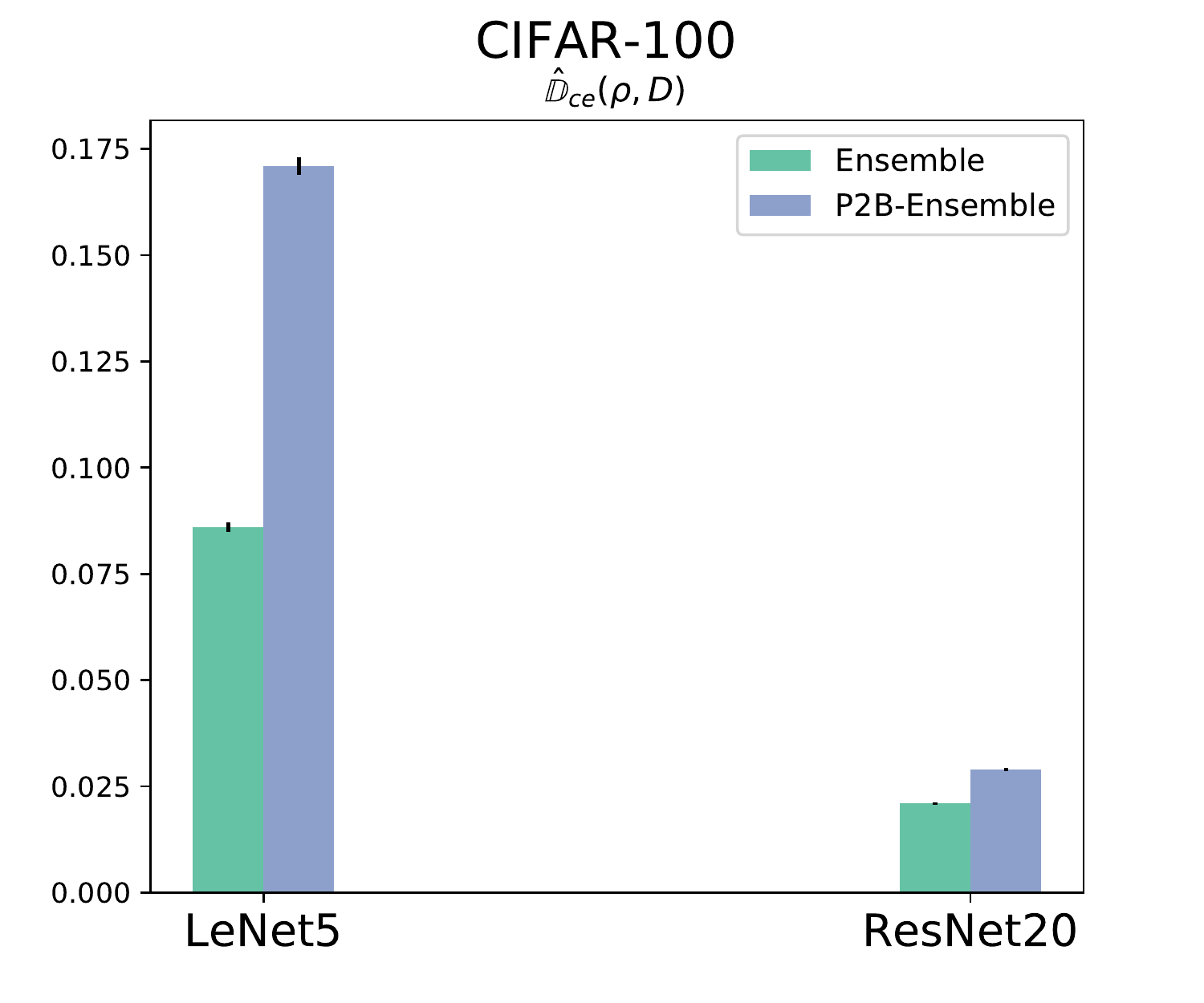}}
&
\multicolumn{2}{c}{\hspace{-5pt}\includegraphics[width=0.5\linewidth]{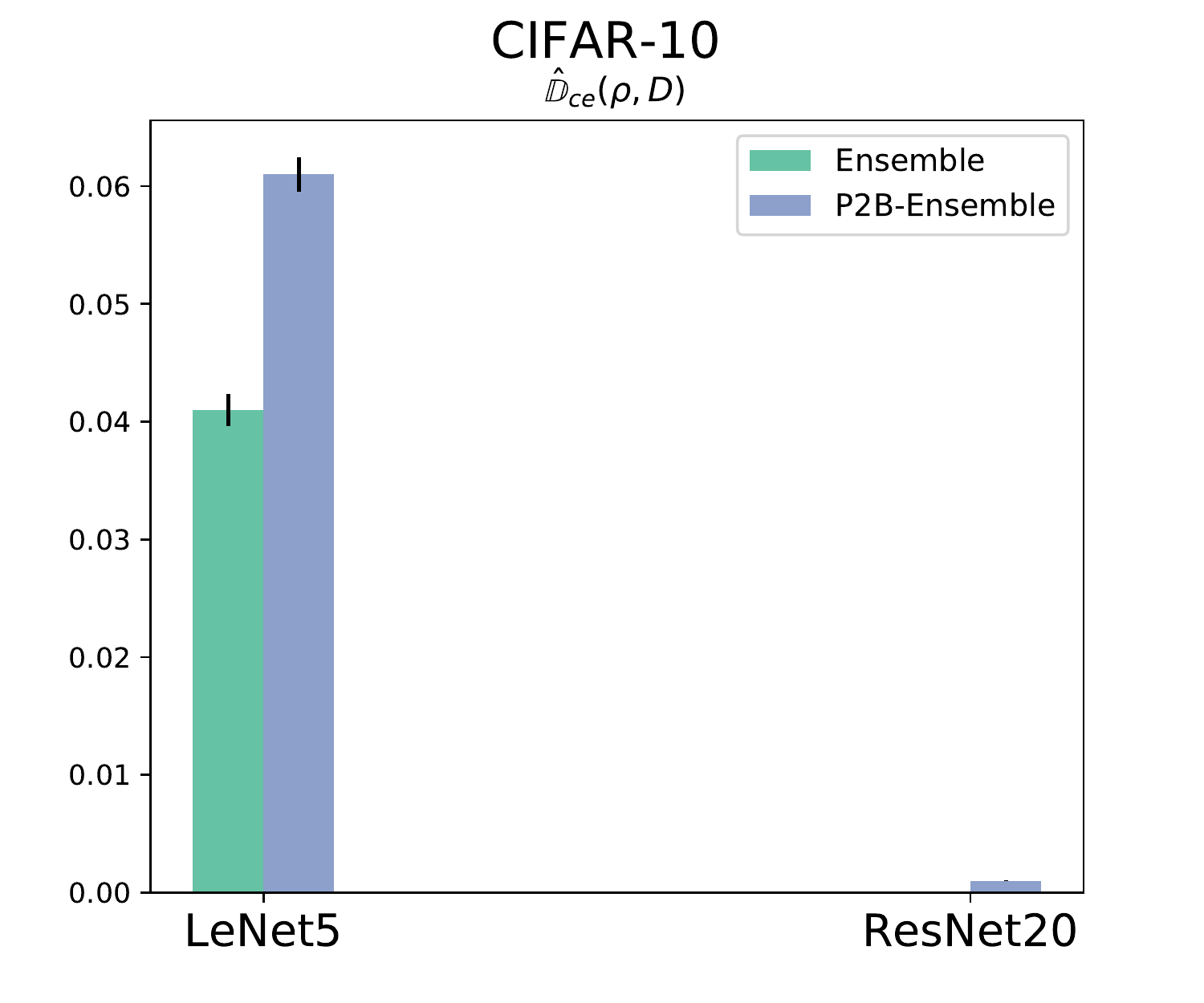}}
\\
\end{tabular}
  \caption{Empirical error of individual models $\E_\rho[\hat{L}(\bmtheta,D)]$ and empirical diversity $\hat{\mathbb{D}}(\rho,D)$ of LetNet5 and ResNet20 on CIFAR-10 and CIFAR-100 data sets.}\label{fig:empirical_diversity}
\end{figure}



\section{Discussion and Limitations}

Although it is well-known that diversity plays a key role in the performance of an ensemble, no theoretical works have previously established a well founded description of this relationship between diversity and generalization that applies to a wide range of ensemble models. This work aims to contribute a new stepping stone in this direction.

We have shown in Theorem \ref{thm:decomposition} that employing an upper-bound-based decomposition of the error of an ensemble is a promising approach that applies to a wide range of ensembles. For example, we have shown which is the role that correlation among predictors plays in ensemble diversity. 
This decomposition has also allowed us to apply a PAC-Bayesian analysis that has established a neat link between the empirical diversity of an ensemble and its generalization performance. And we have also discussed why ensembles of deep neural networks promote diversity through randomization strategies. 

Although some of the results presented in this work are straightforward consequences results that are readily available in the literature, the main contribution of this work has been to put them together and use them as a framework for reasoning about diversity and generalization of neural network ensembles.

This theoretical analysis also presents several limitations that open exciting research directions. First, the upper bounds of Theorem \ref{thm:decomposition} could potentially be tighter. For example, C-bounds \citep{GLL+15} are known to be tighter than our $\zeroone$-bound (see \citep{masegosa2020second} for a discussion at this respect). 

Same happens with the PAC-Bayesian bounds employ in Theorem \ref{thm:2ndPACBayes}. This family of PAC-Bayesian bounds builds on the the upper bounds of Theorem \ref{thm:decomposition}. Hence, if these bounds are not tight then we may expect the same for the associated PAC-Bayesian bound. At the same time, we have only chosen a general form of PAC-Bayesian bounds. Thus, more specialized and tighter PAC-Bayesian bounds could be employed instead, as done by \citep{masegosa2020second}. 



\section*{Acknowledgements}
This research is part of projects PID2019-106758GB-C31, PID2019-106758GB-C32, funded by MCIN/AEI/10.13039/ 501100011033, FEDER ``Una manera de hacer Europa'' funds. This research is also partially funded by Junta de Andaluc\'{i}a grant P20-00091. Finally we would like to thank the ``Mar\'{i}a Zambrano'' grant (RR\_C\_2021\_01) from the Spanish Ministry of Universities and funded with NextGenerationEU funds.

\bibliography{bibliography-yevgeny,bibliography-masegosa,others}
\bibliographystyle{plainnat}

\vfill

\pagebreak

\appendix

\renewcommand\thefigure{\thesection.\arabic{figure}} 
\renewcommand\thetable{\thesection.\arabic{table}} 
\renewcommand{\theequation}{\thesection.\arabic{equation}}
\renewcommand{\thetheorem}{\thesection.\arabic{theorem}}

\onecolumn

\section{Decomposing the Loss of an Ensemble Using an Upper Bound}\label{app:4.1}

\subsection{Proof of Theorem 1}

We first provide the following preliminary result:

\begin{corollary}\label{cor:pairwise}
	The diversity terms $\mathbb{D}(\rho)$ defined in Theorem \ref{thm:decomposition} can be written as, 
	\begin{align*}
	\mathbb{D}_{\MSE}(\rho) & =  \E_{\rho^2}\Big[\E_\nu\Big[h_R(\bmx;\bmtheta)^2 - h_R(\bmx;\bmtheta)h_R(\bmx;\bmtheta')\Big]\Big]\\
	\mathbb{D}_{\zeroone}(\rho) & = \E_{\rho^2}\Big[\E_\nu\Big[\1(h(\bmx;\bmtheta)= y)\1(h(\bmx;\bmtheta')\neq y)\Big]\Big]\\
	\mathbb{D}_{\CE}(\rho) & = \E_{\rho^2}\left[\E_\nu\left[\frac{p( y \mid \bmx,\bmtheta)^2 - p(y\mid \bmx,\bmtheta)p(y \mid \bmx,\bmtheta')}{\displaystyle  2\max_{\bmtheta\in \bmTheta} p(y \mid \bmx,\bmtheta)^2}\right]\right]
	\end{align*}
	\noindent where $\rho^2$ is a shorthand for the product distribution $\rho \times \rho$ over $\bmTheta \times \bmTheta$ and the shorthand $\E_{\rho^2}[f(\bmtheta,\bmtheta')] = \E_{\bmtheta\sim\rho, \bmtheta'\sim\rho}[f(\bmtheta,\bmtheta')]$.
\end{corollary}

\begin{proof}
	This Corollary raises from the fact that the variance of the classifiers can be decomposed as:
	\[
	\begin{aligned}
	Var_{\rho}(f(\bmtheta)) &= \E_\rho[f(\bmtheta)^2] - \E_\rho[f(\bmtheta)]^2 = \E_\rho[f(\bmtheta)^2] - \E_{\rho^2}[f(\bmtheta)f(\bmtheta')] \\
	&= \E_{\rho^2}\Big[f(\bmtheta)^2 - f(\bmtheta)f(\bmtheta')\Big],
	\end{aligned}
	\]
	where \(f\) is determined by the considered loss function: for the $\MSE$-loss, \(f(\bmtheta)= h_R(\bmx;\bmtheta)\), the $\CE$-loss, \( f(\bmtheta)=  p(y \mid \bmx,\bmtheta)\), and the $\zeroone$-loss \( f(\bmtheta) = \1(h(\bmx;\bmtheta)\neq y)\).
	
\end{proof}

\textbf{Theorem 1.} \emph{Under the settings given in Section \ref{sec:rw}, we have that}
\begin{eqnarray}
L_{\MSE}(\rho)  &=  & \E_\rho[\L_{\MSE}(\bmtheta)] - 	\mathbb{D}_{\MSE}(\rho), \label{eq:app:decomposition:mse}\\
L_{\CE}(\rho)  &\leq & \E_\rho[\L_{\CE}(\bmtheta)] - \mathbb{D}_{\CE}(\rho),\label{eq:app:decomposition:ce}\\
L_{\zeroone}(\rho)   &\leq& 4 (\E_\rho[\L_{\zeroone}(\bmtheta)] - \mathbb{D}_{\zeroone}(\rho)).\label{eq:app:decomposition:zeroone}
\end{eqnarray}
\noindent \emph{The diversity terms $\mathbb{D}(\rho)$ have the following expressions:}
\begin{eqnarray}
\mathbb{D}_{\MSE}(\rho) &=& \E_\nu\Big[\E_\rho\Big[\big(h_R(\bmx;\bmtheta) - \E_\rho\left[h_R(\bmx;\bmtheta)\right]\big)^2\Big]\Big],\\
\mathbb{D}_{\CE}(\rho) &=& \E_\nu\left[\frac{1}{\displaystyle  2\max_{\bmtheta\in \bmTheta} p(y|\bmx,\bmtheta)^2}\E_\rho\left[\left(p(y\mid\bmx,\bmtheta) - \E_\rho[p(y\mid\bmx,\bmtheta)]\right)^2\right]\right],\\
\mathbb{D}_{\zeroone}(\rho) &=& \E_\nu\Big[\E_\rho\Big[\big(\1(h_W(\bmx;\bmtheta)\neq y) - \E_\rho\left[\1(h_W(\bmx;\bmtheta)\neq y)\right]\big)^2\Big]\Big],
\end{eqnarray}
\noindent \emph{where $\mathbb{D}_{\CE}(\rho)$ is well-defined since $\max_{\bmtheta\in \bmTheta} p(y\mid\bmx,\bmtheta)\leq 1$.}

\begin{proof}
	Let us begin with the mse-error. Recall the definition of the expected mean squared error of a regression ensemble reparameterized by the distribution \( \rho \), that is,
	\[
	L_{mse}(\rho) = \mathbb{E}_\nu \left[(y-h_R(\bmx;\rho))^2\right],
	\]
	where \( h_R(\bmx; \rho) = \mathbb{E}_\rho \left[h_R(\bmx ; \bmtheta)\right] \) with \( h_R(\bmx ; \bmtheta) \) an individual regression model. We can get the desired result by expanding the square in each of the elements on the right hand side of the equation, that is:
	\[
	\begin{aligned}
	\mathbb{E}\left[L_{mse}(\bmtheta)\right] &= \mathbb{E}_{\rho, \nu} \left[(y - h_R(\bmx; \bmtheta))^2\right] =  \mathbb{E}_{\rho, \nu} \left[y^2 - 2yh_R(\bmx; \bmtheta) + h_R(\bmx ; \bmtheta)^2\right]\\
	&= \mathbb{E}_\nu \left[y^2 -2y h_R(\bmx;\rho) + \mathbb{E}_\rho [h_R(\bmx; \bmtheta)^2]\right],
	\end{aligned}
	\]
	where we used that \( y \) is constant under \( \mathbb{E}_\rho \). On the other hand,
	\[
	\begin{aligned}
	\mathbb{D}_{\MSE}(\rho) &= \mathbb{E}_{\nu,\rho} \left[  (h_R(\bmx, \bmtheta) - \mathbb{E}_\rho[h_R(\bmx;\bmtheta)])^2\right] = \mathbb{E}_{\nu,\rho} \left[  (h_R(\bmx, \bmtheta) - h_R(\bmx;\rho))^2\right]\\
	&= \mathbb{E}_{\nu,\rho} \left[ h_R(\bmx, \bmtheta)^2 -2h_R(\bmx, \bmtheta)h_R(\bmx, \rho) + h_R(\bmx;\rho)^2\right]\\
	&= \mathbb{E}_\nu \left[   \mathbb{E}_\rho[h_R(\bmx, \bmtheta)^2] -2h_R(\bmx, \rho)^2 + h_R(\bmx;\rho)^2 \right].
	\end{aligned}
	\]
	Finally, subtracting both expressions, we get that
	\[
	\begin{aligned}
	\E_\rho[\L_{\MSE}(\bmtheta)] - \mathbb{D}_{\MSE}(\rho)  &= \mathbb{E}_\nu \left[ y^2 - 2yh_R(\bmx; \rho) + h_R(\bmx; \rho)^2\right]\\
	&= \mathbb{E}_\nu \left[(y - h_R(\bmx; \rho))^2\right]\\
	&= L_{mse}(\rho).
	\end{aligned}
	\]
	Continuing with the cross-entropy error, we will be using the Taylor's theorem with a remainder of second order over the logarithm function. That is, given \( \log x \) and a fixed value \( a > 0\),
	\[
	\log x = \log a + \frac{1}{a}(x - a) - \frac{1}{2 \xi^2}(x - a)^2, \quad \xi \in (x, a).
	\]
	Applying this to \( p(y \mid \bmx, \bmtheta) \) centered at \( \E_\rho [p(y \mid \bmx, \bmtheta)] > 0 \),
	\[
	\begin{aligned}
	\log p(y \mid \bmx, \bmtheta) &= \log \E_\rho [p(y \mid \bmx, \bmtheta)] + \frac{1}{\E_\rho [p(y \mid \bmx, \bmtheta)]}\left( p(y \mid \bmx, \bmtheta)- \E_\rho [p(y \mid \bmx, \bmtheta)] \right) \\
	&\quad- \frac{1}{2\xi^2}\left(p(y \mid \bmx, \bmtheta) - \E_\rho [p(y \mid \bmx, \bmtheta)]\right)^2,
	\end{aligned}
	\] 
	taking expectation over \( \rho \) at both sides,
	\[
	\mathbb{E}_\rho [\log p(y \mid \bmx, \bmtheta)] = \log \E_\rho [p(y \mid \bmx, \bmtheta)] - \mathbb{E}_\rho \left[\frac{1}{2\xi^2}\left(p(y \mid \bmx, \bmtheta) - \E_\rho [p(y \mid \bmx, \bmtheta)]\right)^2\right].
	\]
	Rearranging terms,
	\[
	-\log \E_\rho [p(y \mid \bmx, \bmtheta)]= -\mathbb{E}_\rho [\log p(y \mid \bmx, \bmtheta)]- \mathbb{E}_\rho \left[\frac{1}{2\xi^2}\left(p(y \mid \bmx, \bmtheta) - \mathbb{E}_\rho [p(y \mid \bmx, \bmtheta)]\right)^2\right].
	\]
	The desired inequality raises from the fact that \( \xi \) is between \( p(y \mid \bmx, \bmtheta) \) and \( \E_\rho [p(y \mid \bmx, \bmtheta)] \), and hence, is upper bounded by \( max_{\bmtheta \in \bmTheta} p(y \mid \bmx, \bmtheta) \). Additionally, the square in the last term is always positive, implying the whole term is positive. Using this two properties,
	\[
	\begin{aligned}
	-\log \mathbb{E}_\rho [p(y \mid \bmx, \bmtheta)] &\leq -\mathbb{E}_\rho [\log p(y \mid \bmx, \bmtheta)]\\
	&\quad- \mathbb{E}_\rho \left[\frac{1}{2 \max_\bmtheta p(y \mid \bmx, \bmtheta)^2}\left(p(y \mid \bmx, \bmtheta) - \mathbb{E}_\rho [p(y \mid \bmx, \bmtheta)]\right)^2\right].
	\end{aligned}
	\]  
	Finally, taking expectations wrt \( \nu \) on both sides raises the desired result. Lastly, let us consider the $0/1$-error. In order to prove this result, we are using Markov's inequality for monotonically increasing functions, in this case, for \( \psi(a) = a^2 \). That is, for a given random variable \( X \), 
	\[
	\mathbb{P}(|x| \geq a) \leq \frac{\mathbb{E}[\psi(|x|)]}{\psi(a)} = \frac{\mathbb{E}[|x|^2]}{a^2}, \quad \text{for any } a > 0.
	\]
	Applying this theorem to \( \mathbb{E}_\rho\left[ \1(h_W(\bmx ; \bmtheta) \neq y)\right] \) we get that
	\[
	\mathbb{P}\left(  \mathbb{E}_\rho \left[\1 \left(h_W(\bmx; \bmtheta) \neq y\right)\right] \geq 0.5 \right) \leq 4 \mathbb{E}_\nu\left[ \mathbb{E}_\rho \left[\1 \left(h_W(\bmx; \bmtheta) \neq y\right)\right]^2\right].
	\]
	We may notice that if majority vote makes an error, at least half (\( \rho \)-weighted) of the classifiers are wrong, that is
	\[
	\1(h_W(\bmx ; \rho) \neq y) \leq \1[\mathbb{E}_\rho[\1h_W(\bmx;\bmtheta) \neq y] \geq  0.5],
	\]
	which implies
	\[
	\begin{aligned}
	\mathbb{E}_\nu \left[ \1 \left(h_W(\bmx ; \rho) \neq y\right) \right] &\leq \mathbb{E}_\nu \left[ \1 \left( \mathbb{E}_\rho\left[\1 \left(h_W(\bmx;\bmtheta) \neq y \right)\right] \geq  0.5\right) \right] = \\
	& =  \mathbb{P}\left(  \mathbb{E}_\rho \left[\1 \left(h_W(\bmx; \bmtheta) \neq y\right)\right] \geq 0.5 \right).
	\end{aligned}
	\]
	Using the derived inequality of the last term,
	\[
	L_{0/1}(\rho) = \mathbb{E}_\nu \left[ \1 \left(h_W(\bmx ; \rho) \neq y\right) \right] \leq 4 \mathbb{E}_\nu\left[ \mathbb{E}_\rho \left[\1 \left(h_W(\bmx; \bmtheta) \neq y\right)\right]^2\right]. 
	\]
	In order to conclude the proof, we need to show that the right hand side of the inequality is the desired upper bound of the \( 0/1 \) loss. Using that \( Var(X) = \mathbb{E}[(X - \mathbb{E}[X])^2] = \mathbb{E}[X^2] - \mathbb{E}[X]^2\) over \( \1(h_W(\bmx;\bmtheta)\neq y) \), we get that   
	\[
	\begin{aligned}
	\mathbb{D} _{0/1}(\rho)) &= \E_\nu\left[\E_\rho\left[(\1(h_W(\bmx;\bmtheta)\neq y) - \E_\rho\left[\1(h_W(\bmx;\bmtheta)\neq y)\right])^2\right]\right])\\
	&= \E_\nu\left[\E_\rho\left[\1(h_W(\bmx;\bmtheta)\neq y)\right] -\E_\rho \left[\1(h_W(\bmx;\bmtheta)\neq y)\right]^2 \right]\\
	&= \E_\rho\Big[\E_\nu\left[\1(h_W(\bmx;\bmtheta)\neq y)\right]\Big] -\E_\nu \left[\E_\rho \left[\1(h_W(\bmx;\bmtheta)\neq y)\right]^2 \right].
	\end{aligned}
	\]
	Which implies that 
	\[
	\begin{aligned}
	4\left(\mathbb{E}_\rho[L_{0/1}(\bmtheta)] - \mathbb{D} _{0/1}(\rho)\right) &= 4\Big(\mathbb{E}_\rho\left[ \E_\nu \left[\1 (h_W(\bmx; \bmtheta) \neq y)\right]\right] -  \mathbb{D} _{0/1}(\rho)\Big)\\
	&= 4\E_\nu \left[\E_\rho \left[\1(h_W(\bmx;\bmtheta)\neq y)\right]^2 \right].
	\end{aligned}
	\]
\end{proof}

\subsection{Tighter variant of Theorem 1 inequality}\label{subsec:tighter_ineq}

The following result demonstrates how to define a tighter second-order Jensen bound for the cross-entropy error using the Jensen inequality stated in \citep{liao2019sharpening}. 

\textbf{Theorem}. Any distribution $\rho$ over $\bmTheta$ satisfies the following inequality, 
\[
L_{ce}(\rho)\leq \mathbb{E}_{\rho(\bmtheta)}[L(\bmtheta)] - \mathbb{V}^T_{ce}(\rho),
\]
where  $\mathbb{V}^T_{ce}(\rho)$ is the normalized variance of $p(y \mid \bmx, \bmtheta)$ wrt $\rho(\bmtheta)$,
\[
\mathbb{V}^T_{ce}(\rho) = \mathbb{E}_{\nu}\Big[h(m,\mu)\mathbb{E}_{\rho(\bmtheta)}\Big[(p(y \mid \bmx, \bmtheta) - p(y))^2\Big]\Big].
\]
Where $\mu=\mathbb{E}_{\rho}[p(y \mid \bmx, \bmtheta)]$, $m= \max_\bmtheta p(y \mid\bmx, \bmtheta)$ and $h(m,\mu) = \frac{\ln \mu - \ln m }{(m - \mu)^2} + \frac{1}{\mu(m - \mu)}$.
\begin{proof}[Proof sketch]
	Apply \citep{liao2019sharpening}'s result to the random variable $p(\bmx \mid \bmtheta)$, following the same strategy used in the proof of Theorem \ref{thm:decomposition}.
\end{proof}

\section{How to Measure the Diversity of an Ensemble?}\label{app:4.2}

\subsection{Proof of Theorem \ref{theorem:diversity_as_variance}}

\textbf{Theorem \ref{theorem:diversity_as_variance}}. The diversity terms $\mathbb{D}(\rho)$ defined in Theorem \ref{thm:decomposition} can be written as
\[
    \mathbb{D} (\rho) = \V_{\nu\times\rho}\Big(f(y,\bmx;\bmtheta)\Big) - \E_{\rho\times\rho}\Big[Cov_{\nu}(f(y,\bmx;\bmtheta),f(y,\bmx;\bmtheta'))\Big] 
\]
    \noindent were $\rho\times\nu$ denotes the joint distribution over $\bmTheta \times({\cal\mathbf{X}},{\cal Y})$ and \(f\) is completely determined by the considered loss function.

\begin{proof}
    \begin{align*}
         \mathbb{D} (\rho) &= \E_{\nu}\Big[ \V_{\rho}\big[f(y, \bmx; \bmtheta)\big] \Big]\\
         &= \E_\nu \Big[ \E_{\rho^2} \big[f(y, \bmx; \bmtheta)^2 - f(y, \bmx ;\bmtheta)f(y, \bmx ;\bmtheta')\big] \Big]\\
         &= \E_\nu \Big[ \E_{\rho} \big[f(y, \bmx; \bmtheta)^2\big] - \E_{\rho^2} \big[f(y, \bmx ;\bmtheta)f(y, \bmx ;\bmtheta')\big] \Big]\\
         &= \V_{\nu \times \rho}\big[ f(y, \bmx; \bmtheta) \big] + \E_{\nu \times \rho}\big[ f(y, \bmx; \bmtheta) \big]^2- \E_{\nu \times \rho^2} \big[f(y, \bmx ;\bmtheta)f(y, \bmx ;\bmtheta')\big] \Big]\\
         &=\V_{\nu\times\rho}\big[f(y,\bmx;\bmtheta)\big] - \E_{\rho\times\rho}\Big[Cov_{\nu}\big(f(y,\bmx;\bmtheta),f(y,\bmx;\bmtheta')\big)\Big]
    \end{align*}
\end{proof}

\subsection{Proof of Lemma 3}

\textbf{Lemma 3.} The diversity terms $\mathbb{D}(\rho)$ defined in Theorem \ref{thm:decomposition} satisfy the following properties:
\begin{enumerate}[i)]
	\item If all the ensemble’s members provide the same predictions or places all its probability mass in a single model, then $\mathbb{D}(\rho)$ is null.
	\item $0\leq \mathbb{D}(\rho) \leq \E_\rho[\L(\bmtheta)]$
	\item $\mathbb{D}(\rho)$ is invariant to reparametrizations. 
\end{enumerate}

\emph{Proof.}
\begin{enumerate}[i)]
	\item Notice that if all models make the same prediction, \( h(\bmx; \bmtheta) = \mathbb{E}_\rho[h(\bmx ; \bmtheta)] \)  and \( p(y \mid \bmx, \bmtheta) = \E_\rho[p(y \mid \bmx, \bmtheta)] \), which nullifies all diversity definitions from Theorem \ref{thm:decomposition}.
	\item Using that every considered loss function and variance are positive, the given inequality is trivial.
	\item Let \(\phi: \bm{\Omega} \to \bmTheta\) be an injective differentiable function with continuous partial derivatives, with non-zero Jacobian at any point. This result follows from the fact that we are considering a probability distribution that is a finite mixture of delta distributions, as a result, the distribution is compactly supported and the variable change theorem can be applied to a continuous function \( f: \bmTheta \to \mathbb{R} \):
	\[
	\begin{aligned}
	\E_\rho[f(\bmtheta)] &= \int_\bmTheta \rho (\bmtheta) f(\bmtheta) = \int_{\bm\Omega} \rho \circ \phi(\bm\omega) \ f \circ \phi(\bm\omega)\ |det(D\phi)(\bm\omega)| \\
	&= \E_{\rho '}[f \circ \phi (\bm\omega)]
	\end{aligned}
	\]
	where
	\[
	\rho '(\bm\omega) = \rho \circ \phi(\bm\omega)\ |det(D\phi)(\bm\omega)|.
	\]
	The result follows from taking \(f(\bmtheta) = (\bmtheta - \E_\rho (\bmtheta))^2\).
	
\end{enumerate}

\subsection{Pairwise diversity measures}

Note that Corollary \ref{cor:pairwise} shows how the provided diversity measures are pairwise diversity measures \citep{kuncheva2003measures} because they only consider interactions among pair of models.

\section{How is Diversity Related to the Performance of an Ensemble?}\label{app:4.3}

\subsection{Empirical diversity satisfies the same properties as the theoretical}
Let us discuss every point in Lemma 2 from an empirical point of view:

\begin{enumerate}[i)]
	\item It is clear that if all models make the same prediction or place the same probabilities, the empirical diversity is zero.
	\item In order to show this, the same arguments used for the theoretical diversity must be applied to the empirical one, using an empirical version of Theorem 1 that reduces to take the expectation over the empirical distribution at each step.
	\item To show that invariance beholds, we can redo the same proof using the empirical expectation. In this case, given that the empirical probability is compactly supported, the variable change theorem holds.
\end{enumerate}

\subsection{Proof of Theorem 5}

\textbf{Theorem 5 (PAC-Bayes bounds)}. \emph{For any prior distribution $\pi$ over $\bmTheta$ independent of $D$ and for any $\xi\in (0,1)$ and any $\lambda>0$, with probability at least $1-\xi$ over draws of training data $D\sim \nu^n(\bmx,y)$, for all distribution $\rho$ over $\bmTheta$, simultaneously,}
\[
L(\rho) \leq \alpha\left(\E_\rho[\hat L(\bmtheta,D)] - \hat{\mathbb{D}}(\rho,D) + \frac{2\kl{\rho}{\pi}+   \epsilon(\nu,\pi,\lambda,n,\xi)}{\lambda\, n}\right),
\]
\noindent \emph{where $\alpha$ is equal to $1$ if we consider the $\MSE$-loss or the $\CE$-loss, and $4$ is we consider the $\zeroone$-loss.
	$KL$ refers to the Kullback-Leibler divergence between $\rho$ and the prior $\pi$. And $\epsilon(\nu,\pi,\lambda,n,\xi)>0$ is a function, which is independent of $\rho$ but also depends on the specific loss.}

\begin{proof}
	In order to prove this theorem, we are going to show that the rhs term is an upper bound for \( \alpha(\E_\rho[L(\bmtheta)] - \mathbb{D}(\rho))\), which, using Theorem~\ref{thm:decomposition} concludes the proof. First of all, consider the following \emph{tandem} losses:
	\begin{align*}
	L_{mse}(\bmtheta, \bmtheta') & = L_{mse}(\bmtheta)-\E_\nu\Big[h_R(\bmx;\bmtheta)^2 - h_R(\bmx;\bmtheta)h_R(\bmx;\bmtheta')\Big]\\
	L_{\zeroone}(\bmtheta, \bmtheta') & = L_{\zeroone}(\bmtheta) - \E_\nu\Big[\1(h(\bmx;\bmtheta)= y)\1(h(\bmx;\bmtheta')\neq y)\Big]\\
	L_{ce}(\bmtheta, \bmtheta') & = L_{ce}(\bmtheta)- \E_\nu\left[\frac{p(y \mid \bmx,\bmtheta)^2 - p(y \mid \bmx,\bmtheta)p(y \mid \bmx,\bmtheta')}{\displaystyle  2\max_{\bmtheta\in \bmTheta} p(y \mid \bmx,\bmtheta)^2}\right]
	\end{align*}
	which, using Corollary~\ref{cor:pairwise} verifies
	\[
	\E_{\rho^2}[L(\bmtheta, \bmtheta')] =  \E_\rho[L(\bmtheta)] - \mathbb{D}(\rho).
	\]
	Applying \citet[Theorem 3]{germain2016pac} to the tandem loss functionals described above with a prior distribution \(\pi(\bmtheta, \bmtheta') = \pi(\bmtheta)\pi(\bmtheta')\), we got that for any \(\lambda n > 0, \delta \in (0, 1]\), with probability at least \(1 - \delta\):
	\[
	\E_{\rho(\bmtheta, \bmtheta')}[ L(\bmtheta, \bmtheta')] \leq \E_{\rho(\bmtheta, \bmtheta')}[ \hat{L}(\bmtheta, \bmtheta')] + \frac{1}{\lambda n} \left[ \kl{\rho(\bmtheta, \bmtheta')}{\pi(\bmtheta, \bmtheta')}  + \epsilon(\nu,\pi,\lambda,n,\xi) \right].
	\]
	Where 
	\begin{equation}\label{eq:epsilon}
	\epsilon(\nu,\pi,\lambda,n,\xi) = \log \E_{\pi(\bmtheta, \bmtheta')} \left[\E_\nu\left[\exp \left(\lambda\left(L(\bmtheta, \bmtheta') - \hat{L}(\bmtheta, \bmtheta', D)\right)\right)\right]\right] + \log \frac{1}{\delta},
	\end{equation}
	
	and \(\kl{\rho(\bmtheta, \bmtheta')}{\pi(\bmtheta, \bmtheta')} = 2\kl{\rho}{\pi}\). In short, we got that
	\[
	L(\rho) \leq \alpha\left( \E_\rho[\hat{L}(\bmtheta, D)] - \hat{\mathbb{D}}(\rho, D) + \frac{1}{\lambda n} \Big[ 2\kl{\rho}{\pi} + \epsilon(\nu,\pi,\lambda,n,\xi) \Big]\right).
	\]
\end{proof}

\subsection{\texorpdfstring{$\epsilon(\nu,\pi,\lambda,n,\xi)$}\ \ in Theorem 5}

The general expression of \(\epsilon(\nu,\pi,\lambda,n,\xi)\) is described in Equation~\eqref{eq:epsilon}. It can be contextualized for the $\MSE$-loss, the $\CE$-loss and the $\zeroone$-loss by instantiating $\L(\bmtheta,\bmtheta')$ correspondingly, as shown in the beginning of the proof of Theorem 5. 

\section{How to Exploit Diversity to Learn Ensembles}\label{app:4.4}

\subsection{Working with a Finite Parameter space}

The assumption of finite parameter space is not restrictive in this case. We have to consider that if $\bmTheta=\{\bmtheta_1,\ldots,\bmtheta_K\}$, $K$ can be a very large number. Potentially, $\bmTheta$ could contain all finite-precision vectors of size $M$. In any case, the distribution $\rho$ will assign positive probabilities to only those models which are part of the ensemble. 

\subsection{Working with a Continuous Parameter Space}

The only point in this work where considering a continuous parameter space would alter the used reasoning is at Lemma 2, mode precisely the reparameterization invariance property of the considered diversity formulas. The key points is that if \(\bmTheta\) is a continuous non-compact set, such as \(\Re^M\) for a given \(M \in \mathbb{N}\), the change of variable theorem cannot be applied. In order to surpass this difficulty, we could always consider  
\[
\bmTheta = \{\bmtheta\in \Re^M\ : \ \|\bmtheta\|_2 \leq N \} \subset \Re^M.
\]
With \(N\) the highest norm of any representable number within the considered finite-precision of the machine.

\subsection{Mixtures of multivariate Gaussian Distributions approximation}\label{app:MoG}

As detailed in Section~\ref{sec:diversity:learning}, we consider an approach with an uniform Gaussian mixture, denoted by \(\rho_\delta\), to represent an ensemble of \(K\) models. That is, the following distribution models the parameters, given a fixed set of mean values \((\bmtheta_1, \dots, \bmtheta_K)\):
\[
\rho_\delta(\bmtheta) = \frac{1}{K}\sum_{k=1}^{K} \mathcal{N}(\bmtheta; \ \bmtheta_k, \epsilon I).
\]
Using this distribution, the expected value of the loss function is
\[
\E_{\rho_\delta}[\hat{L}(\bmtheta,D)] = \int_\bmtheta \frac{1}{K}\sum_{k=1}^{K} \mathcal{N}(\bmtheta; \ \bmtheta_k, \epsilon I) \hat{L}(\bmtheta,D = \frac{1}{K}\sum_{k=1}^{K} \int_\bmtheta  \mathcal{N}(\bmtheta; \ \bmtheta_k, \epsilon I) \hat{L}(\bmtheta,D).
\]
We then use the following approximation approach to simplify the above expression: given that \(\epsilon\) is sufficiently small, we can approximate the expected value of a function over a highly sharp distribution around its mean, with the evaluation of such function in the mean value of the distribution. More precisely,
\begin{equation}\label{eq:approximateGaussian}
\int_\bmtheta  \mathcal{N}(\bmtheta; \ \bmtheta_k, \epsilon I) f(\bmtheta)\approx f(\bmtheta_k) \quad \forall k =1,\dots,K.    
\end{equation}
Using that, the expected value of the loss can be approximated as
\[
\E_{\rho_\delta}[\hat{L}(\bmtheta,D)] = \frac{1}{K}\sum_{k=1}^{K} \int_\bmtheta \mathcal{N}(\bmtheta; \ \bmtheta_k, \epsilon I) \hat{L}(\bmtheta,D) \approx \frac{1}{K}\sum_{k=1}^K\hat{L}(\bmtheta_k,D).
\]
The same reasoning can be applied to the regularization KL term, along with another approximation:
\begin{equation}\label{eq:approximateGaussianKL}
\sum_{i=1}^{K} \mathcal{N}(\bmtheta_k; \ \bmtheta_i, \epsilon I)\approx \mathcal{N} (\bmtheta_k; \bmtheta_k, \epsilon I) = \frac{1}{\sqrt{(2\pi)^M \epsilon^M}}, \quad \forall k =1,\dots,K,
\end{equation}
with \(M\) the dimensionality of \(\bmtheta\). The main idea behind this is to assume that the considered mean values \((\bmtheta_1,\dots, \bmtheta_K )\) are far enough from each other so that for any pair \((\bmtheta_i, \bmtheta_k), k \neq i\) evaluating a Gaussian distribution centered in one of them \(\bmtheta_i\), with covariance matrix \(\epsilon I \) over the other value \(\bmtheta_k\), is approximately zero. This is not a strong assumption given that we can fix the value of \(\epsilon\) to any value, more precisely, we could set it so that the minimum euclidean distance between any pair of \((\bmtheta_1,\dots, \bmtheta_K )\) is greater than \(3\epsilon\), which contains the \(99.7\%\) of the density of the distribution.

As a result, the following approximation raises for the regularization term:
\[
\begin{aligned}
\kl{\rho_\delta}{\pi} &= \int_{\bmtheta} \rho_\delta(\bmtheta) \log \frac{\rho_\delta(\bmtheta)}{\pi(\bmtheta)} \\
&= \frac{1}{K}\sum_{k=1}^{K} \int_{\bmtheta} \mathcal{N}(\bmtheta; \ \bmtheta_i, \epsilon I) \left(\log \frac{1}{K}\sum_{i=1}^{K} \mathcal{N}(\bmtheta; \ \bmtheta_i, \epsilon I) - \log \pi(\bmtheta)\right) \\
\text{(by Equation \eqref{eq:approximateGaussian})}&\approx \frac{1}{K}\sum_{k=1}^{K}\left(\log \frac{1}{K}\sum_{i=1}^{K} \mathcal{N}(\bmtheta_k; \ \bmtheta_i, \epsilon I) - \log \pi(\bmtheta_k)\right)\\
\text{(by Equation \eqref{eq:approximateGaussianKL})}&\approx  \frac{1}{K}\sum_{k=1}^{K}\left(\log \frac{1}{K}\mathcal{N}(\bmtheta_k; \ \bmtheta_k, \epsilon I) - \log \pi(\bmtheta_k)\right)\\
&= -\frac{1}{K} \sum_{k=1}^K \log \pi(\bmtheta_k) + \frac{1}{K}\sum_{k=1}^K\log \frac{1}{K} \frac{1}{\sqrt{(2\pi)^M \epsilon^M}} \\
\end{aligned}
\]

We also apply the approximation given by Equation~\eqref{eq:approximateGaussian} over the general variance formula:

\[
\begin{aligned}
\hat{V}_{\rho_\delta}(f(\bmtheta)) &= \E_{\rho_\delta^2}\Big[f(\bmtheta)^2 - f(\bmtheta)f(\bmtheta')\Big] = \E_{\rho_\delta}\Big[f(\bmtheta)^2\Big] - \E_{\rho_\delta^2}\Big[f(\bmtheta)f(\bmtheta')\Big] \\
&= \frac{1}{K}\sum_{k=1}^K f(\bmtheta_k)^2 - \frac{1}{K^2}\sum_{i=1}^K \sum_{j=1}^K f(\bmtheta_i) f(\bmtheta_j).\\
\end{aligned}
\]
Given this, it is easy to approximate each of the diversity terms defined in Corollary~\ref{cor:pairwise}. Where \(f\) is determined by the considered loss function: for the $\MSE$-loss, \(f(\bmtheta)= h_R(\bmx;\bmtheta)\), the $\CE$-loss, \( f(\bmtheta)=  p(y \mid \bmx,\bmtheta)\), and the $\zeroone$-loss \( f(\bmtheta) = \1(h(\bmx;\bmtheta)\neq y)\). 

\subsection{Ensemble Learning Algorithms Which Explicitly Promote Diversity}\label{app:ensemblespromotingdiversity}

\paragraph{Negative Correlation Learning \citep{liu1999ensemble}:} This learning algorithm proposes the following minimization objective, which is written following our notation:
\[\E_\rho[\hat \L_{\MSE}(\bmtheta,D)] - \lambda\underbrace{\E_D[\frac{1}{K}\sum_{k=1}^K (h_R(\bmx; \bmtheta_k) - h_R(\bmx;\rho))\sum_{j\neq  k}(h_R(\bmx; \bmtheta_j) + h_R(\bmx;\rho))]}_{\text{NC}(\rho,D)} \]
\noindent where $\lambda\in[0,1]$, $\E_D[\cdot]$ denotes expectation wrt the empirical distribution of the data, and $\text{NC}(\rho,D)$ denotes the empirical negative correlation which promotes diversity. After simple algebraic manipulations, we can prove that $\text{NC}(\rho,D)=-\hat{\mathbb{D}}_{\MSE}(\rho,D)$. So, this algorithm matches our learning algorithm for the $\MSE$-loss. In consequence, our work provides a novel interpretation of the negative correlation ensemble learning algorithm \citep{liu1999ensemble} based on PAC-Bayesian bounds.

\paragraph{Generalized Ambiguity Decomposition \citep{jiang2017generalized}:} This work tries to extend \citet{krogh1995validation}'s decomposition to general loss functions. It employs a similar decomposition of the loss function in individual model errors and ensemble diversity. But their decomposition is not based on upper bounds. It only matches our decomposition for the $\MSE$-loss. This work does not relate the generalization error of the ensemble with the empirical diversity, as we do through PAC-Bayesian upper bounds. They do not consider multi-class classification problems. And the $\zeroone$-loss decomposition does not include a diversity term, and they do not consider weighted majority vote.

\paragraph{Generalized Negative Correlation Learning \citep{buschjager2020generalized}:} This work mainly builds on the decompositions given by \cite{jiang2017generalized}. They first arrive to a learning objective which includes a diversity term. This diversity term is different from the ones presented here. They do not consider the $\zeroone$-loss and weighted majority vote. However, they disregard the learning objective including a diversity term and advocate for a learning objective of the form: $\lambda \hat{\L}(\rho,D) + (1-\lambda)\E_\rho[\hat{\L}(\bmtheta,D)]$ with $\lambda\in[0,1]$.

\subsection{Standard Ensemble Learning Algorithms Do Not Explicitly Promote Diversity}

Standard learning algorithms can be interpreted as methods trying to minimize the following objective function:
\[\E_{\rho_\delta}[\hat\L(\bmtheta,D)] + \frac{\kl{\rho}{\pi}}{\lambda n}\]
where either the $\CE$-loss or the $\MSE$-loss are employed. 

This learning objective does not include the $\hat{\mathbb{D}}(\rho,D)$ term encouraging diversity. In fact, under the approximations discussed in Appendix \ref{app:MoG} and discarding constant terms, this learning objective can be expressed as:
\[\frac{1}{K}\sum_k \big(\hat\L(\bmtheta_k,D) - \frac{\ln \pi(\bmtheta_k)}{\lambda n}\big) \]
\noindent where each $\bmtheta_k$ can be learned independently from the rest due to the presence of the $\hatV(\rho,D)$ term. 

\section{Experimental Evaluation}\label{app:experiments}

\subsection{Experimental Settings}

The experimental evaluation was carried under Google Colab Pro (\href{https://colab.research.google.com}{colab.research.google.com})
 in an environment with 8 TPU cores. The Python packages used are those listed below (together with the corresponding dependencies): 

\begin{itemize}
	\item   Google Uncertainty Metrics \footnote{\url{https://github.com/google/uncertainty-baselines}}.
	\item	edward development version \footnote{\url{https://github.com/google/edward2/tree/f36188e58bd6e0454f551fed6b29beded1b777ad}}
	\item	fsspec 2021.5.0
	\item	gcsfs 2021.5.0
	\item	robustness-metrics development version \footnote{\url{https://github.com/google-research/robustness_metrics/tree/5a380c4ec11fa85255b5db643c64efa42d749110}}
	\item	tensorflow 2.4.1
	\item	tensorflow-datasets 4.0.1
	\item	tensorflow-estimator 2.4.0
	\item	tensorflow-gcs-config 2.4.0
	\item	tensorflow-hub 0.12.0
	\item	tensorflow-metadata 0.30.0
	\item	tensorflow-probability 0.12.1
\end{itemize}

The set of hyper-parameters considered in the experimentation are shown in Table~\ref{table:hyp}.

\begin{table}
	\centering
	\begin{tabular}{r|ccc}
		Hyper-parameter & LeNet5 & ResNet20 & MLP50 \\\hline			
		base learning rate	&	0.001	&	0.1	&	0.001	 \\ 
		epochs	&	200	&	250	&	250	 \\ 
		l2 Regularization	&	$2\cdot10^{-4}$	&	$2\cdot10^{-4}$		&	$2\cdot10^{-4}$		 \\ 
		learning rate decay	&	[60,120,160]	&	[60,120,160]	&	[60,120,160]	 \\ 
		per core batch size	&	64	&	64	&	32\\
	\end{tabular}
	\vspace*{0.5cm}
	\caption{Hyperparameters for each model.}\label{table:hyp}
\end{table}

\begin{figure}[!htb]
	\centering
	\begin{tabular}{cc}
		\includegraphics[scale = 0.4]{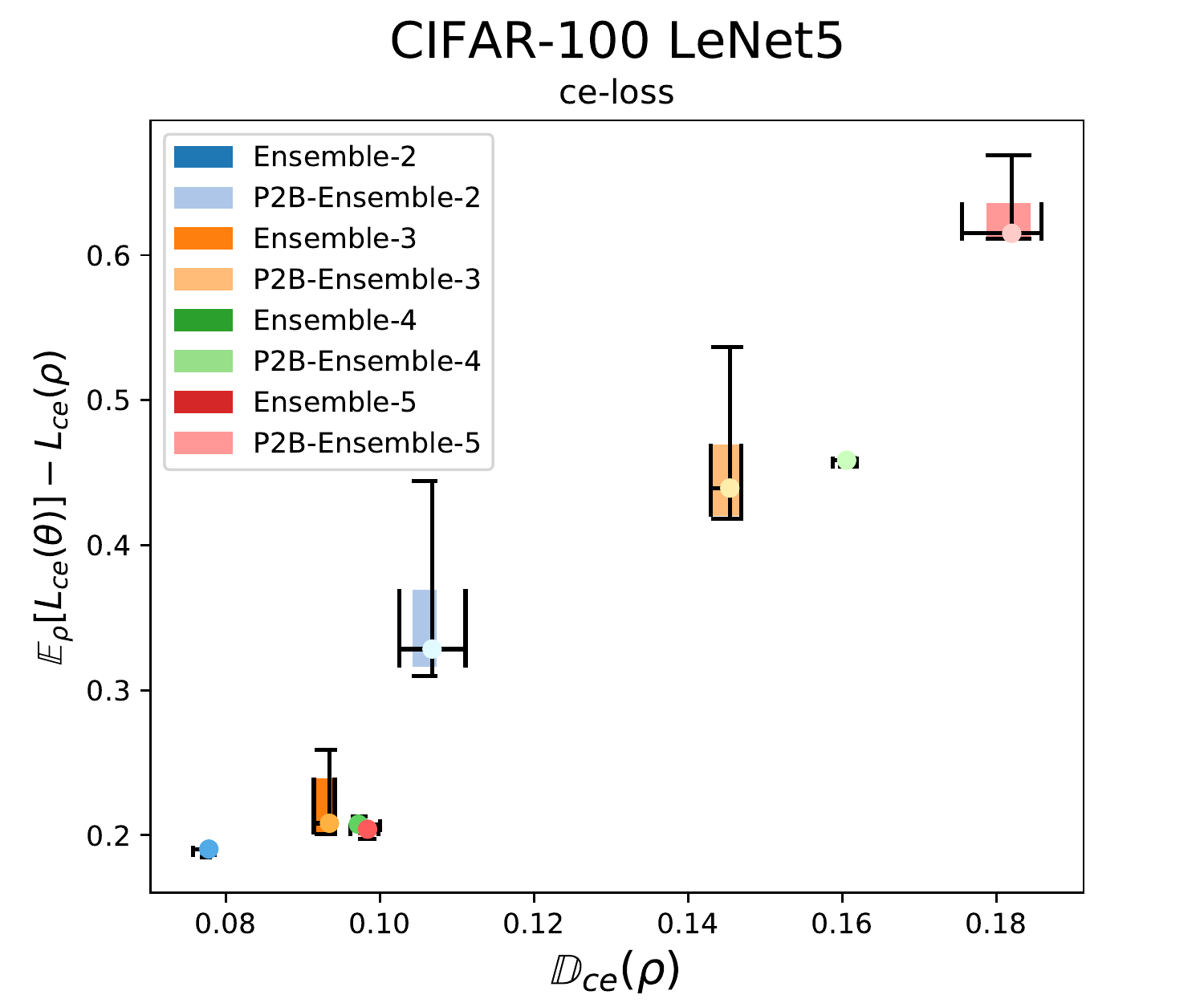}
		&  
		\includegraphics[scale = 0.4]{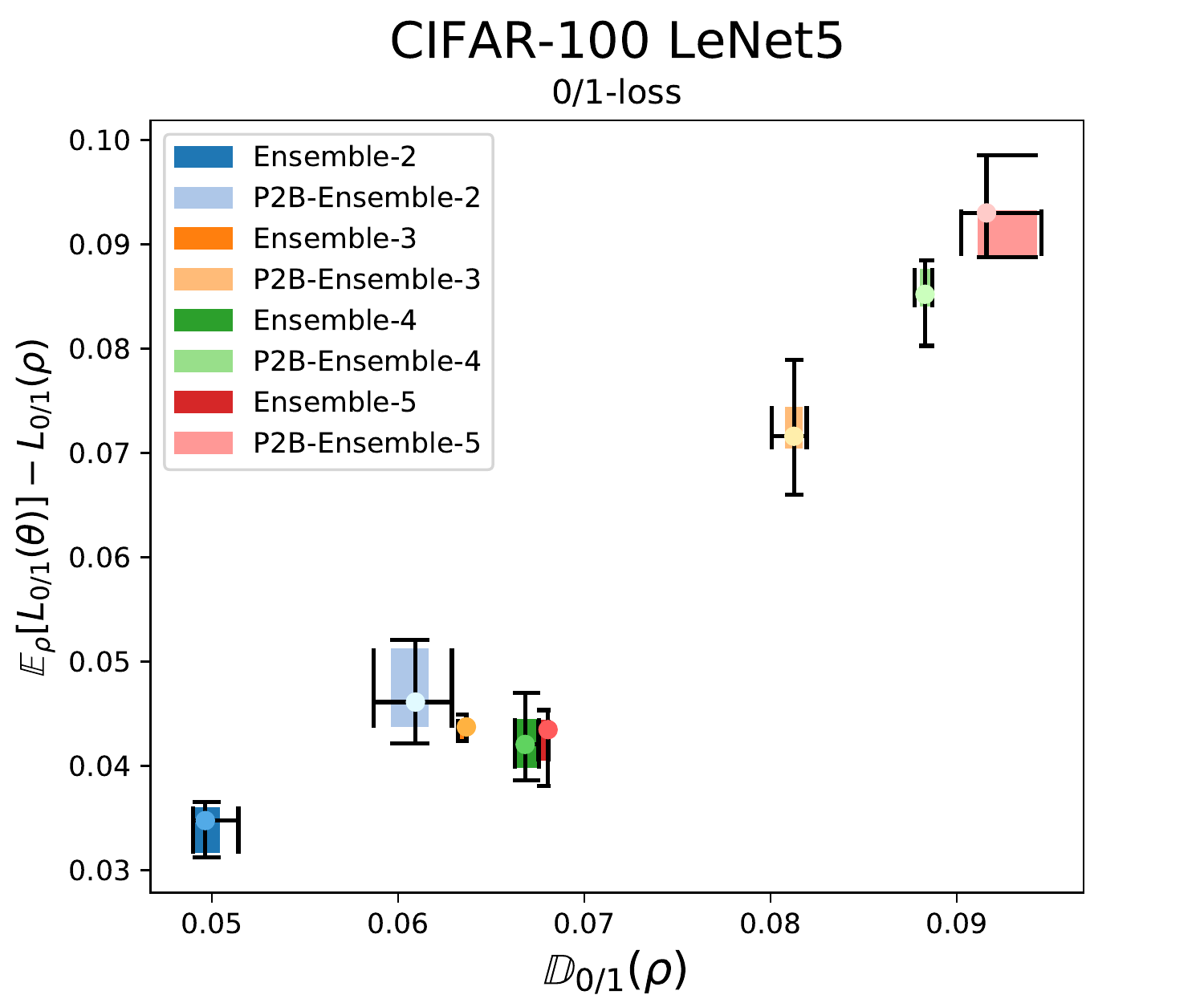}
		\\
		\includegraphics[scale = 0.4]{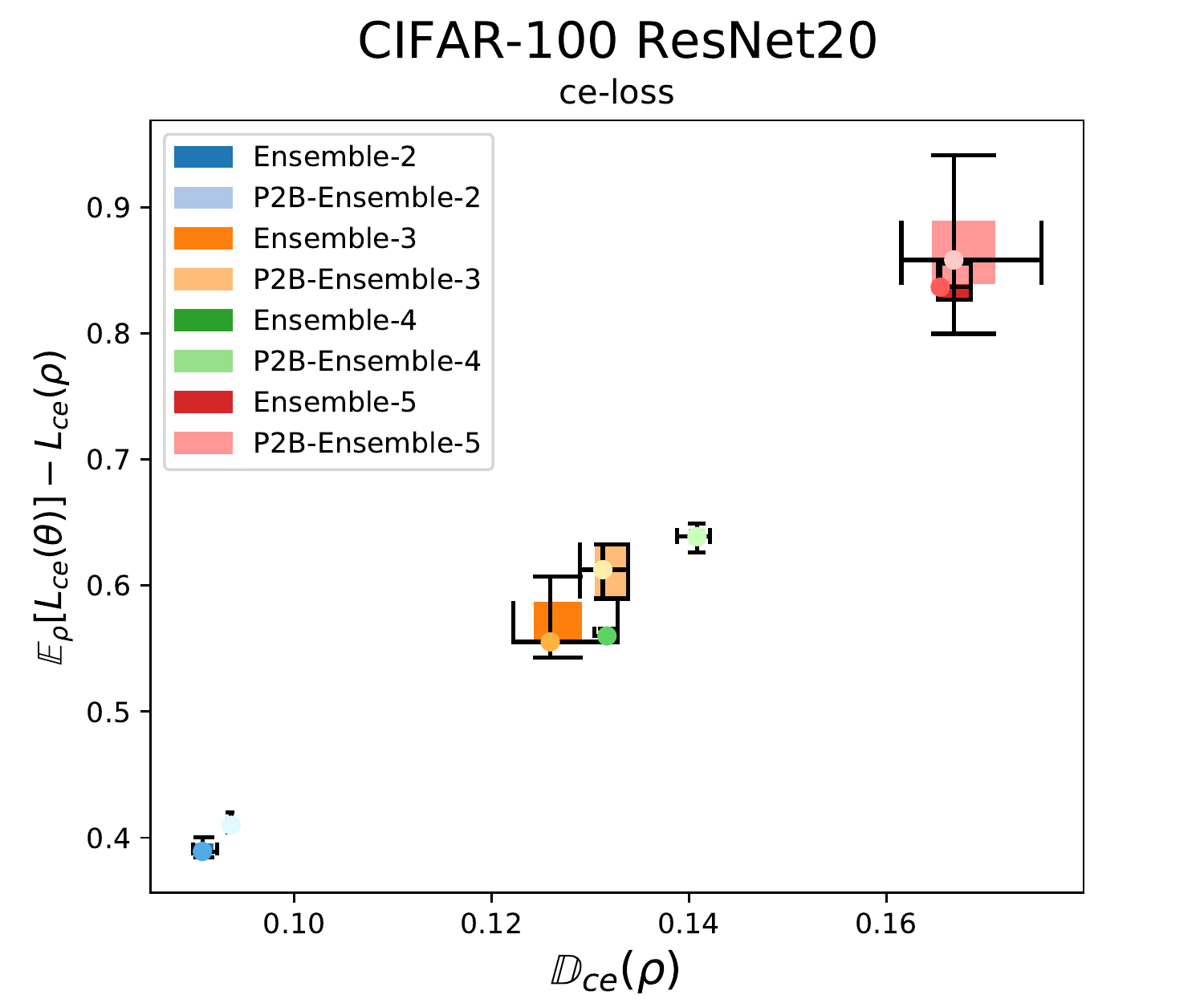}
		&
		\includegraphics[scale = 0.4]{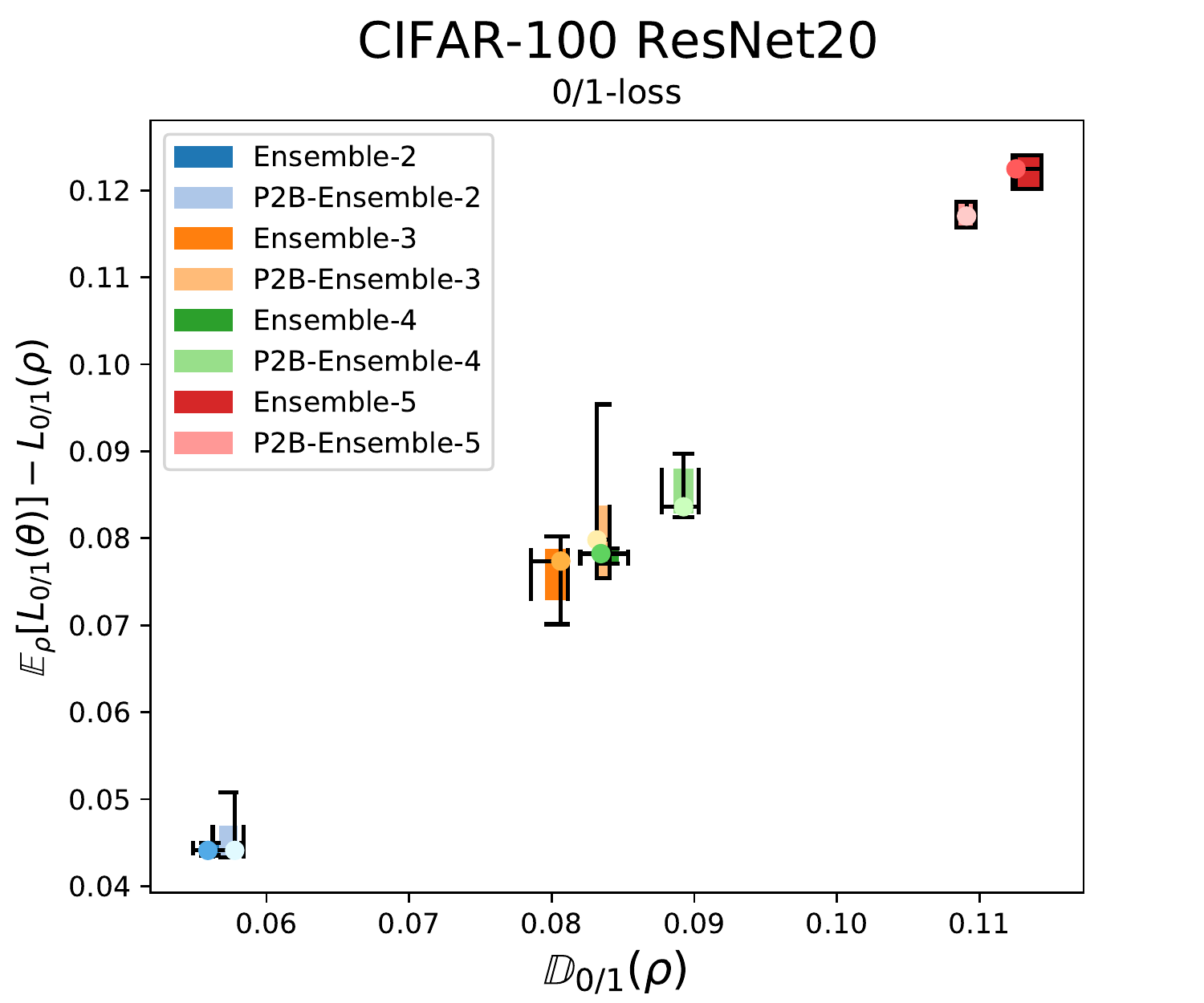}
		\\
		\multicolumn{2}{c}{\includegraphics[scale = 0.4]{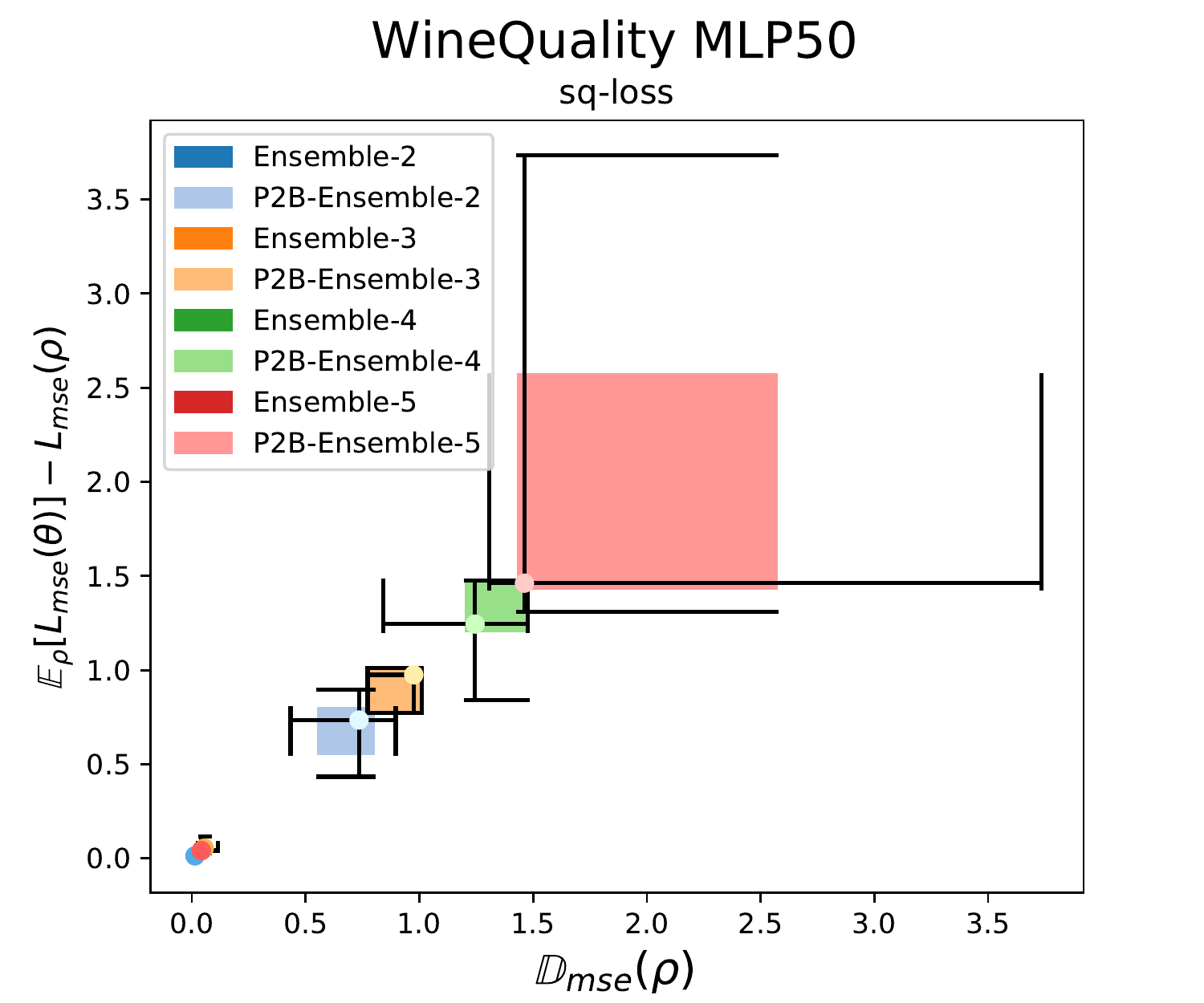}}
	\end{tabular}
	\caption{Each box-plot represents a set of ensemble models learned with the same algorithm and a different ensemble size. The X-axis represents the diversity of the ensemble and the Y-axis represents the gap between the loss of the individual ensemble models and the loss of the ensemble (a positive gap indicates that the ensemble performs better than the individual models). }\label{fig:collorary:gap:extra}
\end{figure}

\subsection{Experiments with different ensemble sizes}

Figure~\ref{fig:collorary:gap:extra} is an extension of Figure \ref{fig:collorary:gap} where different ensemble sizes are considered. Here we can see again how those ensembles with higher diversity $\mathbb{D}(\rho)$ present a higher gap between the average performance of the individual models $\E_\rho[L(\bmtheta)]$ and the performance of the ensemble $\L(\rho)$, as stated in Corollary \ref{collorary:gap}. This figure also shows that, for LeNet5 and MLP50, \textit{Ensemble} do not get a significant gain in diversity by increasing the size of the ensemble. But, for ResNet20, \textit{Ensemble} steadily increases diversity by increasing the ensemble size. We hypothesize that for simple networks random intialization is not as effective to capture different modes as happens with complex neural networks. On the contrary, \textit{P2B-Ensemble} always get an increase in diversity when increasing the ensemble size.

Figure~\ref{fig:EnsembleSize:Barplot} show the generalization performance of the ensembles learned with the \textit{Ensemble} and \textit{P2B-Ensemble} algorithms, for different ensemble size, from two to five. The obtained results from LeNet5 show how increasing the number of models decreases the performance of the ensemble, in contrast to ResNet20 and MLP-50, where the opposite outcome raises. We do not have a convincing explanation for this phenomenom.

\begin{figure}[!htb]
	\centering
	\begin{tabular}{cc}
		\includegraphics[scale = 0.4]{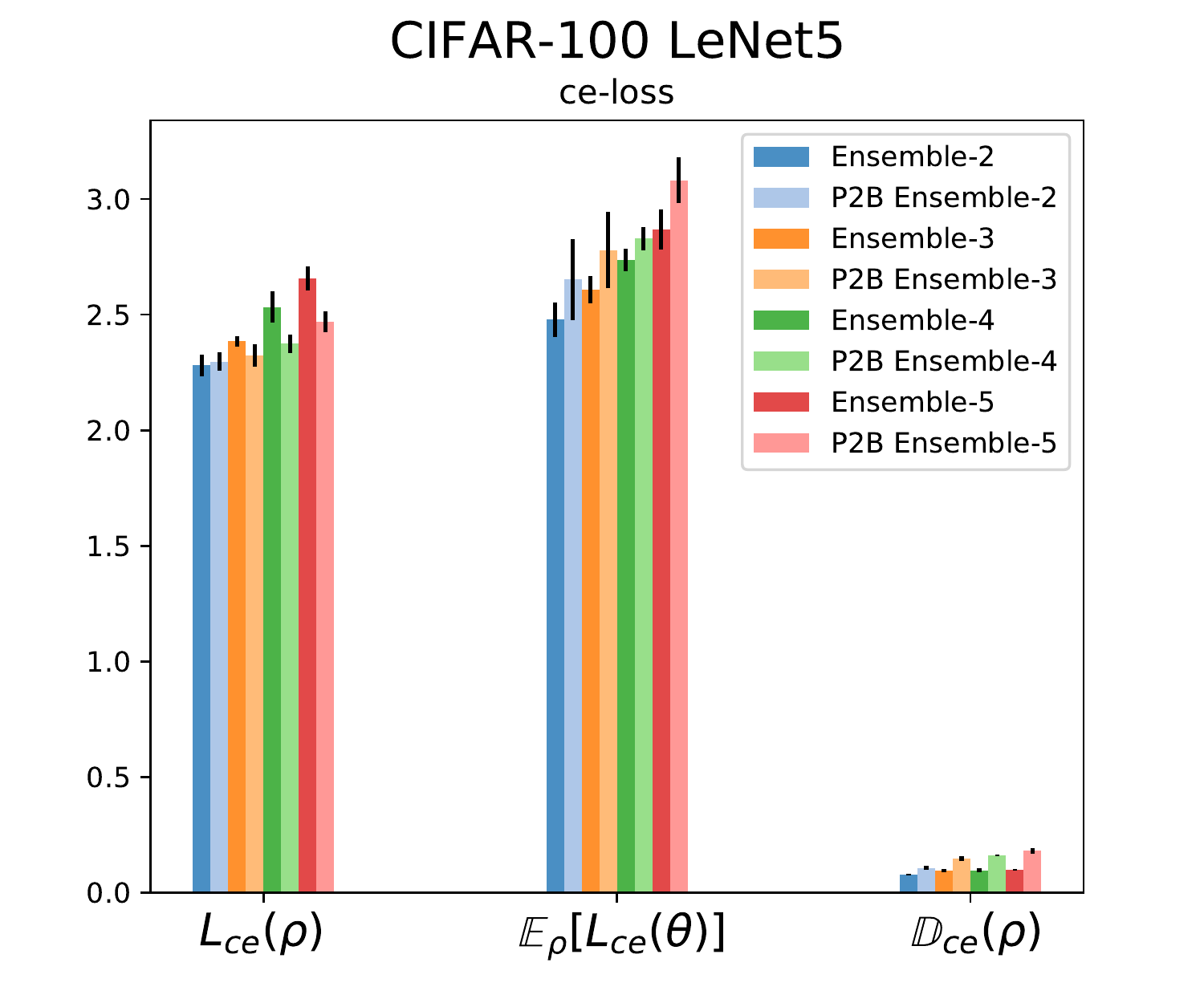}
		&  
		\includegraphics[scale = 0.4]{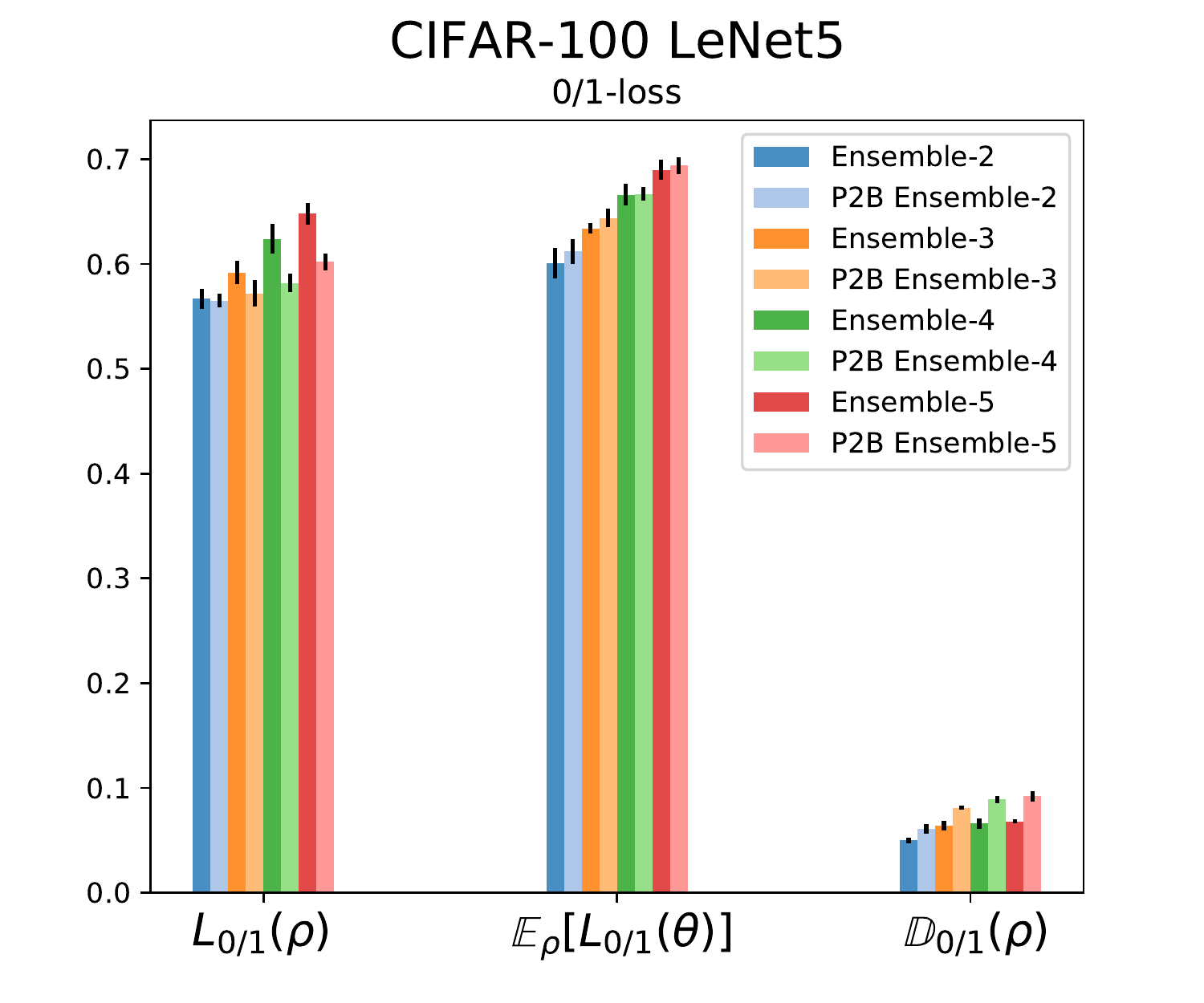}
		\\
		\includegraphics[scale = 0.4]{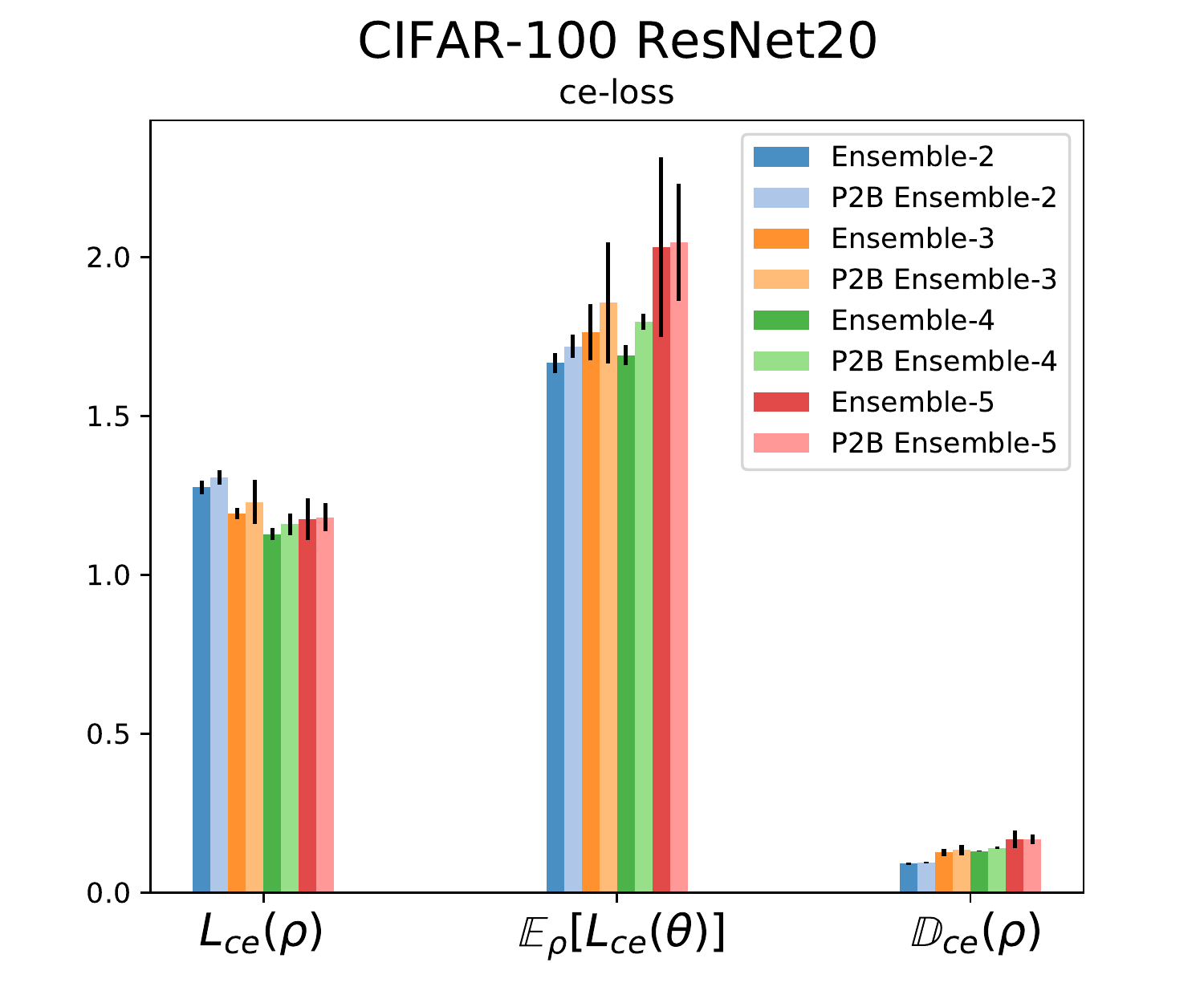}
		&
		\includegraphics[scale = 0.4]{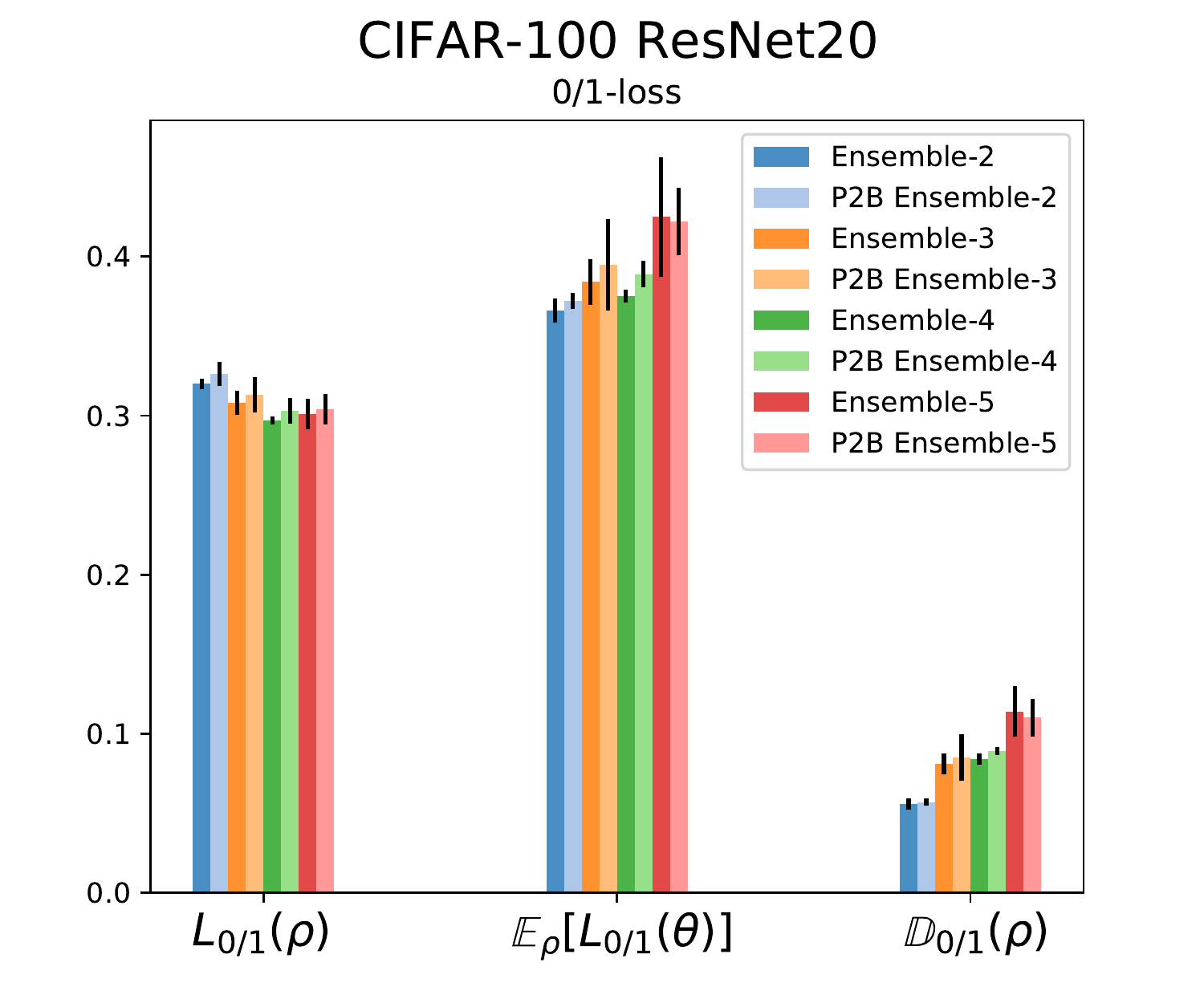}
		\\
		\multicolumn{2}{c}{\includegraphics[scale = 0.4]{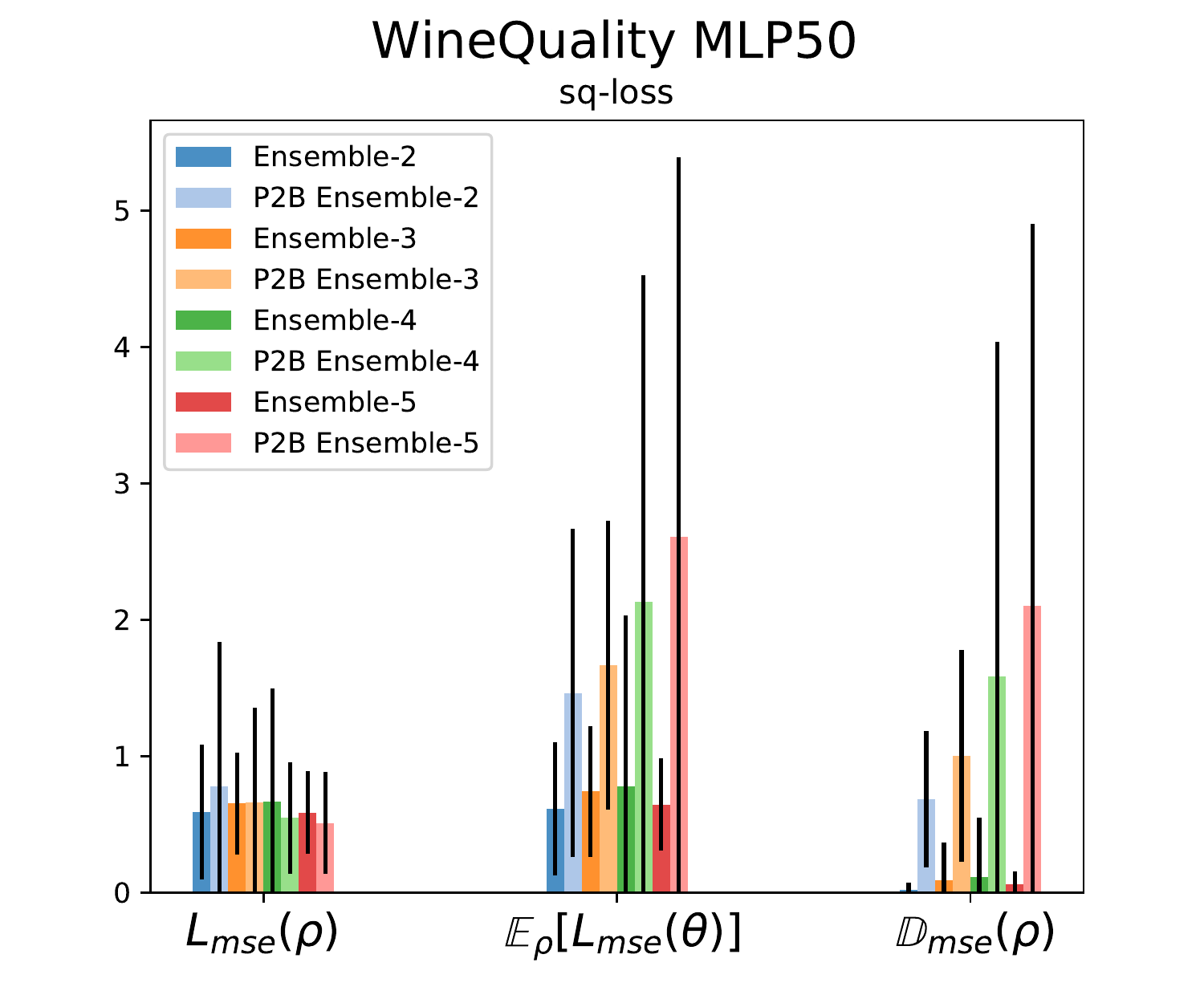}}
	\end{tabular}
	\caption{Mean plus/minus three standard deviations of the ensemble error, average individual models errors and ensemble diversity using LeNet5, ResNet20 and MLP50. Ensemble and P2B-Ensemble algorithms with ensembles of sizes, 2, 3, 4 and 5.}
	\label{fig:EnsembleSize:Barplot}
\end{figure}

\subsection{Evaluation of Corollary \ref{collorary:singlemodel}}

\begin{figure}[t]
	\centering
	\begin{tabular}{cc}
		\hspace{-5pt}\includegraphics[width=0.5\linewidth]{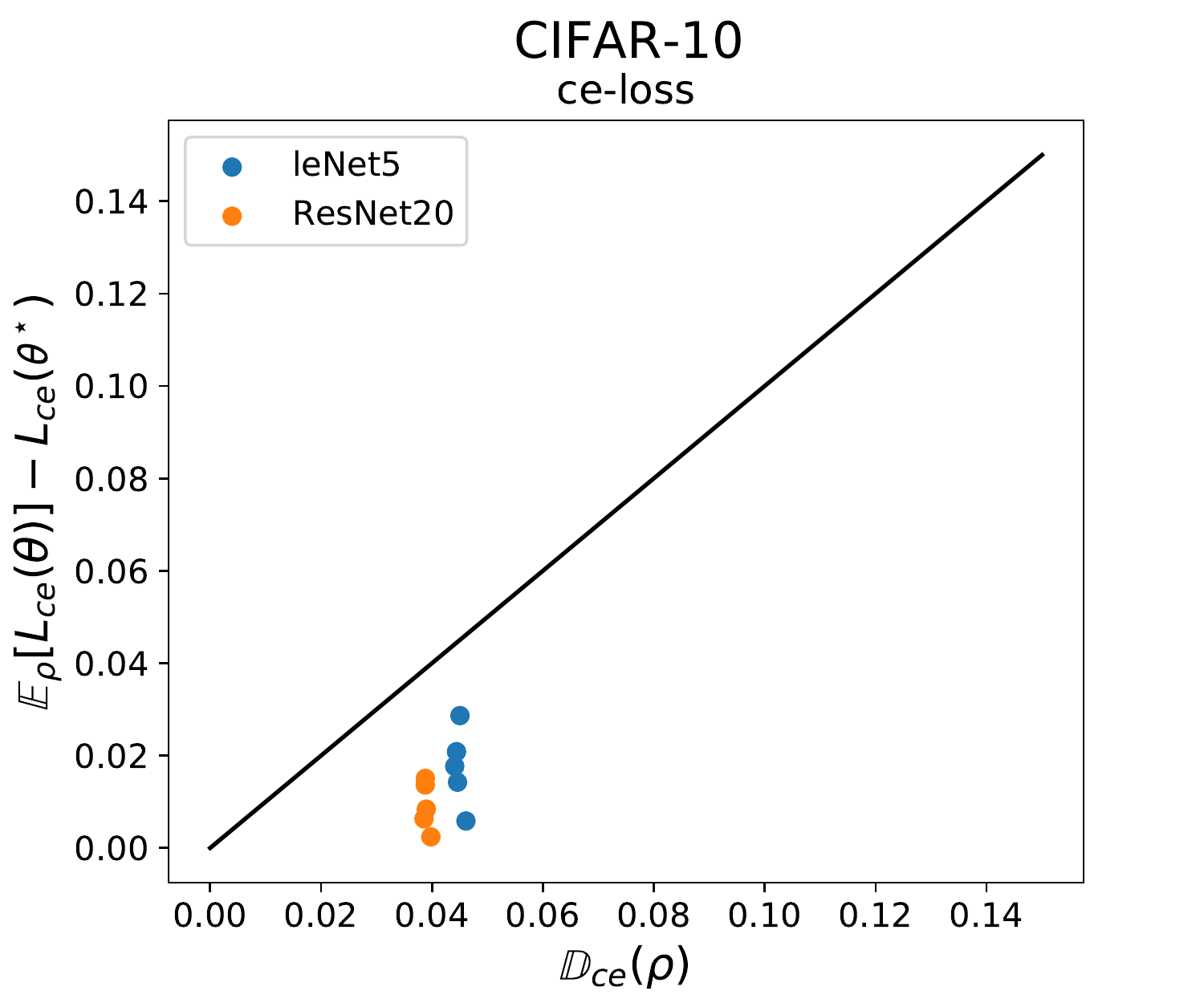}
		&
		\hspace{-5pt}\includegraphics[width=0.5\linewidth]{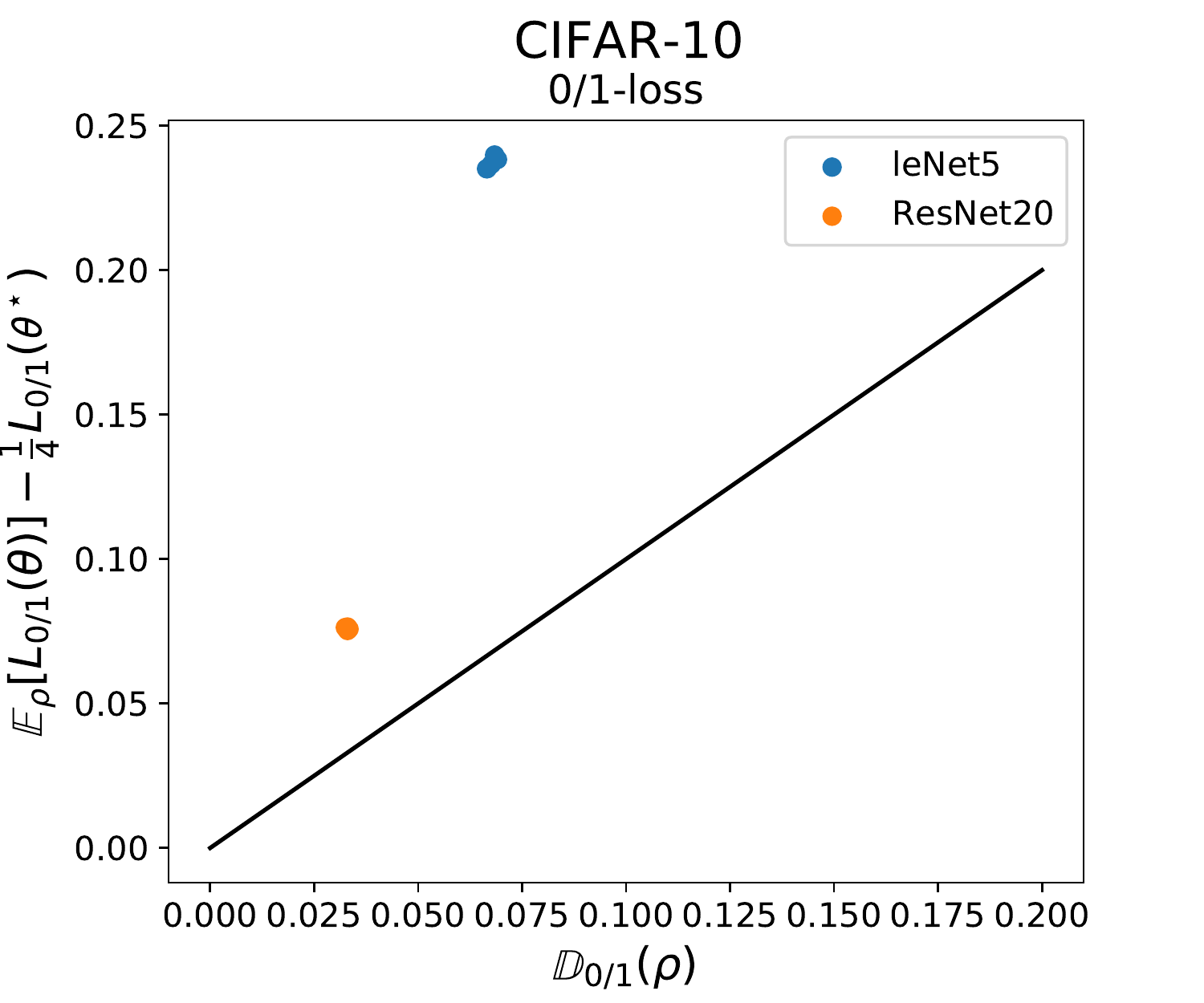}
		\\
		\hspace{-5pt}\includegraphics[width=0.5\linewidth]{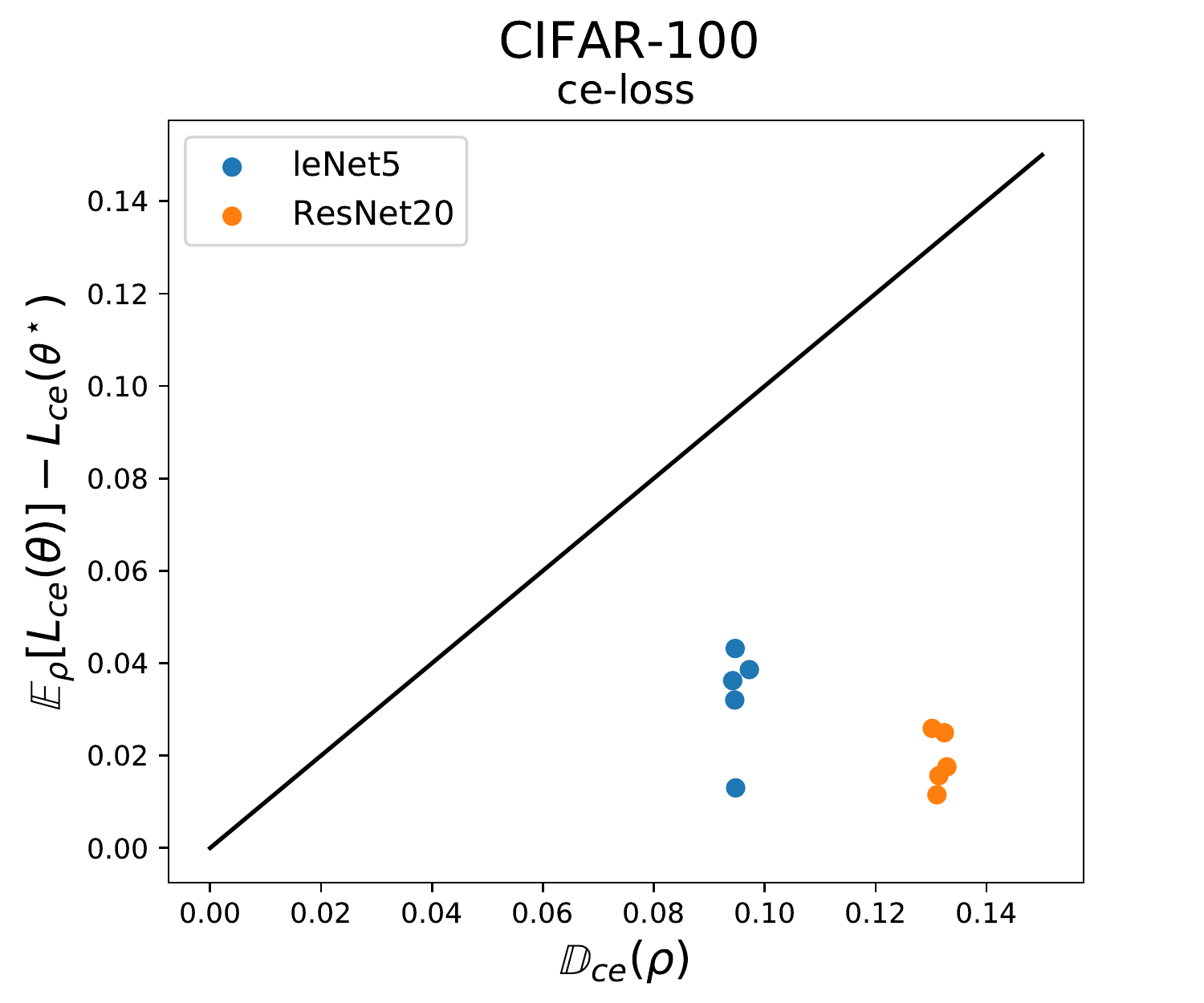}
		&
		\hspace{-5pt}\includegraphics[width=0.5\linewidth]{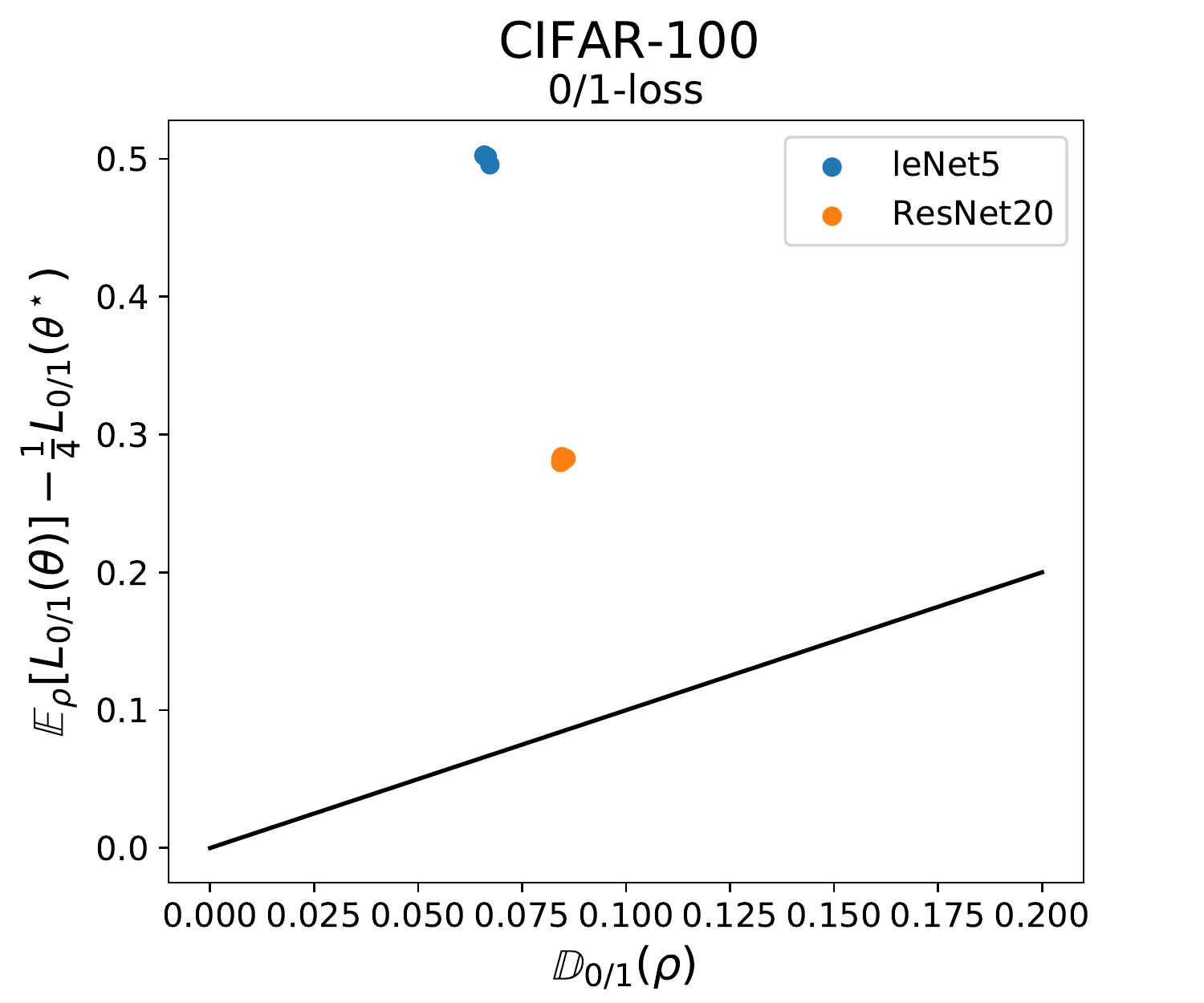}
		\\
		\multicolumn{2}{c}{\hspace{-5pt}\includegraphics[width=0.5\linewidth]{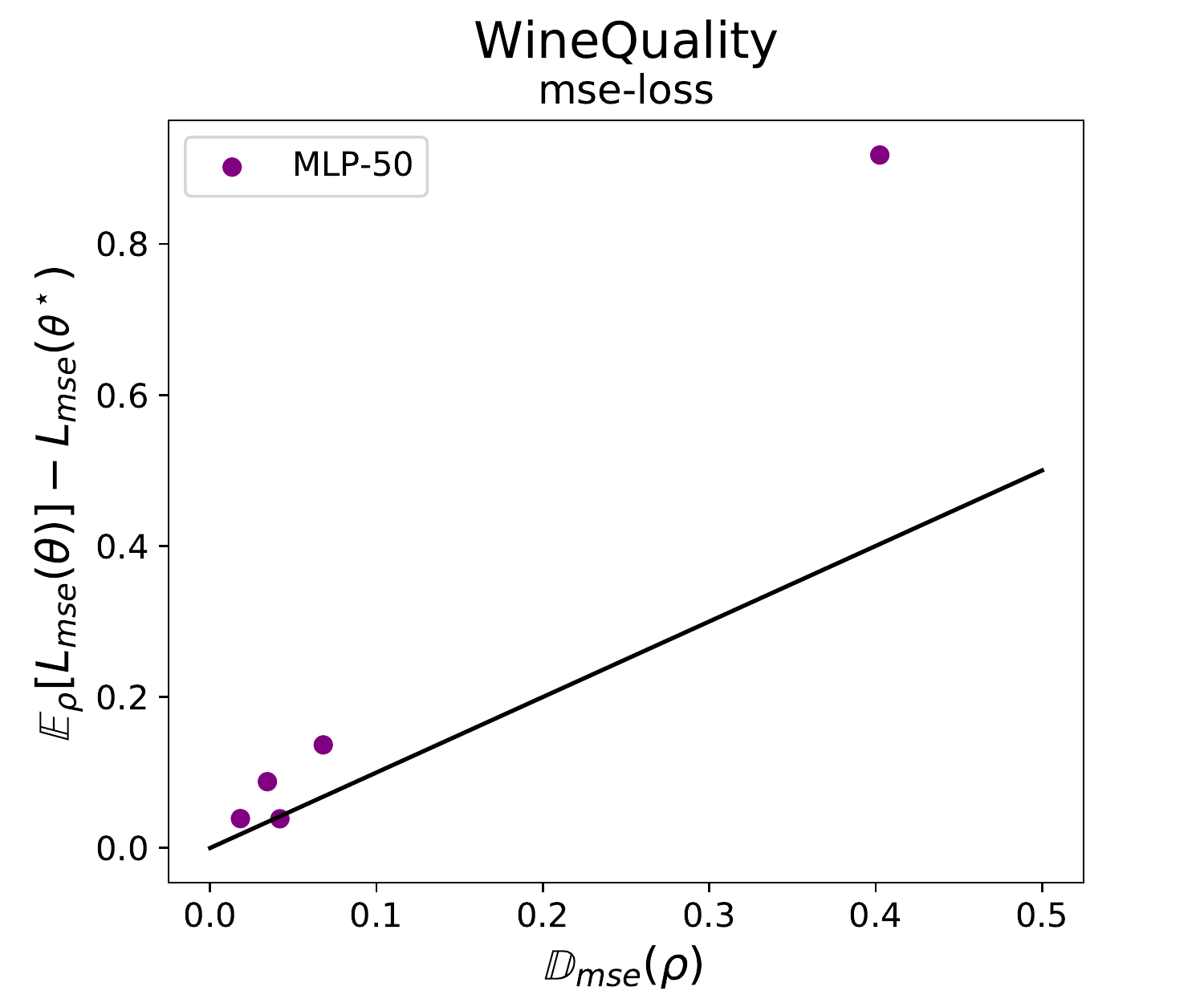}}
		\\
	\end{tabular}
	\caption{\textbf{Evaluation of Corollary \ref{collorary:singlemodel}}. Each point is an ensemble model learned with \textit{Ensemble}. $\bmtheta^\star$ denotes the best single individual model of the ensemble. Ensembles below the black line satisfy the condition of the corollary and outperform the best single individual model. }\label{fig:collorary:singlemodel}
\end{figure}

Finally, we want to through evidence about how good Corollary \ref{collorary:singlemodel} to decide when an ensemble is better than a single model. Figure \ref{fig:collorary:singlemodel} shows how often the condition of Corollary \ref{collorary:singlemodel} is met by those ensembles learned with \textit{Ensemble}, the standard ensemble learning algorithm. In consequence, these ensembles perform better than the best individual model in isolation. Points below the line represents ensembles satisfying this condition. As can be seen, all ensemble models learned by \textit{Ensemble} for the $\CE$-loss in CIFAR-10 and CIFAR-100 satisfy the condition and performs better than the best individual model. So, we can see that the random initialization is an effective way to learn high quality ensembles. However, this is not always the case for $\MSE$-loss in WineQuality using MLP50. In this case, Ensemble often fails to learn ensembles that generalize better than the best individual model. The problem is, as we show in Figure \ref{fig:collorary:gap}, that \textit{Ensemble} learn ensembles with very low diversity. 

In the case of the $\zeroone$-loss, we have the same ensembles than for the $\CE$-loss (remember, majority vote ensembles are trained with the $\CE$-loss). In this case, we do not have any model satisfying the condition of Corollary \ref{collorary:singlemodel}. Even though, one can verify that for all the models the ensemble performs better than best single individual model. Again, the factor $\alpha=4$ has a significant impact because it weakens the bound and reduces its applicability.

\end{document}